\documentclass[twoside]{article}

\usepackage{aistats2026}

\usepackage[round]{natbib}
\usepackage{amsmath,amsfonts,bm, bbm, amssymb}

\newcommand{\hcalG}{\widehat{\mathcal{G}}}

\newcommand{\calC}{\mathcal{C}}
\newcommand{\calS}{\mathcal{S}}
\newcommand{\adj}{\mathrm{adj}}

\newcommand{\indep}{\perp\!\!\!\perp}
\def\eqref#1{equation~\ref{#1}}
\def\1{\bm{1}}

\DeclareMathAlphabet{\mathsfit}{\encodingdefault}{\sfdefault}{m}{sl}
\SetMathAlphabet{\mathsfit}{bold}{\encodingdefault}{\sfdefault}{bx}{n}

\AtBeginEnvironment{align}{\setcounter{equation}{0}}
\usepackage{amsmath}
\newcommand{\ind}{\perp\!\!\!\perp}

\usepackage{hyperref}
\hypersetup{
    colorlinks=true,       % Enable colored links
    linkcolor=blue,        % Color for internal links (e.g., sections)
    citecolor=blue,        % Color for citations
    filecolor=blue,        % Color for file links
    urlcolor=blue          % Color for URLs
}
\usepackage{url}

\usepackage{graphicx}
\usepackage{tikz}
\usepackage{subcaption} 
\usepackage{amsthm}

\usepackage{thmtools}
\usepackage{xcolor} % Add this to your preamble if not already included
\usepackage{algorithm}
\usepackage{algpseudocode}

\usetikzlibrary{positioning}

\newtheorem{theorem}{Theorem}[section]

\newtheorem{proposition}[theorem]{Proposition}
\newtheorem{definition}[theorem]{Definition}
\newtheorem{assumption}[theorem]{Assumption}
\newtheorem{remark}[theorem]{Remark}

\usepackage{enumitem}

\newcommand{\Patrick}[1]{
{\color{green} [Patrick: {#1}]}}

\begin{document}

% If your paper is accepted and the title of your paper is very long,
% the style will print as headings an error message. Use the following
% command to supply a shorter title of your paper so that it can be
% used as headings.
%
%\runningtitle{I use this title instead because the last one was very long}

% If your paper is accepted and the number of authors is large, the
% style will print as headings an error message. Use the following
% command to supply a shorter version of the author names so that
% they can be used as headings (for example, use only the surnames)
%
%\runningauthor{Surname 1, Surname 2, Surname 3, ...., Surname n}

\twocolumn[

\aistatstitle{From Guess2Graph: When and How Can Unreliable Experts Safely Boost Causal Discovery in Finite Samples?}

% \aistatstitle{From Guess2Graph: When and to What Extent Can Unreliable Experts Safely Guide Causal Discovery?}

% \aistatsauthor{Sujai Hiremath \And Dominik Janzing \And Philipp Faller \And Patrick Blöbaum \And Elke Kirschbaum \And Shiva Prasad Kasiviswanathan \And  Kyra Gan }

\aistatsauthor{}

\aistatsaddress{} ]

\begin{abstract}
Causal discovery algorithms often perform poorly with limited samples. While integrating expert knowledge (including from LLMs)  as constraints promises to improve performance, 
guarantees for existing methods require perfect predictions or uncertainty estimates, making them unreliable for practical use.
We propose the \emph{Guess2Graph} (G2G) framework, which uses expert guesses to guide the \emph{sequence} of statistical tests rather than replacing them. 
This maintains statistical consistency while enabling performance improvements.
We develop two instantiations of G2G: PC-Guess, which augments the PC algorithm, and gPC-Guess, a learning-augmented variant designed to better leverage high-quality expert input.
Theoretically, both preserve correctness regardless of expert error, with gPC-Guess provably outperforming its non-augmented counterpart in finite samples when experts are ``better than random."
Empirically, both show monotonic improvement with expert accuracy, with gPC-Guess achieving significantly stronger gains.

\end{abstract}

\section{Introduction}\label{sec:intro}
%%%Kyra: rewrite the intro below %%%%
% \kyra{rewrite below, citations were migrated from prior version, but they don't feel right, fix and add more citations, preferably in numeric style}

Global causal discovery provides a principled framework for inferring causal graphs from observational data, with algorithms that are asymptotically correct and statistically well-characterized \citep{peter_spirtes_causation_2000}. In finite samples, however, these guarantees fail to hold \citep{uhler_geometry_2013}, and performance tends to degrade substantially with insufficient data \citep{zuk_samples_2012}. As a result, discovery in real-world applications often yields unstable \citep{faller2024self}
or inaccurate \citep{brouillard_landscape_2025} graphs, which contradict established domain knowledge \citep{maasch2024local}.

% often yield unstable or inaccurate graphs across paradigms, where outputs may fail on benchmarks,
% Although recent work has attempted to quantify the uncertainty of discovery outputs, performance remains fundamentally limited by sample size.

% Constraint-\citep{peter_spirtes_causation_2000, chickering_optimal_2002} and score-based methods face a super-exponential search space~\citep{chickering2004large}, while the power of conditional independence tests and score functions diminishes rapidly as conditioning sets grow. Functional causal model approaches encounter analogous challenges when running sequential regression tests.

% , as nonparametric regression and residual-independence tests are sensitive to weak signal strengths and the curse of dimensionality.
% As a result, real-world applications often yield unstable or inaccurate graphs across paradigms, particularly in settings with dense structures and limited sample sizes, where outputs may fail on benchmarks or contradict established domain knowledge \citep{maasch2024local}.

% \sujai{will integrate the idea that sequential testing is found throughout all types of methods (constraint based, score-constraint hybrid, ANM, etc.}

\textit{Expert-aided} discovery methods often mitigate finite-sample issues by incorporating domain knowledge, encoding such knowledge as either hard constraints that prune the search space \citep{ankan2025expert} or as soft priors that bias the results of statistical tests \citep{constantinou2023impact}. Classical expert-aided methods rely on human experts to specify much of the constraint \citep{tennant_use_2021, petersen2021data}; however, as graph size grows, this dependence on human input becomes cognitively infeasible and economically unsustainable.
% Discovery methods often mitigate finite-sample issues by incorporating domain knowledge. Classical methods rely on human experts to specify full graphs \citep{tennant_use_2021, petersen2021data}, or to encode such knowledge as either hard constraints that prune the search space \citep{ankan2025expert} or as soft priors that bias it \citep{constantinou2023impact}. However, as graph size grows,
% this dependence on human input becomes cognitively infeasible and economically unsustainable.
% relying on human experts to provide such knowledge quickly becomes cognitively infeasible and economically unsustainable. 
\emph{Large language models} (LLMs),
trained on vast and diverse corpora, encode broad domain knowledge, suggesting they could serve as scalable proxies for human experts. 
Much like a domain expert reasoning about plausible causal mechanisms, an LLM can parse variable names and draw on its internalized knowledge to propose causal constraints, potentially reducing reliance on exhaustive statistical testing. 
Their promise in retrieving known direct relationships \citep{feng2024pretraining} and ancestral orderings \citep{vashishtha2025causal}
make them a compelling alternative, spurring research into their use as expert replacements for causal discovery \citep{ban2023query, cohrs2024large, darvariu2024large, xie2024cloud, kiciman2024causal, vashishtha2025causal}.
% Their broad query capacity and knowledge bases

% While they show promise in retrieving known relationships \citep{feng2024pretraining} and coarse-grained orderings \citep{vashishtha2025causal}, their outputs often contain invalid graphs, are brittle to prompt variations \citep{ban2023query}, and degrade with increasing complexity \citep{llm_causal_discovery_2024}.
% Moreover, they lack reliable uncertainty quantification and theoretical guarantees: treating their outputs as hard constraints risks unbounded error \citep{knowledge_guided_discovery_2024}, while soft-constraint methods suffer from poorly calibrated confidence scores \citep{wu2025llm}.

% Yet neither human experts nor LLMs are infallible sources of causal knowledge. 
% Human input is prone to bias and inconsistency \citep{fill in}, while LLMs exhibit critical limitations: they often output invalid graphs \citep{fill in}, are brittle to prompt variations \citep{ban2023query}, degrade with increasing complexity \citep{llm_causal_discovery_2024}, perform poorly on out-of-distribution domains \citep{feng2024pretraining}, produce unreliable reasoning \citep{ye2024benchmarking_arxiv, dong-etal-2024-generalization}, as well as poor uncertainty calibration \citep{manggala2025qa, wu2025llm}. 

Yet neither human experts nor LLMs are infallible sources of causal knowledge. Human input is prone to bias and inconsistency \citep{dror_cognitive_2020}, while LLMs exhibit critical limitations: they often output invalid graphs \citep{jiralerspong_causal_2024}, are brittle to prompt variations \citep{ban2023query}, degrade with increasing complexity \citep{llm_causal_discovery_2024}, perform poorly on out-of-distribution domains \citep{feng2024pretraining}, produce unreliable reasoning \citep{ye2024benchmarking_arxiv, dong-etal-2024-generalization}, as well as poor uncertainty calibration \citep{manggala2025qa}. Given the potential for error, traditional methods that leverage such unreliable experts to generate hard or soft constraints lack theoretical guarantees \citep{wu2025llm}—this means that misleading advice can cause error (sometimes unbounded error, see Appendix \ref{appendix: unbounded error}) even in the large sample limit, rendering them unsuitable for safety-critical applications.
To address this, we argue that \emph{expert guidance should not replace statistical procedures, but rather, complement them to improve finite-sample efficiency without sacrificing worst-case guarantees.}

This principle directly motivates our two core research questions. First, we ask: \emph{What formal framework ensures that integrating an unreliable expert never harms, and may improve, algorithmic performance?} Establishing this framework, however, reveals a second critical challenge. The rigid, statistically-optimized architectures of many high-performance algorithms (e.g., the PC algorithm's fixed-size conditioning set iteration) inherently resist external guidance. This leads to our second question: \emph{How can we redesign such algorithms to become more receptive to this framework, thereby unlocking greater performance gains from accurate expertise?}
% 
% Therefore, a key challenge for safe expert-aided discovery is developing a robust way to integrate imperfect experts that maintains correctness while still benefiting from accurate guidance. This requires addressing: (1) \emph{when} discovery algorithms can incorporate expert guidance to provably improve performance without sacrificing guarantees, (2) \emph{how} to augment specific algorithms to provably benefit from imperfect signals, and (3) \emph{to what extent} different algorithms benefit from increasingly accurate experts. 

% \shiva{Can you call out contributions as a bulleted list}
\textbf{Contributions\;\;}  This work introduces a principled approach for leveraging fallible experts in causal discovery, addressing the dual challenges of \emph{when} such guidance is safe and \emph{how} to best implement it. Our contributions are fourfold:
\begin{itemize}[leftmargin=*,itemsep=0pt,parsep=2pt, topsep=0pt, partopsep=0pt]
% \item \textbf{A Novel Framework}: We introduce Guess2Graph (G2G) (Sec. \ref{sec: setup}), 
% the first framework for causal discovery that formally integrates potentially erroneous expert predictions while guaranteeing asymptotic correctness and provable performance gains. G2G guides the sequence of statistical tests conducted by an method's internal subroutines (Figure \ref{fig: pipeline_subroutines_v2}). We identify two subroutines common to constraint-based algorithms that are susceptible to guidance (Sec. \ref{sec: g2g cb theory}), and modify them to leverage expert predictions.

\item \textbf{A Framework for Expert-Guided Causal Discovery}: We introduce the Guess2Graph (G2G) (Sec.~\ref{sec: setup}) framework,
which enables causal discovery algorithms to incorporate expert predictions while guaranteeing statistical consistency (\ref{criterion:consistency}), providing a pathway to achieve monotonic improvement (\ref{criterion:monotone}) and finite-sample robustness (\ref{criterion:finite-sample}). G2G uses expert predictions to guide test sequences rather than outcomes, requiring no uncertainty quantification.
We provide the theoretical foundation for instantiating G2G in constraint-based algorithms through subroutine modification (Sec.~\ref{sec: g2g cb theory}).

    \item \textbf{PC-Guess}:
    We instantiate G2G in the PC algorithm (Sec.~\ref{subsec: pcg2g}) to create PC-Guess (Alg.~\ref{algo: guess-pc}), which maintains \ref{criterion:consistency} while partially achieving \ref{criterion:monotone} and \ref{criterion:finite-sample} through per-iteration performance guarantees. We prove that when starting from identical states, each iteration of PC-Guess shows provable improvement over standard PC with experts `better than random' (Thm.~\ref{theorem: pc performance}), though cascading effects prevent end-to-end guarantees.    
    % leveraging expert guidance to optimize the sequence of independence tests conducted by the aforementioned subroutines. We prove that the resulting PC-Guess algorithm (Alg.~\ref{algo: guess-pc}) maintains asymptotic correctness (Thm.~\ref{theorem: guess acc correct}). When the expert is `better than random', PC-Guess provably outperforms standard PC at each iteration when starting from identical states (Thm.~\ref{theorem: pc performance}), though cascading effects from improved early iterations make end-to-end guarantees challenging.

    % \item \textbf{Learning Augmented PC for Better Expert Guidance} (LAPC-G2G):
    % We demonstrate that the rigid, statistically optimized structure of PC limits the potential gains from expert guidance (Sec. \ref{subsec: lapcg2g}). To overcome this, we propose a redesigned algorithm, LAPC-G2G (Alg. \ref{algo: guess-sgs}), which modifies PC's core procedure to be more susceptible to expert input. We provide theoretical guarantees showing that LAPC-G2G achieves stronger performance coupling with expert quality than PC-G2G (Thms. \ref{theorem: guess acc correct}, \ref{theorem: sgs performance}).

    \item \textbf{Augment PC for Better Expert Guidance (gPC-Guess)}: We demonstrate that the rigid, statistically optimized structure of PC limits the potential gains from expert guidance (Sec.~\ref{subsec: lapcg2g}). To overcome this, we propose a redesigned algorithm, gPC-Guess (Alg.~\ref{algo: guess-sgs}), which modifies PC's core procedure to be more susceptible to expert input. 
    This approach fully achieves all three criteria \ref{criterion:consistency}-\ref{criterion:finite-sample}, with
    % Unlike PC-Guess's per-iteration guarantees, we prove that gPC-Guess achieves 
    provable end-to-end finite-sample performance improvements that increase monotonically with expert quality (Thms.~\ref{theorem: guess acc correct}, \ref{theorem: sgs performance}).
    
    % improvements than PC-G2G under the same expert quality.
    % We show how to extract search strategies from an expert's predicted causal graph to guide sequences of independence tests in constraint-based algorithms, leveraging insight form prior work \citep{colombo2014order} that test ordering can dramatically affect finite-sample performance. We instantiate this approach by augmenting the SGS and PC algorithms to get SGS-GUESS and PC-GUESS, proving that both variants maintain correctness while achieving monotonic performance improvement. When experts predict `better than random', the augmented algorithms provably outperform their non-augmented counterparts.
     % with expert accuacy in the expected probability that no errors are made
     \item \textbf{Empirical Validation and Insights}:
     Our experiments (Sec.~\ref{sec: exps}) validate the theoretical distinction between algorithm augmentation and redesign for expert guidance. PC-Guess shows modest gains (up to 5\%)
% Our experiments (Sec.~\ref{sec: exps}) demonstrate a critical distinction between augmenting versus redesigning algorithms for expert guidance. PC-Guess provides modest gains of up to 5\%,
limited by PC's inherent rigidity. In contrast, gPC-Guess fully achieves all three criteria, with up to 30\% performance gains when experts are accurate, across both synthetic and real-world data. These results persist with LLM experts, confirming that full achievement of our criteria requires algorithmic redesign rather than simple augmentation.

\end{itemize}

\vspace{-5pt}
\section{Related Work}
\vspace{-5pt}

% \kyra{missing related work on learning with augmented algorithms, could also add a few sentences on local discovery, and how this aids our current line of work of using local structures to improve efficiency, for example.}
% \kyra{skipping related work for now}

% with further details about each topic in Appendix \ref{sec: related work}.

% \kyra{the only section that i see useful in the appendix is C.3, the rest should all be commented out. update the sentence directly to say what is included in Appendix C }
% \sujai{here, 1) refere again the general decomposition of data-based causal discovery into constraint-based, score-based, FCM-based (ANM being most prominent), then 2) emphasize that CI-based algorithms show up even in the other hybrid methods score-absed and FCM, 3) justify our focus on CI based methods.}\sujai{all algorithms run sequential test, however ordering-dependence, cna differ betwee methods. we focus on sequential dependence in CI algorithms, then we discuss how their performance is impacted differently by expert guidance.}
\textbf{Global Causal Discovery\;\;} 
Global causal discovery methods—constraint-based, score-based, and \emph{functional causal model} (FCM) based—all face finite-sample challenges from error propagation in sequential statistical tests. Constraint-based methods perform conditional independence tests \citep{peter_spirtes_causation_2000,spirtes_algorithm_1991,spirtes_anytime_2001,lee_constraint_2025}, score-based methods make sequential edge comparisons \citep{chickering_optimal_2002}, and FCM methods conduct residual independence tests \citep{zhang_identiability_2009, suj_2024}, all suffering from diminishing test power in super-exponential search spaces \citep{lee_constraint_2025,chickering_statistically_2020}. While prior work addressed this by eliminating order-dependence (i.e., PC-Stable \citep{colombo2014order}), we instead optimize test sequences using expert predictions. This approach applies broadly since sequential testing underlies these constraint-based, hybrid score-based \citep{tsamardinos_max-min_2006, chickering_statistically_2020, zhu_hybrid_2024}, and hybrid FCM methods \citep{peters_causal_2014, suj_2025_losam} (Appendix~\ref{appendix: sequential test issue}).

\textbf{Expert-Aided Discovery\;\;}
Existing frameworks for integrating expert predictions into causal discovery face two primary challenges. \emph{Direct hard or soft constraint-based approaches} \citep{ban_causal_2023, susanti2025can, takayama_integrating_2025} use expert outputs to replace or bias statistical procedures, risking unbounded error propagation with incorrect experts \citep{knowledge_guided_discovery_2024} and suffering from poorly calibrated confidence scores \citep{campbell2024overprecision, wu2025llm}. \emph{Guidance-based approaches} \citep{constantinou2023impact, wu2025llm, ejaz_less_2025} use predictions for algorithm guidance without replacing tests, but provide limited benefits—initialization helps only with near-perfect experts, while heuristic guidance lacks performance guarantees. Both approaches rely on empirical validation rather than robust statistical foundations, limiting reliability (Appendix~\ref{appendix: expert error violation}).

\textbf{Algorithms with Predictions\;\;}
Our work builds on the 
% algorithms with predictions (or
\emph{algorithms with predictions} or the \emph{learning-augmented algorithms} paradigm \citep{algorithms_predictions_2020}, which integrates predictions into classical algorithms to improve performance while preserving worst-case guarantees. 
In online settings where data or requests arrive sequentially, this approach provides \emph{approximation ratios} guarantees via consistency (near-optimal performance with accurate predictions) and robustness (bounded worst-case performance with poor predictions) \citep{competitive_caching_2018, wei2020optimal, jin2022online, liu2024online}. In offline settings, it reduces \emph{computational} or \emph{query complexity} while maintaining guarantees \citep{kraska_case_2018, lykouris_competitive_2021, balcan_learning_2021, chen_learning_2022}. 
% performance is measured through \emph{computational} or \emph{query complexity}, where predictions reduce expected runtime or required operations—such as decreasing the expected number of comparisons in search algorithms—while again preserving guarantees.
We adapt this framework to causal discovery by using expert predictions to optimize statistical test sequences, primarily targeting improved estimation accuracy rather than computational benefits. While learning-augmented methods have been applied to causal intervention design \citep{active-learning-advice-arnab-2023}, ours is the first adaptation to purely observational causal discovery.

\vspace{-5pt}
\section{Problem Setup and Guess2Graph}\label{sec: setup}
\vspace{-5pt}

% \subsection{Notation and Problem Setup}\label{subsec: setup}
% \subsection{Causal Discovery Aided by Experts}
% In this section we first introduce our problem setup, then present our high-level solution framework.
% \section{Problem Setup}\label{sec: setup}
% In this section we introduce our problem setup.

% the Causal Discovery Guided by an Unreliable Expert's Search Strategy (CD-GUESS) Framework.

 % including notaions, assumptions and desiderata,We denote by $\hat{\mathcal{D}_{i,j}}$ all d-seperating sets of $x_i,x_j$ in $\mathcal{G}$.
% \section{Problem Formulation}\label{sec: problemsetup}
% \subsection{Problem Setup}\label{subsec: problemsetup}

% \textbf{Notation.} 
Unless otherwise mentioned, we denote random variables by lowercase letters and sets of variables by uppercase letters. A \emph{directed acyclic graph} (DAG)
% using lowercase letters like $x$. A DAG 
$\mathcal{G} = (V, E)$ consists of nodes $V$ and edges $E$.
We use $e_{i,j}$ to denote the \emph{directed} edge from $x_i$ to $x_j$ and $n_{i,j}$ to denote the \emph{undirected} edge between $x_i$ and $x_j$, regardless of whether these edges exist in $\mathcal{G}$.
% , encoding direct causal relationships.
The model is defined by structural equations: 
% is an acyclic directed graph whose nodes $V$ correspond to random variables and whose edges $E$ represent direct causal relationships. For every
for each $x_i \in V$,
% we have
$x_i = f(\mathrm{Pa}(x_i), \varepsilon_i)$, with jointly independent noise terms
% and all 
$\bm{\varepsilon}$.
% are jointly independent noise terms.

The skeleton $\mathcal{S}$ of $\mathcal{G}$ is its undirected version.
% The skeleton $\mathcal{S} = (V, E_{\mathcal{S}})$ of a DAG $\mathcal{G} = (V, E_{\mathcal{G}})$ is the graph with undirected edges $E_{\mathcal{S}}$.
For any partial skeleton $\mathcal{C}$, let
% We use $C$ to denote the current skeleton during discovery, 
$\mathrm{adj}(\calC, x_i)$ be
the adjacency of $x_i$ in $\calC$, 
% for the adjacency set of $x_i$ in $C$, 
and $\adj_{-j}(\calC, x_i) = \mathrm{adj}(\calC, x_i) \setminus \{x_j\}$ be the adjacency set excluding $x_j$.
% for the adjacency set of $x_i$ excluding $x_j$.
% Let $[A]_k$ be the set of all size-$k$ subsets of a set $A$, with $[A]_[i:k]$ the set of all subsets of sizes $i$ up to $k$.
Let $[A]_k$ denote all size-$k$ subsets of set $A$, and $[A]_{i:k} = \bigcup_{j=i}^k [A]_j$.
A \emph{conditional independence test} (CIT), $\text{CIT}(x,y|Z)$, tests the null hypothesis that $x\indep y|Z$.  
% Let $e_{i,j}$ and $n_{i,j}$ be binary indicators for the \emph{directed}  and \emph{undirected} edge between distinct nodes $x_i, x_j \in V$, respectively, such that: $e_{i,j}=1$ if $x_i \rightarrow x_j \in\mathcal{G}$, and  $0$ otherwise; and $n_{i,j}=1$ if an undirected edge exists between $x_i$ and $x_j$ in $\calG$, and 0 otherwise.
An edge ordering $\mathsf{O}$ is a sequence of undirected edges, while a subset list $\mathsf{L}$ is a sequence of variable subsets.
Finally, a domain expert (or LLM) is modeled as a predictor $\psi$
% :V\rightarrow \hcalG$
that, given a set of variables $V$, outputs a prediction $\hcalG$, while a causal discovery algorithm outputs a prediction $\widetilde{\mathcal{G}}$.
\vspace{-5pt}
\subsection{Problem Statement and Design Criteria}
\vspace{-5pt}

\label{subsec: desiderata} 
We consider the problem of causal discovery from a finite-sample dataset $\mathcal{X}$, generated by some underlying causal system. Under the standard assumptions (Markov condition, acyclicity, faithfulness, and causal sufficiency, Def.s~\ref{def:markov}-\ref{def:sufficiency}), there exists a true causal DAG $\mathcal{G}^*$ that perfectly characterizes the conditional independence structure via d-separation (Def.~\ref{def:d_separation}), and all direct common causes of variables in $V$ are contained within $V$.
% such that:
% % the d-separation (Def.~\ref{def:d_separation}) relations in $\mathcal{G}^*$ correspond exactly to conditional independence relations in the population distribution:
% 1) For any nodes $x, y$ and set $Z$, $x$ and $y$ are d-separated (Def.~\ref{def:d_separation}) by $Z$ in $\mathcal{G}$ if and only if $x \perp\!\!\!\perp y \mid Z$; and 2)
% All direct common causes of variables in $V$ are contained within $V$.

In practice, finite-sample conditional independence tests are error-prone, making exact recovery of $\mathcal{G}^*$ challenging. We address this by augmenting causal discovery with predictions from an expert $\psi$. While inspired by learning-augmented algorithms, our setting differs in two key aspects: 1) we face statistical (not adversarial) data, with potentially adversarial expert quality, and 2) we require no uncertainty quantification and treat experts as a black box (unlike typical learning-augmented approaches that require tunable confidence parameters). This yields three key criteria:
\begin{enumerate}[leftmargin=*,itemsep=0pt,parsep=0pt, topsep=0pt,partopsep=0pt, label=\textbf{C\arabic*}]
    \item \textbf{Statistical Consistency:}\label{criterion:consistency} As sample size grows, the recovery of the true graph is guaranteed, regardless of expert quality: $\lim_{n \to \infty} \mathbb{P}[\tilde{\mathcal{G}} = \mathcal{G}^*] = 1$. 
    \item \textbf{Monotonic Improvement:} The algorithm's finite-sample performance improves monotonically with expert accuracy.\label{criterion:monotone}
    \item \textbf{Finite-Sample Robustness:} There exists an expert accuracy threshold such that, for finite samples, the algorithm's performance with expert guidance is not worse in expectation than without it when expert accuracy exceeds this threshold.\label{criterion:finite-sample}
\end{enumerate}
Criterion \ref{criterion:consistency} ensures the algorithm remains fundamentally sound even with poor experts, while \ref{criterion:finite-sample} ensures practical utility with sufficiently accurate experts. Criterion \ref{criterion:monotone} connects these guarantees, ensuring a smooth transition between regimes.
We note that traditional frameworks for incorporating expert knowledge as hard or soft constraints violate \ref{criterion:consistency}, as we illustrate in Appendix~\ref{appendix: expert error violation}.

\vspace{-5pt}
\subsection{Guess2Graph Framework}\label{subsec: framework}
\vspace{-5pt}

% In this section we present our solution framework, and show explain how to leverage it in constraint-based discovery.

We now propose our Guess2Graph (G2G) framework, which enables algorithms to satisfy our three criteria by strategically incorporating expert guidance while maintaining statistical foundations. The core insight is that many causal discovery algorithms contain subroutines that perform sequences of statistical tests, often with orders sampled uniformly at random. In subroutines where any valid sequence maintains asymptotic consistency, we can replace random sampling with expert-guided ordering while preserving theoretical guarantees.

The G2G framework operates in three steps: (1) identify a subroutine of an asymptotically correct causal discovery algorithm that performs sequences of statistical tests; 
(2) request expert $\psi$ to predict a causal structure $\widehat{\mathcal{G}}$; and (3) extract and use an ordering from $\widehat{\mathcal{G}}$ in place of random sampling.
This approach automatically ensures statistical consistency (\ref{criterion:consistency}) by keeping all decisions grounded in test outcomes rather than expert judgments.\footnote{While we focus on deterministic  predictions, G2G can be extended with expert selection/validation (App.~\ref{appendix: cd-guess framework extensions}).}
However, achieving monotonic improvement (Criterion~\ref{criterion:monotone}) and finite-sample guarantees (Criterion~\ref{criterion:finite-sample}) requires careful algorithmic design within this framework.

Although the framework is general, i.e., applicable to any discovery algorithm maintaining consistency across test sequences, extracting effective orderings requires careful analysis of each subroutine's role. We therefore focus on constraint-based methods in this paper, with extensions to score-based and FCM-based algorithms discussed in Appendices~\ref{appendix: cd-guess score extension} and~\ref{appendix: cd-guess anm extension}.

Next, in Section~\ref{sec: g2g cb theory}, we demonstrate how this framework can be applied to constraint-based algorithms by identifying common subroutines that can incorporate learning augmentation. In Section \ref{sec: method guarantees}, we show how the framework can be applied to the PC algorithm to partially achieve \ref{criterion:monotone} and \ref{criterion:finite-sample}. By further modifying PC to create our gPC-Guess variant, we show that both \ref{criterion:monotone} and \ref{criterion:finite-sample} can be fully achieved.

\begin{figure}[t]
\centering
\begin{tikzpicture}[
    scale=0.9, transform shape,
    box/.style={rectangle, draw, rounded corners, minimum width=2.8cm, minimum height=1cm, align=center, font=\small},
    testbox/.style={rectangle, draw, minimum width=2.5cm, minimum height=0.6cm, align=left, font=\footnotesize},
    arrow/.style={->, thick, >=stealth},
    node/.style={circle, draw, minimum size=0.4cm, inner sep=1pt, font=\footnotesize}
]

% Top row: Expert Guess (left) and Default Subroutine (right)
% Expert Graph
\node[font=\small\bfseries] at (-2, 3.5) {Expert Initial};
\node[font=\small\bfseries] at (-2, 3.1) {Graph Guess};
\node[node] (x1) at (-3, 2.5) {$x_1$};
\node[node] (x2) at (-2, 2.5) {$x_2$};
\node[node] (x3) at (-1, 2.5) {$x_3$};
\draw[thick] (x1) -- (x2);
\draw[thick] (x2) -- (x3);

% Default test sequence
\node[font=\small\bfseries] at (2, 3.5) {Subroutine's Random};
\node[font=\small\bfseries] at (2, 3.1) {Test Sequence};
\node[testbox, fill=red!10] (t1) at (2, 2.5) {1. Test $x_1 \indep x_3 | x_2$};
\node[testbox, fill=yellow!10, below=0.1cm of t1] (t2) {2. Test $x_2 \indep x_3| x_1$};
\node[testbox, fill=green!10, below=0.1cm of t2] (t3) {3. Test $x_1 \indep x_2| x_3$};

% Middle: Arrows converging to G2G sequence
\draw[arrow, very thick, blue!60] (-2, 2.1) -- (-2, 0.8) node[midway, left, font=\footnotesize] {};
\draw[arrow, very thick, blue!60] (0.3, 1.3) -- (-2, 0.8) node[midway, right, font=\footnotesize] {};

% Bottom: G2G reordered sequence (moved to lower left)
\node[font=\small\bfseries, blue!70!black] at (-2, 0.3) {G2G's Reordered};
\node[font=\small\bfseries, blue!70!black] at (-2, -0.1) {Test Sequence};
\node[testbox, fill=green!10] (g1) at (-2, -0.7) {1. Test $x_1 \indep x_2|x_3$};
\node[testbox, fill=yellow!10, below=0.1cm of g1] (g2) {2. Test $x_2 \indep x_3|x_1$};
\node[testbox, fill=red!10, below=0.1cm of g2] (g3) {3. Test $x_1 \indep x_3|x_2$};

% Arrow to discovered graph
\draw[arrow, very thick, blue!60] (-0.05, -1.2) -- (1.35, -1.2) node[midway, above, font=\footnotesize] {Causal} node[midway, below, font=\footnotesize] {Discovery};
% Discovered Graph (lower right)
\node[font=\small\bfseries] at (3.15, 0.15) {Discovered Graph};
\node[node] (d1) at (2.15, -0.8) {$x_1$};
\node[node] (d2) at (3.15, -0.8) {$x_2$};
\node[node] (d3) at (4.15, -0.8) {$x_3$};
% Edges in discovered graph
\draw[thick] (d1) to[bend right=45] (d3);
\draw[thick] (d2) -- (d3);

\end{tikzpicture}
\caption{Guess2Graph uses expert graph predictions to reorder test sequences in causal discovery subroutines. Example above for constraint-based discovery.}
\vspace{-\baselineskip}
\label{fig:g2g_framework}
\end{figure}
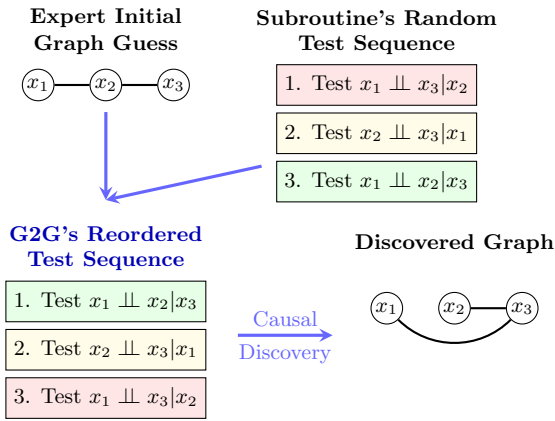

\vspace{-5pt}
\section{G2G in Constraint-Based Discovery}\label{sec: g2g cb theory}
\vspace{-5pt}

In this section, we instantiate the G2G framework for constraint-based methods \citep{spirtes_anytime_2001}. We focus specifically on skeleton discovery for three reasons: it bears the primary computational burden, suffices for many causal tasks, and improvements propagate to edge orientation since orientations derive from skeleton tests.
We observe that skeleton discovery decomposes into two core subroutines that rely on uniformly sampled orderings: \emph{Edge Prune} (EP) and \emph{Edge Loop} (EL). 
The EP subroutine (Subroutine~\ref{algo: edge-ordering}) tests edge $e_{i,j}$ with conditioning sets of size $k$ following subset ordering $\mathsf{L}$. 
The EL subroutine (Subroutine~\ref{algo: subset-ordering}) sequences edge testing following edge ordering $\mathsf{O}$ while iterating through conditioning set sizes, from $k_{\min}$ to $k_{\max}$, calling EP for each $k$.
For example, in the PC algorithm, $k_{\min}$ and $k_{\max}$ are both set to $\ell$ in iteration $\ell$. In PC, an edge $e_{i,j}$ is considered valid in iteration $\ell$ if $n_{i,j}$ remains in the current skeleton $\calC$ and the adjacency set of $x_i$ (excluding $x_j$) is sufficiently large to test conditioning sets of size $\ell$ (i.e., $|\adj_{-j}(\calC, x_i)| \geq \ell$).
PC calls the EL subroutine up to $|V|$ times.
% , each time with $k_{\min}=k_{\max}=\ell$ for iteration $\ell$.
% PC classifies $e_{i,j}$ as valid in iteration $\ell$ if $n_{i,j}\in C$ and $|\adj_{-j}(\calC, x_i)| > \ell$, calling EL with $k_{\min}=k_{\max}=\ell$ in iteration $\ell$. 
Different constraint-based algorithms vary in their invocation of these subroutines and validity rules $R$.
See Appendix~\ref{appendix: cb subroutine decomposition} for further decomposition details.

Under oracle conditional independence tests, EL and EP produce identical results regardless of orderings $\mathsf{O}, \mathsf{L}$ when starting from a complete graph (Lemmas~\ref{lemma: EL perfect}, \ref{lemma: EP perfect}). However, with finite-sample errors, orderings critically impact performance \citep{colombo2014order}: edge removals update the skeleton, changing adjacency sets for subsequent tests. While PC-Stable \citep{colombo2014order} addresses this by batching edge removals, we leverage this order-dependence for guidance. We develop modified subroutines that use expert predictions to generate $\mathsf{O}, \mathsf{L}$ by prioritizing edges and subsets the expert considers relevant, rather than using uniform random orderings.

\floatname{algorithm}{Subroutine}
\begin{algorithm}[t]
\caption{Edge Loop (EL)}\label{algo: edge-ordering}
\begin{algorithmic}[1]
\State\textbf{Input}: Current skeleton $\calC$, edge ordering $\mathsf{O}$, conditioning set sizes $[k_{\min}, \ldots, k_{\max}]$, subset ordering $\mathsf{L}$, validity rule $R$, EP subroutine
\For{$k = k_{\min}$ to $k_{\max}$}
    \For{each undirected edge $n_{i,j}$ in order $\mathsf{O}$}
        % \If{$n_{i,j} \in \calC$ and $R(\calC, e_{i,j}, k)$}
        \If{$R(\calC, e_{i,j}, k)$}
            \State $\calC \leftarrow \text{EP}(\calC, e_{i,j}, k, \mathsf{L})$
        \EndIf
    \EndFor
\EndFor
\State \Return $\calC$
% $\calC$, $\mathsf{O}$, $[k_{\min}, \ldots, k_{\max}]$, $\mathsf{L}$, $R$, EP
% $\{x_i\}^{|V|}_{i=1}$
% \Ensure $C$
% \State Following $\mathsf{L}$, for each edge $n_{i,j} \in \calC$:
% \State \quad While $R(\calC,e_{i,j}, k)$ is true, for $k = k_{\min}$ to $k_{\max}$, update $\calC\leftarrow$EP($\calC$, $e_{i,j}$, $k$, $\mathsf{L}$). If $n_{i,j}$ is not removed, repeat this step with pair $e_{j,i}$.\kyra{why $e_{j,i}$? how is $O$ used?}
% \State \Return Current skeleton $\calC$
\end{algorithmic}
\end{algorithm}

\floatname{algorithm}{Subroutine}  
\begin{algorithm}[t]
\caption{Edge Prune (EP)}\label{algo: subset-ordering}
\begin{algorithmic}[1]
\State \textbf{Inputs}: Current skeleton $\calC$, directed edge $e_{i,j}$, conditioning set size $k$, subset ordering $\mathsf{L}$
\State Let $A \leftarrow \adj_{-j}(\calC, x_i)$
\For{each subset $W \in [A]_k$ in order $\mathsf{L}$}
    \If{$\text{CIT}(x_i, x_j \mid W)$ returns independent}
        \State Remove $n_{i,j}$ from $\calC$
        \State \Return $\calC$
    \EndIf
\EndFor
\State \Return $\calC$
% \Require $C$, $e_{i,j}$, $k$, $\mathsf{L}$
% \Ensure $C$
% \State Following $\mathsf{L}$, test CIT($x_i, x_j | s$) for  $s\subseteq [\adj_{-j}(\calC, x_i)]_k$. Remove $n_{i,j}$ from $C$ if any CIT = I.
% \State \Return $C$
\end{algorithmic}
% \vspace{-\baselineskip}
\end{algorithm}

\vspace{-5pt}
\subsection{A Tractable Metric for Analyzing Ordering Effects}\label{subsec:metric}
\vspace{-5pt}

% \textbf{Selecting an Analytically Tractable Metric.} 
To analyze how orderings affect algorithm performance, we require a metric that captures correctness while remaining analytically tractable.
We evaluate algorithm performance by the probability of \emph{perfect recovery} of the true skeleton $\calS^*$.
For any candidate skeleton $\mathcal{C}$ produced by the algorithm, perfect recovery occurs when $\mathcal{C}$ and $\mathcal{S}^*$ are identical on all edges. For each edge $n_{i,j}$ (undirected, or directed with $e_{i,j}$), we define the correctness indicator:
$$Y_{n_{i,j}}=\mathbbm{1}\{n_{i,j}\in\calC \land n_{i,j}\in \calS^*\} + \mathbbm{1}\{n_{i,j}\notin \calC \land n_{i,j}\notin \calS^*\}.$$
This indicator equals $1$ when the edge $n_{i,j}$'s status in $\calC$ matches the ground truth in $\mathcal{S}^*$, and $0$ otherwise.
The perfect recovery probability is then defined as:
$$\Phi = \mathbb{P} \left[\prod_{n_{i,j}:i\neq j}Y_{n_{i,j}}=1\right],$$
representing the probability that all edges are correctly specified. 

This metric has two key analytical advantages. First, the perfect recovery probability factors into a product of conditional probabilities along the edge ordering $\mathsf{O}$ used by the algorithm. Abusing the notation, if edges are processed in order $n_1, n_2, \dots, n_m$, then: 
\begin{align}
    \Phi =& \mathbb{P}(Y_{n_1}=1)\cdot \mathbb{P}(Y_{n_2}=1|Y_{n_1}=1)\cdots\notag\\ &\cdot \mathbb{P}(Y_{n_m}=1|Y_{n_1}=1, \cdots, Y_{n_{m-1}}=1).\label{eq:perfect_recovery}
\end{align}
This factorization enables compositional analysis by studying the algorithm's behavior sequentially. Second, perfect recovery avoids error propagation entirely. Error propagation is difficult to analyze because false positive and false negative errors have opposing effects on adjacency sets---false positives inflate them while false negatives shrink them. This makes the overall impact of different errors dependent on the underlying graph structure, rendering the analysis of error-tolerant metrics (e.g., $\mathbb{E}[\sum Y_{n_{i,j}}]$) challenging without graph-specific assumptions (Appendix~\ref{appendix: edge loop error propogation}).

\vspace{-5pt}
\subsection{Guiding Edge Loop}\label{subsec: edge-ordering}
\vspace{-5pt}

We now apply this metric to develop a principled approach for guiding EL 
(Subroutine~\ref{algo: edge-ordering}). 
We start by
% Our approach proceeds in three steps: first, we 
characterizing how edge orderings affect the perfect recovery probability. This enables us to
% ; second, we 
identify ordering modifications that provably improve this metric for any underlying graphical structure.
% ; third, 
We demonstrate that leveraging an expert to guide these modifications leads to monotonic improvement with expert accuracy.

\textbf{Graph-independent Ordering Principles.}
Building on our perfect recovery metric $\Phi$, we investigate ordering principles that improve $\Phi$ regardless of the underlying graph structure.
Specifically, we analyze when correctly specifying edge $n_{i,j}$ first increases the probability of correctly specifying edge $n_{g,h}$ second---i.e., when $\mathbb{P}(Y_{n_{g,h}}=1 \mid Y_{n_{i,j}} = 1) > \mathbb{P}(Y_{n_{g,h}}=1)$. 

% We investigate when correctly classifying edge $n_{i,j}$ first increases the probability of correctly classifying edge $n_{g,h}$ second---i.e., when $\mathbb{P}(c_{n_{g,h}}=1 \mid c_{n_{i,j}} = 1) > \mathbb{P}(c_{n_{g,h}}=1)$. 
The key insight is an asymmetry in how false versus true edge decisions affect subsequent tests: correctly removing false edges reduces adjacency sets (simplifying future tests), while correctly retaining true edges leaves adjacency sets unchanged. Combined with the fact that true edges become easier to retain with smaller adjacency sets (Lemma~\ref{lemma: edge-asymmetry}), this implies that removing false edges first can only increase the probability of correctly retaining subsequent true edges, while retaining a true edge first does not affect the success probability of removing a false edge (Lemma~\ref{lemma: edge-ordering-asymmetry}). 

By placing false edges before true edges in $\mathsf{O}$, we can only increase 
the perfect recovery probability $\Phi$,
 % $\mathbb{P}\left[\bigcap_{n_{i,j} \in \mathsf{O}} c_{n_{i,j}} = 1\right]$,
which we formalize in Lemma~\ref{lemma: false-before-true}: for any ordering $\mathsf{O}$, swapping adjacent edges to place a false edge before a true edge is never worse and sometimes strictly better.

\textbf{Expert-guided Algorithm with Monotonicity Guarantees.} Based on Lemma~\ref{lemma: false-before-true}, we modify the EL subroutine to incorporate expert predictions as shown in Subroutine~\ref{algo: expert-edge-ordering1}. 
The algorithm operates as follows: given an expert graph $\widehat{\mathcal{G}}$, it extracts the skeleton $\widehat{\mathcal{S}}$ and partitions the current skeleton $\mathcal{C}$ into edges the expert believes are false ($\mathcal{C} \setminus \widehat{\mathcal{S}}$) and true ($\mathcal{C} \cap \widehat{\mathcal{S}}$), then processes the false edges before the true edges while maintaining random order within each group.

% We model the expert $\psi$ as a binary symmetric channel: for any edge $n_{i,j}$, the expert independently predicts whether it exists in the true skeleton $\mathcal{S}^*$ with accuracy $p_{edge}$. Under this model, we establish the following monotonicity guarantee:
We make the assumption that the expert $\psi$ acts as a symmetric binary channel: for any edge $n_{i,j}$, the expert independently predicts if it exists in the true skeleton $\mathcal{S}^*$ with accuracy 
$p_\psi$. 
% $p_{\text{edge}}$. 
Under this assumption, we establish the following monotonicity guarantee:
\floatname{algorithm}{Subroutine}
\begin{algorithm}[t]
\caption{Edge Loop Guess (EL-G)}\label{algo: expert-edge-ordering1}
\begin{algorithmic}[1]
\State \textbf{Inputs}: Current skeleton $\calC$, expert graph $\hcalG$, conditioning set sizes $[k_{\min}, \ldots, k_{\max}]$, subset ordering $\mathsf{L}$, validity rule $R$, EP subroutine
\State Extract skeleton $\widehat{\mathcal{S}}$ from $\hcalG$. Randomly order $\calC$, then set $\mathsf{O} = \calC \setminus \widehat{\mathcal{S}} + \calC \cap \widehat{\mathcal{S}}$
\State \Return EL($\calC$, $\mathsf{O}$, $[k_{\min}, \ldots, k_{\max}]$, $\mathsf{L}$, $R$, EP)
\end{algorithmic}
\end{algorithm}

% \begin{restatable}[]{lemma}{EdgeOrderingAccuracyMonotonicity}\label{lemma: edge-accuracy-monotonic}
% Let $P^{\calC,\mathcal{G}^*}_{n,p_{edge}}$ denote the distribution over $\mathbb{P}\left[\bigcap_{n_{i,j} \in \mathsf{O}} c_{n_{i,j}} = 1\right]$ induced by Subroutine~\ref{algo: expert-edge-ordering1} when applied to $n$ samples from true graph $\mathcal{G}^*$, with skeleton prediction $\widehat{\mathcal{S}}$ generated by expert $\psi$ with edge prediction accuracy $p_{edge}$. Then $\mathbb{E}[P^{\calC,\mathcal{G}^*}_{n,p_{edge}}]$ increases monotonically with $p_{edge}$, strictly increasing when $\calC$ contains false edges adjacent to true edges.
% \end{restatable}

\begin{restatable}[Monotonicity of Perfect Recovery in Expert Accuracy]{lemma}{EdgeOrderingAccuracyMonotonicity}\label{lemma: edge-accuracy-monotonic}
For a fixed partial skeleton $\calC$ and true DAG $\mathcal{G}^*$, 
let $\Phi_{\text{EL-G}}(p_{\psi})$ denote the perfect recovery probability when we sample an expert graph $\widehat{\mathcal{G}}$ from expert $\psi$ with accuracy $p_{\psi}$, draw $n$ samples from $\mathcal{G}^*$, and run EL-G. 
Then $\mathbb{E}[\Phi_{\text{EL-G}}(p_{\psi})]$ increases monotonically with $p_{\psi}$, strictly increasing when $\mathcal{C}$ contains false edges adjacent to true edges.
\end{restatable}
% let $p_{true}$ be the probability of perfect classification of all edges after sampling a guess $\hcalG$ from an expert $\psi$ with accuracy $p_{edge}$, drawing $n$ samples from $\mathcal{G}^*$, and running EL-G. Then $\mathbb{E}[p_{true}]$ increases monotonically with $p_{edge}$, strictly increasing when $\calC$ contains false edges adjacent to true edges.

\textbf{Proof sketch.} 
The proof (App.~\ref{proof: edge-accuracy-monotonic}) establishes monotonicity via a coupling argument between experts with accuracies $p_{\psi_1} < p_{\psi_2}$. Both experts observe the same true skeleton $\mathcal{S}^*$ and use identical randomness for edge classification, but the higher-accuracy expert makes fewer errors. This ensures that every edge correctly classified by the weaker expert is also correctly classified by the stronger expert.
Consequently, the better expert's edge ordering has fewer ``inversions'' (true edges incorrectly placed before false edges). These orderings are related by the weak Bruhat order, meaning the better ordering can be obtained through a sequence of adjacent swaps that move false edges leftward past true edges. By Lemma~\ref{lemma: false-before-true}, each such swap weakly improves the perfect recovery probability $\Phi_{\text{EL-G}}$, with strict improvement when swapped edges share vertices (since removing false edges first shrinks adjacency sets and reduces false negative probabilities).
Since the better expert's ordering is reachable through beneficial swaps, it achieves pointwise improvement for any fixed realization of data and expert predictions. Strassen's Coupling theorem then implies that $\Phi_{\text{EL-G}}(p_{\psi_2})$ stochastically dominates $\Phi_{\text{EL-G}}(p_{\psi_1})$, yielding the monotonicity in expectation.

\vspace{-5pt}
\subsection{Guiding Edge Prune}\label{subsec: subset-ordering}
\vspace{-5pt}

We now develop a principled approach to guiding EP (Subroutine~\ref{algo: subset-ordering}). Unlike EL, where ordering affects accuracy, EP's ordering only impacts runtime, which can also be decreased by leveraging expert predictions. 
Since our focus is on accuracy improvements through expert guidance, we briefly outline how expert predictions can accelerate EP by prioritizing promising conditioning sets, and defer the complete theoretical analysis of runtime to Appendix~\ref{appendix: subset ordering details}.
% However, as our paper focuses on improving accuracy with expert guidance rather than runtime, we present only key details here and defer the full theoretical treatment to Appendix~\ref{appendix: subset ordering details}. 

The EP subroutine requires an ordering $\mathsf{L}$ to sequence conditional independence tests for edge $e_{i,j}$. Given adjacency set $\adj_{-j}(\calC, x_i)$ and conditioning set size $k$, let $\text{CIT}^k_{\adj_{-j}(\calC, x_i)} := \{\text{CIT}(x_i, x_j \mid W): W \subseteq \adj_{-j}(\calC, x_i), |W|=k\}$ denote all possible independence tests of size $k$.
Since EP removes an edge only if any test in $\text{CIT}^k_{\adj_{-j}(\calC, x_i)}$ returns independence, 
and the timing of test execution does not affect test outcomes, the edge recovery accuracy $\mathbb{P}(Y_{e_{i,j}}=1)$ 
% and the timing of when a test is run does not affect whether it returns independence, accuracy $\mathbb{P}(c_{n_{i,j}}=1)$ 
remains constant across all orderings (Lemma \ref{lemma: orderingindaccuracy}).

% proof in Appendix~\ref{proof: orderindaccuracy}):

% \begin{restatable}[]{lemma}{orderingindaccuracy}\label{lemma: orderingindaccuracy}
% Accuracy $\mathbb{P}(c_{n_{i,j}}=1)$ is constant for all possible orderings $\mathsf{L}_1, \mathsf{L}_2, \ldots$ of $\text{CIT}^k_{\adj_{-j}(\calC, x_i)}$.
% \end{restatable}

Although EP orderings cannot affect accuracy, they can impact computational runtime (Lemma~\ref{lemma: no-dseps}). Orderings that place d-separating sets of $x_i, x_j$ earlier achieve lower expected runtimes (Lemma~\ref{lemma: optimal-sequence}). We therefore guide EP by prioritizing conditioning sets predicted to d-separate $x_i, x_j$ according to the expert's graph $\widehat{\mathcal{G}}$, as formalized in Subroutine~\ref{algo: expert-subset-ordering1}.

We make similar assumptions in d-separation prediction analogously to edge prediction: the expert acts as a binary symmetric channel with accuracy $p_{\text{d-sep}}$ for identifying d-separating sets. Under a common technical condition (Definition~\ref{assumption:cit_power_independence}) from testing literature \citep{brown2008strategy, li2009controlling, strobl2019estimating}, a coupling argument similar to Lemma~\ref{lemma: edge-accuracy-monotonic} shows that EP-Guess's expected runtime decreases monotonically with $p_{\text{d-sep}}$ (Lemma~\ref{lemma: expert-accuracy}).

% Although we cannot optimize for accuracy, runtime may still vary across orderings (Lemma~\ref{lemma: no-dseps}), with orderings that place d-separating sets of $x_i, x_j$ at the front achieving lower expected runtimes (Lemma~\ref{lemma: optimal-sequence}). We propose to guide EP by moving subsets predicted to d-separate $x_i, x_j$ according to the expert's graph $\hcalG$ to the front of the ordering as shown in Subroutine~\ref{algo: expert-subset-ordering1}. We model d-separation prediction similar to edge prediction, i.e. as a binary symmetric channel with accuracy $p_{dsep}$. A coupling argument (similar to Lemma~\ref{lemma: edge-accuracy-monotonic}) shows that, under a simplifying technical condition (Definition \ref{assumption:cit_power_independence}) made commonly in the literature \citep{brown2008strategy, li2009controlling, strobl2019estimating}, the runtime of EP-Guess decreases monotonically with $p_{dsep}$ (Lemma \ref{lemma: expert-accuracy}).

% since placing true d-separating sets earlier increases early termination probability (Lemma~\ref{lemma: expert-accuracy}).

% which requires a) CITs are adequately powered, and b) mutually independent, 

\floatname{algorithm}{Subroutine}
\begin{algorithm}[t]
\caption{Edge Prune Guess (EP-G)}\label{algo: expert-subset-ordering1}
\begin{algorithmic}[1]
\State \textbf{Inputs}: Current skeleton $\calC$, directed edge $e_{i,j}$, conditioning set size $k$, expert graph $\hcalG$
\State Let $A \leftarrow \adj_{-j}(\calC, x_i)$
\State Extract all d-sep. sets $\widehat{\mathcal{D}}_{ij}$ from $\hcalG$. Randomly order $[A]_k$, then set $\mathsf{L} = [A]_k \cap \widehat{\mathcal{D}}_{ij} + [A]_k \setminus \widehat{\mathcal{D}}_{ij}$
\State \Return EP($\calC$, $e_{i,j}$, $k$, $\mathsf{L}$)
\end{algorithmic}
\end{algorithm}

\vspace{-5pt}
\section{Expert Augmented Algorithms}\label{sec: method guarantees}
\vspace{-5pt}
% \kyra{here}
We introduce PC-Guess (Alg.~\ref{algo: guess-pc}, Section~\ref{subsec: pcg2g}) and gPC-Guess (Alg.~\ref{algo: guess-sgs}, Section~\ref{subsec: lapcg2g}), which implement the expert-guided framework from Section~\ref{sec: g2g cb theory} using a unified extraction subroutine (Subroutine \ref{algo: extract-orderings}) that generates both orderings $\mathsf{O}$ and $\mathsf{L}$ from $\widehat{\mathcal{G}}$. We provide theoretical guarantees in Section~\ref{subsec: theory guarantees}.
% In this section, we leverage the modified subroutines from Section~\ref{sec: g2g cb theory} to create expert-augmented constraint-based algorithms: PC-Guess (Alg.~\ref{algo: guess-pc}) and gPC-Guess (Alg.~\ref{algo: guess-sgs}). Note, we summarize the sequence extraction from EP-G and EL-G into a single subroutine (Subroutine \ref{algo: extract-orderings}) which obtains $\mathsf{O}, \mathsf{L}$ from $\widehat{\mathcal{G}}$.

% We first note that the sequence extraction steps that take place in EP-G and EL-G can be summarized in a single subroutine (Subroutine \ref{algo: extract-orderings}, see Appendix \ref{appendix: extractorderingsubroutine}) that obtains orderings from expert predictions.

% We now leverage this subroutine in PC-Guess and its augmented variant gPC-Guess.

\vspace{-5pt}
\subsection{PC-Guess}\label{subsec: pcg2g}
\vspace{-5pt}

We instantiate the G2G framework as PC-Guess (Alg.~\ref{algo: guess-pc}), which modifies PC's skeleton discovery by integrating expert guidance through orderings $\mathsf{O}$ and $\mathsf{L}$ from Subroutine~\ref{algo: extract-orderings}. The algorithm then iterates through conditioning set sizes $\ell$, at each stage calling EL with validity rule $R_{PC}^{\ell}$ that determines whether an edge qualifies for testing by checking two conditions: (1) the edge remains in the current skeleton $\mathcal{C}$, and (2) at least one endpoint's adjacency set (excluding the other endpoint) has size at least $\ell$ to enable conditioning sets of size $\ell$.
\floatname{algorithm}{Algorithm}
\begin{algorithm}[t]
\caption{PC-Guess}
\label{algo: guess-pc}
\begin{algorithmic}[1]
\State \textbf{Inputs}: Expert $\psi$, complete skeleton $\calC$
\State $\mathsf{O}$, $\mathsf{L} \leftarrow$ Subroutine \ref{algo: extract-orderings}($\psi$, $\calC$)
\State Define 
$R_{PC}^{\ell}(\calC, e_{i,j}) = n_{i,j} \in \calC \land |\adj_{-j}(\calC, x_i)| \geq \ell$
\For{$\ell = 0$ to $d-1$}
    \State $\calC \leftarrow$ EL($\calC$, $\mathsf{O}$, [$\ell$, $\ell$], $\mathsf{L}$, $R_{PC}^{\ell}$, EP)
    \State \textbf{if} no edges satisfy $R_{PC}^{\ell}$ \textbf{then break}
\EndFor
\State \Return $\calC$
\end{algorithmic}
\end{algorithm}

PC-Guess inherits PC's iterative structure for processing conditioning sets, which constrains how expert predictions are utilized. Due to the curse of dimensionality \citep{li2020nonparametric}, conditional independence tests become less reliable with larger conditioning sets. PC addresses this through a "statistical conditioning" bias that prioritizes smaller conditioning sets first—a common design pattern in causal discovery algorithms.
% \footnote{This bias toward smaller cond. sets appears in other methods \citep{chickering_statistically_2020, suj_2025_losam}.}
% % score-based \citep{chickering_statistically_2020} and FCM \citep{suj_2025_losam} approaches.}

While this approach maximizes test reliability without expert knowledge, it limits the potential benefit of expert predictions: false edges requiring larger conditioning sets for removal cannot be eliminated early, forcing unnecessary tests at lower conditioning levels and inflating adjacency sets for neighboring edges (Appendix \ref{appendix: pcg2g suboptimal}). This limitation motivates removing the level-by-level constraint to enable earlier removal of false edges with non-trivial minimal d-separating sets.

\subsection{gPC-Guess}\label{subsec: lapcg2g}
To address the limitations of PC's level-by-level constraint, we propose gPC-Guess (Alg. \ref{algo: guess-sgs}), which enables immediate action on expert predictions. Like PC-Guess, gPC-Guess extracts orderings $\mathsf{O}$ and $\mathsf{L}$ from $\widehat{\mathcal{G}}$ via Subroutine~\ref{algo: extract-orderings}, but replaces PC's iterative structure with a single-pass approach that tests all edges using conditioning sets from size $0$ to $|V|-1$.

This design eliminates the statistical conditioning bias in favor of expert responsiveness: gPC-Guess uses a simplified validity rule $R_{gPC}$ that only checks edge presence, allowing false edges with non-trivial minimal d-separating sets to be removed immediately when placed early in $\mathsf{O}$. While this can reduce adjacency set inflation and improve accuracy with good expert guidance, it risks testing edges with unnecessarily large conditioning sets when expert predictions are inaccurate, potentially compromising test reliability.

% To enable algorithms to act on expert predictions immediately rather than waiting through multiple conditioning set levels, we propose gPC-Guess (Alg. \ref{algo: guess-sgs}). 
% Similar to PC-Guess, gPC-Guess uses Subroutine~\ref{algo: extract-orderings} to obtain orderings $\mathsf{O}$ and $\mathsf{L}$ from the expert's graph $\hcalG$. However, unlike PC-Guess, which iterates through up to $d$ loops with fixed conditioning set size $\ell$ per loop and only tests edges with adjacency sizes exceeding $\ell$, gPC-Guess uses a simpler validity rule $R_{gPC}$ that checks only edge presence, processing all edges in a single pass with conditioning sets ranging from size $0$ to $|V|-1$. This design trades PC's statistical conditioning bias for responsiveness to expert predictions. With accurate expert guidance, gPC-Guess can more quickly remove false edges placed earlier in $\mathsf{O}$ even when their minimal d-separating sets are non-zero, reducing adjacency set sizes for subsequently tested edges and potentially improving accuracy. Without accurate guidance, however, gPC-Guess may test edges with unnecessarily large conditioning sets earlier than statistically optimal, reducing test reliability.

\floatname{algorithm}{Algorithm}
\begin{algorithm}[t]
\caption{Guided PC (gPC-Guess)}
\label{algo: guess-sgs}
\begin{algorithmic}[1]
\State \textbf{Inputs}: Expert  $\psi$, complete skeleton $\calC$
\State $\mathsf{O}$, $\mathsf{L} \leftarrow$ Subroutine \ref{algo: extract-orderings}($\psi$, $\calC$)
\State Define $R_{gPC}(\calC, e_{i,j}) = n_{i,j} \in \calC$
\State $\calC \leftarrow$ EL($\calC$, $\mathsf{O}$, [0, $|V|-1$], $\mathsf{L}$, $R_{gPC}$, EP)
\State \Return $\calC$
\end{algorithmic}
\end{algorithm}
\setcounter{algorithm}{4}

\subsection{Theoretical Guarantees}\label{subsec: theory guarantees}

\textbf{Correctness and Statistical Consistency (\ref{criterion:consistency}).} Both PC-Guess and gPC-Guess maintain the theoretical guarantees of their base algorithms regardless of expert quality, converging to the true graph with consistent conditional independence tests:
% . Under a consistent conditional independence test, the probability of recovering the true graph converges to 1 as $n\rightarrow \infty$.

\begin{restatable}[Asymptotic Correctness]{theorem}{GuessPCCorrectness}\label{theorem: guess acc correct}
    Under a consistent conditional independence test, for both PC-Guess and gPC-Guess, $\lim_{n \to \infty} \mathbb{P}(\widetilde{\mathcal{G}} = \mathcal{G}^*) = 1$. 
    
    % With oracle (no error) CITs, both algorithms are sound and complete.
\end{restatable}
The proof (Appendix~\ref{proof: guess acc correct}) follows from both algorithms eventually considering all possible edges and CITs regardless of ordering. This ensures asymptotic performance is never compromised by potentially poor expert guidance, satisfying \ref{criterion:consistency}.
% This implies that if you have access to an oracle CI test, both PC-Guess and gPC-Guess are sound and complete. The proof (Appendix \ref{proof: guess acc correct}) follows from both algorithms eventually considering all possible edgeals and CITs regardless of ordering. This correctness guarantee ensures that asymptotic performance is never compromised by expert guidance, satisfying \ref{criterion:consistency}.

\textbf{Monotonic Improvement and Finite-Sample Robustness (\ref{criterion:monotone}, \ref{criterion:finite-sample}).}
We characterize how expert quality affects finite-sample performance, with per-iteration guarantees for PC-Guess and end-to-end guarantees for gPC-Guess. Consider a fixed DAG $\mathcal{G}^*$ with $d$ variables and expert $\psi$ with edge accuracy $p_\psi$ and d-separation accuracy $p_{\text{d-sep}}$, drawing $n$ samples from $\mathcal{G}^*$ and prediction $\widehat{\mathcal{G}}$ from $\psi$. We compare guided variants (PC-Guess, gPC-Guess) against unguided baselines (PC, gPC) denoted with overbars $\bar{(\cdot)}$.

% We now characterize how expert quality affects finite-sample performance. We present per-iteration guarantees for PC-Guess and end-to-end guarantees for gPC-Guess. Throughout, we consider a fixed DAG $\mathcal{G}^*$ with $d$ variables and expert $\psi$ with edge prediction accuracy $p_{\text{edge}}$ and d-separating set accuracy $p_{\text{dsep}}$, and draw $n$ samples from $\mathcal{G}^*$ and a prediction $\hcalG$ from $\psi$. We say an edge is \emph{perfectly classified} if it is correctly kept when true or correctly removed when false. We denote guided variants (PC-Guess, gPC-Guess) without bars and unguided baselines (PC, gPC) with bars $\bar{(\cdot)}$. 
% We use overbar notation to distinguish baseline algorithms: unbarred quantities refer to guided variants (PC-Guess, gPC-Guess), while barred quantities $\bar{(\cdot)}$ refer to unguided variants using uniform random orderings (PC, gPC).

For PC-Guess, fix iteration $\ell \in  
\{0, 1, \ldots, d-1\}$
and suppose both algorithms start with an identical partial skeleton $\mathcal{C}$. Let $\Phi_\ell$ and $\bar{\Phi}_\ell$ denote perfect recovery probabilities in iteration $\ell$ for PC-Guess and PC, respectively.
Let $t_\ell$ denote the number of tests run in PC-Guess.
\begin{restatable}[Performance of PC-Guess]{theorem}{PCGuessIterationPerformance}\label{theorem: pc performance}
PC-Guess 
% achieves both accuracy and runtime improvements that scale with expert quality per-iteration, 
satisfies \ref{criterion:monotone}-\ref{criterion:finite-sample} at per-iteration level:
\textbf{(a)} $\mathbb{E}[\Phi_\ell]$ increases monotonically with $p_\psi$; \textbf{(b)} For fixed $p_\psi$, $\mathbb{E}[t_\ell]$ decreases monotonically with $p_{\text{d-sep}}$; \textbf{(c)} When $p_\psi \geq 0.5$, $\mathbb{E}[\Phi_\ell] \geq \mathbb{E}[\bar{\Phi}_\ell]$.
% Fix any iteration $\ell \in \{0, 1, \ldots, d-1\}$. Suppose that at the start of iteration $\ell$ both PC-Guess and PC are provided with the same partial skeleton $\mathcal{C}$. Let $p_{\text{true}}^\ell$ and $\bar{p}_{\text{true}}^\ell$ denote the probability of perfectly classifying all edges tested in iteration $\ell$ of PC-Guess and PC respectively, and let $t_\ell$ denote the number of tests run in PC-Guess. 
% Then: 
% \textbf{(a)} $\mathbb{E}[p_{\text{true}}^\ell]$ increases monotonically with $p_{\text{edge}}$; \textbf{(b)} For each fixed $p_{\text{edge}}$, $\mathbb{E}[t_\ell]$ decreases monotonically with $p_{\text{dsep}}$; \textbf{(c)} When $p_{\text{edge}} \geq 0.5$, $\mathbb{E}[p_{\text{true}}^\ell] \geq \mathbb{E}[\bar{p}_{\text{true}}^\ell]$.
\end{restatable}

Let $\Phi$ and $\bar{\Phi}$ denote perfect skeleton recovery probabilities for gPC-Guess and gPC, respectively, and let $t$ denote the total number of tests run in gPC-Guess.
\begin{restatable}[Performance of gPC-Guess]{theorem}{gPCGuessPerformance}\label{theorem: sgs performance}
gPC-Guess
% provides end-to-end accuracy and efficiency guarantees,
satisfies \ref{criterion:monotone}-\ref{criterion:finite-sample}: \textbf{(a)} $\mathbb{E}[\Phi]$ increases monotonically with $p_\psi$; \textbf{(b)} For fixed $p_\psi$, $\mathbb{E}[t]$ decreases monotonically with $p_{\text{d-sep}}$; \textbf{(c)} When $p_\psi \geq 0.5$, $\mathbb{E}[\Phi] \geq \mathbb{E}[\bar{\Phi}]$.
% Let $p_{\text{true}}$ and $\bar{p}_{\text{true}}$ denote the probability that gPC-Guess and gPC correctly recover the skeleton respectively, and let $t$ denote the total number of tests run in gPC-Guess. Then: \textbf{(a)} $\mathbb{E}[p_{\text{true}}]$ increases monotonically with $p_{\text{edge}}$; \textbf{(b)} For each fixed $p_{\text{edge}}$, $\mathbb{E}[t]$ decreases monotonically with $p_{\text{dsep}}$; \textbf{(c)} When $p_{\text{edge}} \geq 0.5$, $\mathbb{E}[p_{\text{true}}] \geq \mathbb{E}[\bar{p}_{\text{true}}]$.
\end{restatable}
Proofs appear in App. \ref{proof: pc performance} and \ref{proof: sgs performance}, following directly from Lemmas \ref{lemma: edge-accuracy-monotonic} and \ref{lemma: expert-accuracy}.
\begin{remark}
% vacuous graphs (not empty nor fully connected). 
All montonicity relations and inequalities are strict for nonempty and non-fully connected graphs. The key distinction lies in guarantee scope: while PC-Guess achieves improvements per iteration, these may not compose into end-to-end guarantees due to complex cascading effects across iterations. For example, correctly removing false edges early helps retain true edges later but may make it harder to remove persistent false edges. In contrast, gPC-Guess provides end-to-end guarantees, fully satisfying criteria \ref{criterion:consistency}-\ref{criterion:finite-sample} for the final output. This reflects the fundamental tradeoff between PC-Guess's statistical conditioning bias and gPC-Guess's expert responsiveness.
% All inequalities are strict when the graph is non-vacuous—i.e., when it is neither empty nor fully connected.
% It is unclear whether PC-Guess's 
% The key distinction lies in the scope of guarantees: PC-Guess achieves all three parts (a)-(c) at the iteration level—with (a)-(b) partially satisfying \ref{criterion:monotone} and (c) partially satisfying \ref{criterion:finite-sample}—but it remains unclear whether these 
% per-iteration guarantees compose into end-to-end guarantees for the final algorithm output. This is because improving early iterations can have complex cascading effects on later iterations: removing false edges helps retain true edges later, but may make it harder to remove false edges that persist. In contrast, gPC-Guess satisfies all \ref{criterion:consistency}-\ref{criterion:finite-sample}, reflecting their design tradeoffs.
% we show that gPC-Guess achieves all three parts (a)-(c) at the algorithm level for its final output—with (a)-(b) fully satisfying \ref{criterion:monotone} and (c) fully satisfying \ref{criterion:finite-sample}. The algorithmic redesign of gPC-Guess, rather than simple augmentation of PC-Guess, results in stronger guarantees.
% monotonicity and robustness guarantees.
\end{remark}

\section{Experiments}\label{sec: exps}

\begin{figure*}[t!]
    \centering
    \begin{subfigure}[b]{0.32\textwidth}
        \centering
        \includegraphics[width=\textwidth]{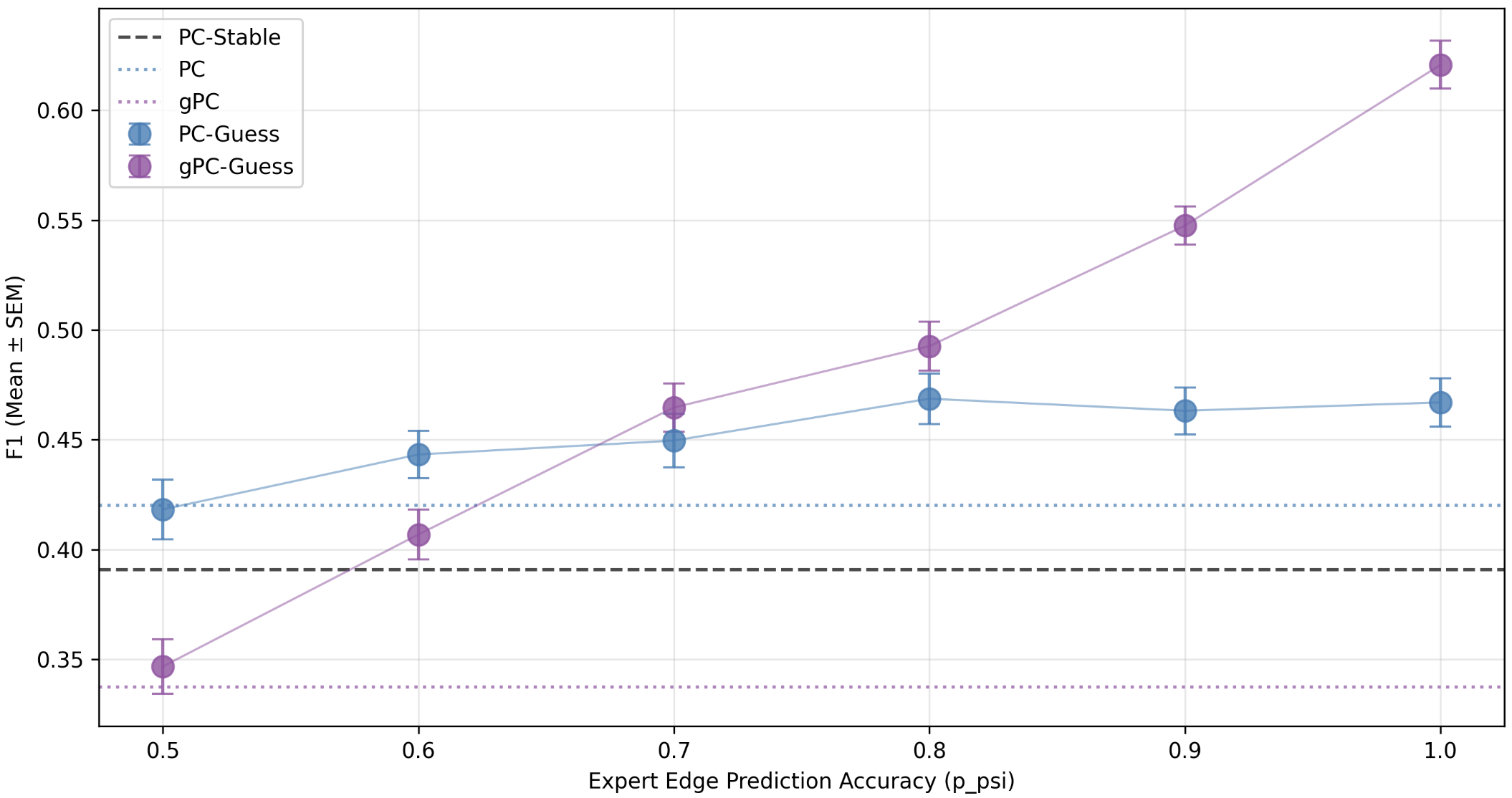}
        \caption{ER3 data with simulated expert.}
        \label{fig:first}
    \end{subfigure}
    \begin{subfigure}[b]{0.32\textwidth}
        \centering
        \includegraphics[width=\textwidth]{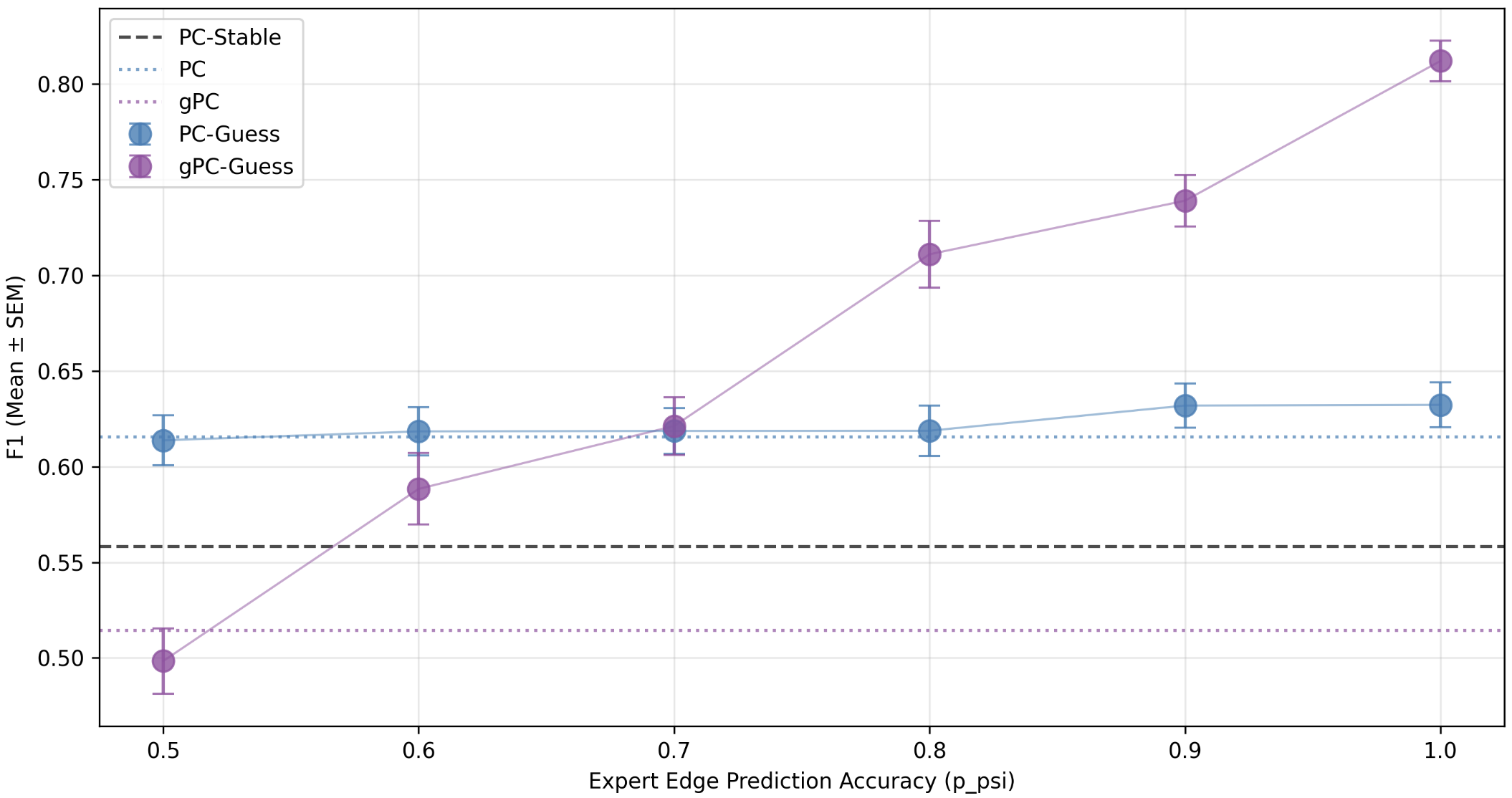}
        \caption{Sachs data with simulated expert.}
        \label{fig:second}
    \end{subfigure}
    \begin{subfigure}[b]{0.32\textwidth}
        \centering
        \includegraphics[width=\textwidth]{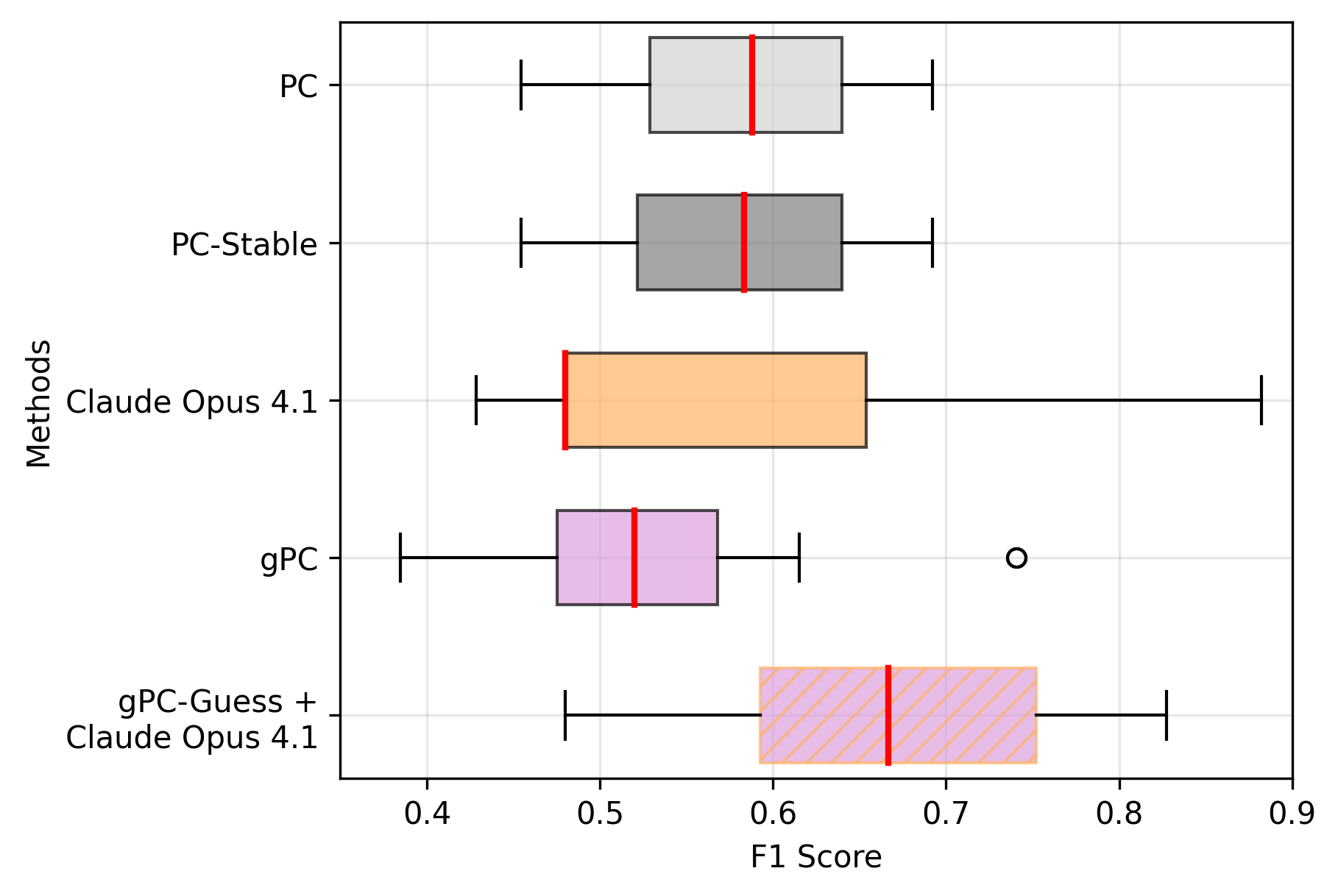}
    \caption{Sachs data with LLM expert.}\label{fig:third}
    \end{subfigure}
\caption{Performance improvement of PC-Guess and gPC-Guess with expert guidance.}
    % \hfill  % <-- Horizontal spacing
\end{figure*}
% We warm-up using synthetic data and synthetic experts, then investigate real-world performance using a LLM expert. Our results reflect the theoretical insights in Section \ref{sec: method guarantees}, as gPC-Guess and PC-Guess increase monotonically in performance, with gPC-Guess exhibiting stronger gains with more accurate experts. Further, gPC-Guess boosted by a recent state-of-the-art LLM model outperforms both the LLM and purely data-driven methods on a real-world dataset.

We evaluate how our algorithms satisfy criteria \ref{criterion:consistency}-\ref{criterion:finite-sample} through experiments with synthetic and real-world data. Our results validate that gPC-Guess fully achieves \ref{criterion:consistency}-\ref{criterion:finite-sample}, while PC-Guess empirically exceeds its guarantees. LLM-boosted gPC-Guess outperforms both the LLM alone and data-driven methods.
% Our results demonstrate that gPC-Guess fully achieves all three criteria, while PC-Guess shows empirical improvement beyond its theoretical guarantees. Further, gPC-Guess boosted by a recent state-of-the-art LLM model outperforms both the LLM and purely data-driven methods on a real-world dataset.

% We vary sample size to assess statistical consistency (\ref{criterion:consistency}) and expert accuracy 
% to measure monotonic improvement (\ref{criterion:monotone}), and identify accuracy thresholds for finite-sample robustness (\ref{criterion:finite-sample}). 

% We study our algorithms across additional settings in Appendix \ref{appendix: additional experimental results}.

% see that our hybrid method outperforms data-only approaches when the expert guess is better than random, and is robust to poor expert performance.

% confirming the existence of our 3-world theory, and 

% We test two expert types: (1) simulated experts $\mathcal{E}$ with manually-tuned prediction accuracy (we focus on better than random edge prediction, i.e. $p_{\text{edge}}\geq 0.5$, and hold constant d-separating set prediction accuracy as entirely random, i.e. $p_{dsep}=0.5$), and (2) a LLM expert, specifically Claude Opus 4.1 \citep{anthropic2025claude4} as a state-of-the-art model.

\textbf{Datasets and Experts.}
We evaluate on synthetic data (linear Gaussian models on Erdos-Renyi graphs, \citealt{erdos_renyi}) and real-world benchmarks. Experiments in the main text focus on sparse ER graphs ($d=20$ variables, $n=100$ samples) and subsampled Sachs protein data \citep{sachs2005causal} ($d=11, n=100$). We test two expert types: (1) simulated experts $\psi$ with manually-tuned prediction accuracy (we focus on better than random edge prediction, i.e., $p_{\psi} \geq 0.5$ to validate \ref{criterion:monotone}, and hold constant d-separating set prediction accuracy as entirely random, i.e., $p_{\text{d-sep}}=0.5$), and (2) a LLM expert, specifically Claude Opus 4.1 \citep{anthropic2025claude4}. Appendix \ref{appendix: experimental details} provides full simulation parameters, information about real data, and details on how predictions are generated (for both simulated and LLM experts). We explore additional experiments in App. \ref{appendix: additional experimental results}: varying sample size to validate \ref{criterion:consistency} (statistical consistency), dimensionality, d-separating accuracy ($p_{\text{d-sep}}$), and worst-case performance with experts below the \ref{criterion:finite-sample} threshold ($p_{\psi} \leq 0.5$).

\textbf{Methods and Metrics.}
We compare PC-Guess and gPC-Guess against the order-independent baseline PC-Stable \citep{ramsey2006adjacency}. We include PC and gPC as baseline versions of PC-Guess and gPC-Guess where the edge predictions supplied to both methods are uniformly sampled, i.e. $p_{\psi}=0.5$. All methods use identical CI tests (see Appendix \ref{appendix: implementations}), with $\alpha=0.05$. Performance is measured using skeleton F1 scores (see Appendix \ref{appendix: runtime results} for runtime).

% which does not receive input from an expert.
% (Appendix \ref{appendix: precision/recall results} reports precision/recall breakdowns) 

% For our approach, we initialize a guess of the DAG from an expert, and run SGS-S-GUESS and PC-S-GUESS. For baseline comparators, we use the original PC and SGS algorithms, as well as the order-independent PC extension, PC-stable \citep{ramsey2006adjacency}. We evaluate the accuracy of DAG guesses and final outputs by evaluating the F1 score of the edges predicted in to be in the underlying skeleton (see Appendix \ref{} for precision and recall results). Additionally, we record the total runtime for method, to get a sense of the computational burden changes with expert accuracy.

% We note that the test sequence guided-by-expert-guess approach we develop can extend any order-dependent constraint based algorithm, such as an alternative variant of PC (ex: PC-max, Maj-PC, conservative PC, etc. \citep{ramsey2016improving, ramsey2006adjacency,colombo2014order}). All methods use the same CI test (details in Appendix \ref{appendix: implementations}).

% \paragraph{Methods.}

% \sujai{final paper will include multiple variants of the guess-PC algorithm for ablation}

\textbf{Simulated Expert Results.}
Figures \ref{fig:first} and \ref{fig:second} demonstrate how augmented algorithm performance changes as synthetic experts provide increasingly accurate edge predictions on synthetic and real-world datasets. Figure \ref{fig:first} shows that on synthetic data, both PC-Guess and gPC-Guess increase monotonically in F1 score with expert accuracy (verifying \ref{criterion:monotone}). Despite PC-Guess only having theoretical guarantees for per-iteration improvement, it empirically achieves monotonic improvement that is no worse than baseline for $p_{\psi} \geq 0.5$ (empirically satisfying \ref{criterion:finite-sample}). As predicted by theory, gPC-Guess benefits more from guidance, achieving the highest accuracy when expert quality is sufficient ($p_{\psi} \geq 0.7$). These results replicate in Figure \ref{fig:second} on the Sachs real-world dataset: we again observe monotonic gains in both algorithms, but PC-Guess remains relatively flat while gPC-Guess's F1 increases by over 30 percentage points, confirming that algorithmic redesign is necessary to fully realize \ref{criterion:monotone}.

% Figures \ref{fig:first} and \ref{fig:second} demonstrate how augmented algorithm performance changes as synthetic experts provide increasingly accurate edge predictions on synthetic and real-world datasets. Figure \ref{fig:first} shows that on synthetic data, both PC-Guess and gPC-Guess increase monotonically in F1 score with expert accuracy, with gPC-Guess benefiting more from guidance and achieving the highest accuracy when expert quality is sufficient ($p_{\text{edge}} = 0.7$). These results replicate in Figure \ref{fig:second} on the Sachs real-world dataset: we again observe monotonic gains in both algorithms, but PC-Guess remains relatively flat while gPC-Guess's F1 increases by over 30\%. This shows that redesigning algorithms for expert guidance yields substantial performance gains with accurate experts.

% an inferior baseline algorithm modific can benefit tremendously from expert knowledge, calling into question how algorithms should be compared. 

% We note the runtime of both PC-GUESS and UPC-GUESS also monotonically decrease \sujai{add runtime discussion in Appendix \ref{appendix: runtime results}.}

\textbf{LLM Expert Results.}
Figure \ref{fig:third} explores how DAG guesses from real-world expert Claude Opus 4.1 benefit our augmented algorithms on the Sachs dataset. gPC-Guess achieves a 15\% performance boost when combined with Claude's predictions, outperforming the baselines by roughly $10$ percentage points. This demonstrates that our framework extends beyond theory and has potential for combination with existing LLM experts in real-world applications. 

\textbf{Additional Experiments.}
Appendix \ref{appendix: additional experimental results} presents results across varying sparsity, sample sizes, dimensionality, and d-separation accuracy. We report the key findings: both algorithms retain monotonic improvement with expert accuracy in sparse graphs, though gains are reduced (App. \ref{appendix: sparse graph results}). Performance gains from expert predictions diminish with larger samples as all methods converge to the true graph, confirming \ref{criterion:consistency} (App. \ref{appendix: sample size results}). Expert guidance value increases with dimensionality, with greater improvements in high-dimensional, low-sample settings (App. \ref{appendix: dimensionality results}). Below the \ref{criterion:finite-sample} threshold ($p_{\psi}\leq 0.5$), performance drops $\sim8\%$ in F1, but this penalty remains modest due to robust correctness guarantees (App. \ref{appendix: worst case results}).
\textbf{Discussion.} We introduced G2G and applied it to constraint-based discovery to develop PC-Guess and gPC-Guess, which provably leverage unreliable expert predictions with robust guarantees. Future work includes extensions to score/FCM-based algorithms, and integration with hard/soft constraint approaches.
\bibliographystyle{plainnat}  % or other styles like abbrvnat, unsrtnat, etc.
\bibliography{references}

\begin{thebibliography}{73}
\providecommand{\natexlab}[1]{#1}
\providecommand{\url}[1]{\texttt{#1}}
\expandafter\ifx\csname urlstyle\endcsname\relax
  \providecommand{\doi}[1]{doi: #1}\else
  \providecommand{\doi}{doi: \begingroup \urlstyle{rm}\Url}\fi

\bibitem[Ankan and Textor(2025)]{ankan2025expert}
Ankur Ankan and Johannes Textor.
\newblock Expert-in-the-loop causal discovery: Iterative model refinement using expert knowledge.
\newblock In \emph{Proceedings of the 41st Conference on Uncertainty in Artificial Intelligence}, UAI 2025, Rio de Janeiro, Brazil, July 2025.

\bibitem[{Anthropic}(2025)]{anthropic2025claude4}
{Anthropic}.
\newblock System card: Claude opus 4 \& claude sonnet 4.
\newblock Technical report, Anthropic, May 2025.

\bibitem[Balcan et~al.(2021)Balcan, Dick, and Manuel]{balcan_learning_2021}
Maria-Florina Balcan, Travis Dick, and Colin Manuel.
\newblock Learning to link.
\newblock In \emph{International Conference on Machine Learning (ICML)}, 2021.

\bibitem[Ban et~al.()Ban, Chen, Wang, and Chen]{ban2023query}
Taiyu Ban, Lyvzhou Chen, Xiangyu Wang, and Huanhuan Chen.
\newblock From query tools to causal architects: Harnessing large language models for advanced causal discovery from data.

\bibitem[Ban et~al.(2023)Ban, Chen, Lyu, Wang, and Chen]{ban_causal_2023}
Taiyu Ban, Lyuzhou Chen, Derui Lyu, Xiangyu Wang, and Huanhuan Chen.
\newblock Causal structure learning supervised by large language model.
\newblock 2023.

\bibitem[Brouillard et~al.(2025)Brouillard, Squires, Wahl, K{\"o}rding, Sachs, Drouin, and Sridhar]{brouillard_landscape_2025}
Philippe Brouillard, Chandler Squires, Jonas Wahl, Konrad K{\"o}rding, Karen Sachs, Alexandre Drouin, and Dhanya Sridhar.
\newblock The landscape of causal discovery data: Grounding causal discovery in real-world applications.
\newblock In \emph{Proceedings of the Fourth Conference on Causal Learning and Reasoning}, volume 275 of \emph{Proceedings of Machine Learning Research}, pages 834--873. PMLR, 2025.

\bibitem[Brown and Tsamardinos(2008)]{brown2008strategy}
Laura~E. Brown and Ioannis Tsamardinos.
\newblock A strategy for making predictions under manipulation.
\newblock In \emph{Proceedings of the Workshop on the Causation and Prediction Challenge at WCCI 2008}, volume~3 of \emph{PMLR}, pages 35--52, 2008.

\bibitem[Campbell and Moore(2024)]{campbell2024overprecision}
Sandy Campbell and Don~A. Moore.
\newblock Overprecision in the survey of professional forecasters.
\newblock \emph{Collabra: Psychology}, 10\penalty0 (1):\penalty0 92953, 2024.

\bibitem[Chen et~al.(2022)Chen, Silwal, Vakilian, and Zhang]{chen_learning_2022}
Justin~Y. Chen, Sandeep Silwal, Ali Vakilian, and Fred Zhang.
\newblock Faster fundamental graph algorithms via learned predictions.
\newblock In \emph{International Conference on Machine Learning (ICML)}, 2022.

\bibitem[Chickering(2002)]{chickering_optimal_2002}
David~Maxwell Chickering.
\newblock Optimal {Structure} {Identiﬁcation} {With} {Greedy} {Search}.
\newblock \emph{Journal of Machine Learning Research}, 3, 2002.

\bibitem[Chickering(2020)]{chickering_statistically_2020}
David~Maxwell Chickering.
\newblock Statistically {Efﬁcient} {Greedy} {Equivalence} {Search}.
\newblock In \emph{Proceedings of the 36th {Conference} on {Uncertainty} in {Artificial} {Intelligence}}, 2020.

\bibitem[Choo et~al.(2023)Choo, Gouleakis, and Bhattacharyya]{active-learning-advice-arnab-2023}
Davin Choo, Themistoklis Gouleakis, and Arnab Bhattacharyya.
\newblock Active causal structure learning with advice.
\newblock In \emph{Proceedings of the 40th International Conference on Machine Learning}, Proceedings of Machine Learning Research. PMLR, 2023.

\bibitem[Cohrs et~al.(2024)Cohrs, Varando, Diaz, Sitokonstantinou, and Camps-Valls]{cohrs2024large}
Kai-Hendrik Cohrs, Gherardo Varando, Emiliano Diaz, Vasileios Sitokonstantinou, and Gustau Camps-Valls.
\newblock Large language models for constrained-based causal discovery, 2024.

\bibitem[Colombo and Maathuis(2014)]{colombo2014order}
Diego Colombo and Marloes~H. Maathuis.
\newblock Order-independent constraint-based causal structure learning.
\newblock \emph{Journal of Machine Learning Research}, 2014.

\bibitem[Constantinou et~al.(2023)Constantinou, Guo, and Kitson]{constantinou2023impact}
A.C. Constantinou, Z.~Guo, and N.K. Kitson.
\newblock {The Impact of Prior Knowledge on Causal Structure Learning}.
\newblock \emph{Knowledge and Information Systems}, 65:\penalty0 3385--3434, 2023.
\newblock \doi{10.1007/s10115-023-01858-x}.

\bibitem[Cooper and Yoo(1999)]{cooper1999causal}
Gregory~F. Cooper and Changwon Yoo.
\newblock Causal discovery from a mixture of experimental and observational data.
\newblock In \emph{Proceedings of the Conference on Uncertainty in Artificial Intelligence}, pages 116--125, 1999.

\bibitem[Darvariu et~al.(2024)Darvariu, Hailes, and Musolesi]{darvariu2024large}
Victor-Alexandru Darvariu, Stephen Hailes, and Mirco Musolesi.
\newblock Large language models are effective priors for causal graph discovery, 2024.

\bibitem[Dong et~al.(2024)Dong, Jiang, Liu, Jin, Gu, Yang, and Li]{dong-etal-2024-generalization}
Yihong Dong, Xue Jiang, Huanyu Liu, Zhi Jin, Bin Gu, Mengfei Yang, and Ge~Li.
\newblock Generalization or memorization: Data contamination and trustworthy evaluation for large language models.
\newblock Association for Computational Linguistics, 2024.

\bibitem[Dror(2020)]{dror_cognitive_2020}
Itiel~E. Dror.
\newblock Cognitive and human factors in expert decision making: Six fallacies and the eight sources of bias.
\newblock \emph{Analytical Chemistry}, 92\penalty0 (12):\penalty0 7998--8004, 2020.
\newblock \doi{10.1021/acs.analchem.0c00704}.

\bibitem[Ejaz and Bareinboim(2025)]{ejaz_less_2025}
Adiba Ejaz and Elias Bareinboim.
\newblock Less greedy equivalence search.
\newblock In \emph{Advances in Neural Information Processing Systems (NeurIPS)}, 2025.

\bibitem[Erdos and Renyi(1960)]{erdos_renyi}
Paul Erdos and Alfred Renyi.
\newblock {On the evolution of random graphs}.
\newblock \emph{Publication of the Mathematical Institute of the Hungarian Academy of Sciences}, 1960.

\bibitem[Eulig et~al.(2025)Eulig, Mastakouri, Bl{\"o}baum, Hardt, and Janzing]{eulig2025toward}
Eric Eulig, Atalanti~A. Mastakouri, Patrick Bl{\"o}baum, Moritz Hardt, and Dominik Janzing.
\newblock Toward falsifying causal graphs using a permutation-based test.
\newblock In \emph{Proceedings of the AAAI Conference on Artificial Intelligence}, 2025.

\bibitem[Faller et~al.(2024)Faller, Vankadara, Mastakouri, Locatello, and Janzing]{faller2024self}
Philipp~M. Faller, Leena~Chennuru Vankadara, Atalanti~A. Mastakouri, Francesco Locatello, and Dominik Janzing.
\newblock Self-compatibility: Evaluating causal discovery without ground truth.
\newblock AISTATS, 2024.

\bibitem[Feng et~al.(2025)Feng, Qu, Tandon, Li, Kang, and Haffari]{feng2024pretraining}
Tao Feng, Lizhen Qu, Niket Tandon, Zhuang Li, Xiaoxi Kang, and Gholamreza Haffari.
\newblock On the reliability of large language models for causal discovery.
\newblock In \emph{Proceedings of the 63rd Annual Meeting of the Association for Computational Linguistics}, 2025.

\bibitem[Fisher(1921)]{fisher1921probable}
Ronald~A. Fisher.
\newblock On the probable error of a correlation coefficient deduced from a small sample.
\newblock \emph{Metron}, 1:\penalty0 3--32, 1921.

\bibitem[Hasan and Gani(2024)]{knowledge_guided_discovery_2024}
Uzma Hasan and Md~Osman Gani.
\newblock {Optimizing Data-driven Causal Discovery Using Knowledge-guided Search}, 2024.

\bibitem[Hiremath et~al.(2024)Hiremath, Maasch, Gao, Ghosal, and Gan]{suj_2024}
Sujai Hiremath, Jacqueline Maasch, Mengxiao Gao, Promit Ghosal, and Kyra Gan.
\newblock Hybrid top-down global causal discovery with local search for linear and nonlinear additive noise models.
\newblock In \emph{Proceedings of the 37th Conference on Neural Information Processing Systems (NeurIPS)}, 2024.

\bibitem[Hiremath et~al.(2025)Hiremath, Ghosal, and Gan]{suj_2025_losam}
Sujai Hiremath, Promit Ghosal, and Kyra Gan.
\newblock Losam: Local search in additive noise models with mixed mechanisms and general noise for global causal discovery.
\newblock In \emph{Proceedings of the 38th Conference on Neural Information Processing Systems (NeurIPS)}, 2025.

\bibitem[Huang et~al.(2018)Huang, Zhang, Lin, Schölkopf, and Glymour]{huang_generalized_2018}
Biwei Huang, Kun Zhang, Yizhu Lin, Bernhard Schölkopf, and Clark Glymour.
\newblock Generalized {Score} {Functions} for {Causal} {Discovery}.
\newblock In \emph{Proceedings of the 24th {ACM} {SIGKDD} {International} {Conference} on {Knowledge} {Discovery} \& {Data} {Mining}}, pages 1551--1560, London United Kingdom, July 2018. ACM.

\bibitem[Jin and Ma(2022)]{jin2022online}
Billy Jin and Will Ma.
\newblock Online bipartite matching with advice: Tight robustness-consistency tradeoffs for the two-stage model.
\newblock \emph{Advances in Neural Information Processing Systems}, 35:\penalty0 14555--14567, 2022.

\bibitem[Jiralerspong et~al.(2024)Jiralerspong, Chen, More, Shah, and Bengio]{jiralerspong_causal_2024}
Thomas Jiralerspong, Xiaoyin Chen, Yash More, Vedant Shah, and Yoshua Bengio.
\newblock Efficient causal graph discovery using large language models.
\newblock \emph{arXiv preprint arXiv:2402.01207}, 2024.

\bibitem[Kraska et~al.(2018)Kraska, Beutel, Chi, Dean, and Polyzotis]{kraska_case_2018}
Tim Kraska, Alex Beutel, Ed~H. Chi, Jeffrey Dean, and Neoklis Polyzotis.
\newblock The case for learned index structures.
\newblock In \emph{Proceedings of the 2018 International Conference on Management of Data (SIGMOD)}, pages 489--504, 2018.

\bibitem[Kıcıman et~al.(2024)Kıcıman, Ness, Sharma, and Tan]{kiciman2024causal}
Emre Kıcıman, Robert Ness, Amit Sharma, and Chenhao Tan.
\newblock Causal reasoning and large language models: Opening a new frontier for causality.
\newblock \emph{Transactions on Machine Learning Research (TMLR)}, 2024.

\bibitem[Lee et~al.(2025)Lee, Ribeiro, and Kocaoglu]{lee_constraint_2025}
Kenneth Lee, Bruno Ribeiro, and Murat Kocaoglu.
\newblock Constraint-based causal discovery from a collection of conditioning sets.
\newblock In \emph{Proceedings of the 41st Conference on Uncertainty in Artificial Intelligence (UAI)}, 2025.

\bibitem[Levin and Peres(2023)]{levin2023coupling}
David~A Levin and Yuval Peres.
\newblock Coupling.
\newblock In \emph{Modern Discrete Probability: An Essential Toolkit}, chapter~4. 2023.
\newblock Version: December 20, 2023.

\bibitem[Li and Wang(2009)]{li2009controlling}
Junning Li and Z.~Jane Wang.
\newblock Controlling the false discovery rate of the association/causality structure learned with the pc algorithm.
\newblock \emph{Journal of Machine Learning Research}, 10\penalty0 (17):\penalty0 475--514, 2009.

\bibitem[Li et~al.(2020)Li, Shen, and Zhou]{li2020nonparametric}
Runze Li, Wei Shen, and Haibo Zhou.
\newblock On nonparametric conditional independence tests for continuous variables.
\newblock \emph{WIREs Computational Statistics}, 12\penalty0 (2):\penalty0 e1489, 2020.

\bibitem[Lindvall(1999)]{lindvall1999strassen}
Torgny Lindvall.
\newblock On strassen's theorem on stochastic domination.
\newblock \emph{Electronic Communications in Probability}, 4:\penalty0 51--59, 1999.
\newblock \doi{10.1214/ECP.v4-1005}.

\bibitem[Liu et~al.(2024)Liu, Gan, Keyvanshokooh, and Murphy]{liu2024online}
Xueqing Liu, Kyra Gan, Esmaeil Keyvanshokooh, and Susan Murphy.
\newblock Online uniform sampling: Randomized learning-augmented approximation algorithms with application to digital health.
\newblock \emph{arXiv preprint arXiv:2402.01995}, 2024.

\bibitem[Lykouris and Vassilvitskii(2018)]{competitive_caching_2018}
Thodoris Lykouris and Sergei Vassilvitskii.
\newblock {Competitive Caching with Machine Learned Advice}, 2018.

\bibitem[Lykouris and Vassilvitskii(2021)]{lykouris_competitive_2021}
Thodoris Lykouris and Sergei Vassilvitskii.
\newblock Competitive caching with machine learned advice.
\newblock In \emph{Journal of the ACM}, volume~68, 2021.

\bibitem[Maasch et~al.(2024)Maasch, Pan, Gupta, Kuleshov, Gan, and Wang]{maasch2024local}
Jacqueline Maasch, Weishen Pan, Shantanu Gupta, Volodymyr Kuleshov, Kyra Gan, and Fei Wang.
\newblock Local discovery by partitioning: Polynomial-time causal discovery around exposure-outcome pairs.
\newblock In \emph{Proceedings of the 40th Conference on Uncertainy in Artificial Intelligence}, 2024.
\newblock \doi{https://doi.org/10.48550/arXiv.2310.17816}.

\bibitem[Manggala et~al.(2025)Manggala, Mastakouri, Kirschbaum, Kasiviswanathan, and Ramdas]{manggala2025qa}
Putra Manggala, Atalanti Mastakouri, Elke Kirschbaum, Shiva~Prasad Kasiviswanathan, and Aaditya Ramdas.
\newblock Qa-calibration of language model confidence scores.
\newblock In \emph{Proceedings of the International Conference on Learning Representations}, 2025.

\bibitem[Mitzenmacher and Vassilvitskii(2020)]{algorithms_predictions_2020}
Michael Mitzenmacher and Sergei Vassilvitskii.
\newblock {Algorithms with Predictions}, 2020.

\bibitem[Montagna et~al.(2023{\natexlab{a}})Montagna, Noceti, Rosasco, Zhang, and Locatello]{montagna_causal_2023}
Francesco Montagna, Nicoletta Noceti, Lorenzo Rosasco, Kun Zhang, and Francesco Locatello.
\newblock Causal {Discovery} with {Score} {Matching} on {Additive} {Models} with {Arbitrary} {Noise}.
\newblock In \emph{Proceedings of the 2nd {Conference} on {Causal} {Learning} and {Reasoning}}. arXiv, April 2023{\natexlab{a}}.
\newblock arXiv:2304.03265 [cs, stat].

\bibitem[Montagna et~al.(2023{\natexlab{b}})Montagna, Noceti, Rosasco, Zhang, and Locatello]{montagna_scalable_2023}
Francesco Montagna, Nicoletta Noceti, Lorenzo Rosasco, Kun Zhang, and Francesco Locatello.
\newblock Scalable {Causal} {Discovery} with {Score} {Matching}.
\newblock In \emph{Proceedings of the 2nd {Conference} on {Causal} {Learning} and {Reasoning}}. arXiv, April 2023{\natexlab{b}}.
\newblock arXiv:2304.03382 [cs, stat].

\bibitem[Pearson(1900)]{pearson1900criterion}
Karl Pearson.
\newblock On the criterion that a given system of deviations from the probable in the case of a correlated system of variables is such that it can be reasonably supposed to have arisen from random sampling.
\newblock \emph{Philosophical Magazine Series 5}, 50\penalty0 (302):\penalty0 157--175, 1900.
\newblock \doi{10.1080/14786440009463897}.

\bibitem[Peters et~al.(2014)Peters, Mooij, Janzing, and Schölkopf]{peters_causal_2014}
Jonas Peters, Joris Mooij, Dominik Janzing, and Bernhard Schölkopf.
\newblock Causal {Discovery} with {Continuous} {Additive} {Noise} {Models}, 2014.

\bibitem[Petersen et~al.(2021)Petersen, Osler, and Ekstr{\o}m]{petersen2021data}
Anne~H Petersen, Merete Osler, and Claus~T Ekstr{\o}m.
\newblock Data-driven model building for life-course epidemiology.
\newblock \emph{American Journal of Epidemiology}, 2021.

\bibitem[Ramsey et~al.(2006)Ramsey, Spirtes, and Zhang]{ramsey2006adjacency}
Joseph Ramsey, Peter Spirtes, and Jiji Zhang.
\newblock Adjacency-faithfulness and conservative causal inference.
\newblock In \emph{Proceedings of the Twenty-Second Conference on Uncertainty in Artificial Intelligence}, pages 401--408, 2006.

\bibitem[Sachs et~al.(2005)Sachs, Perez, Pe'er, Lauffenburger, and Nolan]{sachs2005causal}
Karen Sachs, Omar Perez, Dana Pe'er, Douglas~A Lauffenburger, and Garry~P Nolan.
\newblock Causal protein-signaling networks derived from multiparameter single-cell data.
\newblock \emph{Science}, 308\penalty0 (5721):\penalty0 523--529, 2005.

\bibitem[Scanagatta et~al.(2014)Scanagatta, de~Campos, and Zaffalon]{scanagatta2014minbdeu}
Mauro Scanagatta, Cassio~P. de~Campos, and Marco Zaffalon.
\newblock Min-bdeu and max-bdeu scores for learning bayesian networks.
\newblock In \emph{Probabilistic Graphical Models}, volume 8754 of \emph{Lecture Notes in Computer Science}, pages 426--441. Springer, 2014.
\newblock European Workshop on Probabilistic Graphical Models (PGM 2014).

\bibitem[Sondhi and Shojaie(2019)]{sondhi2019reduced}
Arjun Sondhi and Ali Shojaie.
\newblock The reduced pc-algorithm: Improved causal structure learning in large random networks.
\newblock \emph{Journal of Machine Learning Research}, 20:\penalty0 1--31, 2019.

\bibitem[Spirtes(2001)]{spirtes_anytime_2001}
Peter Spirtes.
\newblock An {Anytime} {Algorithm} for {Causal} {Inference}.
\newblock In \emph{Proceedings of the {Eighth} {International} {Workshop} on {Artificial} {Intelligence} and {Statistics}}, volume~R3, pages 278--285. PMLR, 2001.

\bibitem[Spirtes and Glymour(1991)]{spirtes_algorithm_1991}
Peter Spirtes and Clark Glymour.
\newblock An {Algorithm} for {Fast} {Recovery} of {Sparse} {Causal} {Graphs}.
\newblock \emph{Social Science Computer Review}, 1991.

\bibitem[Spirtes et~al.(2000)Spirtes, Glymour, and Scheines]{peter_spirtes_causation_2000}
Peter Spirtes, Clark Glymour, and Richard Scheines.
\newblock \emph{Causation, {Prediction}, and {Search}}, volume~81 of \emph{Lecture {Notes} in {Statistics}}.
\newblock Springer New York, New York, NY, 2000.
\newblock ISBN 978-1-4612-7650-0 978-1-4612-2748-9.
\newblock \doi{10.1007/978-1-4612-2748-9}.

\bibitem[Strobl et~al.(2019)Strobl, Spirtes, and Visweswaran]{strobl2019estimating}
Eric~V. Strobl, Peter~L. Spirtes, and Shyam Visweswaran.
\newblock Estimating and controlling the false discovery rate of the pc algorithm using edge-specific p-values.
\newblock \emph{ACM Transactions on Intelligent Systems and Technology}, 10\penalty0 (5):\penalty0 1--37, 2019.

\bibitem[Sun and Li(2024)]{llm_causal_discovery_2024}
Zhuofan Sun and Qingyi Li.
\newblock {Leveraging LLMs for Causal Inference and Discovery}, 2024.

\bibitem[Susanti and Faber(2025)]{susanti2025can}
Yuni Susanti and Michael Faber.
\newblock {Can LLMs Leverage Observational Data? Towards Data-Driven Causal Discovery with LLMs}, 2025.

\bibitem[Takayama et~al.(2025)Takayama, Okuda, Pham, Ikenoue, Fukuma, Shimizu, and Sannai]{takayama_integrating_2025}
Masayuki Takayama, Tadahisa Okuda, Thong Pham, Tatsuyoshi Ikenoue, Shingo Fukuma, Shohei Shimizu, and Akiyoshi Sannai.
\newblock Integrating large language models in causal discovery: A statistical causal approach.
\newblock \emph{Transactions on Machine Learning Research}, 2025.
\newblock arXiv:2402.01454.

\bibitem[Tennant et~al.(2021)Tennant, Murray, Arnold, Berrie, Fox, Gadd, Harrison, Keeble, Ranker, Textor, Tomova, Gilthorpe, and Ellison]{tennant_use_2021}
Peter W~G Tennant, Eleanor~J Murray, Kellyn~F Arnold, Laurie Berrie, Matthew~P Fox, Sarah~C Gadd, Wendy~J Harrison, Claire Keeble, Lynsie~R Ranker, Johannes Textor, Georgia~D Tomova, Mark~S Gilthorpe, and George T~H Ellison.
\newblock {Use of directed acyclic graphs ({DAGs}) to identify confounders in applied health research: review and recommendations}.
\newblock \emph{International Journal of Epidemiology}, 2021.

\bibitem[Tsamardinos et~al.(2006)Tsamardinos, Brown, and Aliferis]{tsamardinos_max-min_2006}
Ioannis Tsamardinos, Laura~E. Brown, and Constantin~F. Aliferis.
\newblock The max-min hill-climbing {Bayesian} network structure learning algorithm.
\newblock \emph{Machine Learning}, 65\penalty0 (1):\penalty0 31--78, October 2006.
\newblock ISSN 0885-6125, 1573-0565.
\newblock \doi{10.1007/s10994-006-6889-7}.

\bibitem[Uhler et~al.(2013)Uhler, Raskutti, Bühlmann, and Yu]{uhler_geometry_2013}
Caroline Uhler, Garvesh Raskutti, Peter Bühlmann, and Bin Yu.
\newblock Geometry of the faithfulness assumption in causal inference.
\newblock \emph{The Annals of Statistics}, 41\penalty0 (2):\penalty0 436--463, 2013.
\newblock \doi{10.1214/12-AOS1080}.

\bibitem[Vashishtha et~al.(2025)Vashishtha, Abbavaram, Kumar, Bachu, Balasubramanian, and Sharma]{vashishtha2025causal}
Aniket Vashishtha, Gowtham~Reddy Abbavaram, Abhinav Kumar, Saketh Bachu, Vineeth~N. Balasubramanian, and Amit Sharma.
\newblock {Causal Order: The Key to Leveraging Imperfect Experts in Causal Inference}.
\newblock In \emph{The Thirteenth International Conference on Learning Representations (ICLR)}, 2025.

\bibitem[Wei and Zhang(2020)]{wei2020optimal}
Alexander Wei and Fred Zhang.
\newblock Optimal robustness-consistency trade-offs for learning-augmented online algorithms.
\newblock \emph{Advances in Neural Information Processing Systems}, 33:\penalty0 8042--8053, 2020.

\bibitem[Wu et~al.(2025)Wu, Yu, Wu, and Tan]{wu2025llm}
Xingyu Wu, Kui Yu, Jibin Wu, and Kay~Chen Tan.
\newblock {LLM Cannot Discover Causality, and Should Be Restricted to Non-Decisional Support in Causal Discovery}, 2025.

\bibitem[Xie et~al.(2024)Xie, Zheng, Ottens, Zhang, Kozyrakis, and Mace]{xie2024cloud}
Zhiqiang Xie, Yujia Zheng, Lizi Ottens, Kun Zhang, Christos Kozyrakis, and Jonathan Mace.
\newblock Cloud atlas: Efficient fault localization for cloud systems using language models and causal insight.
\newblock 2024.

\bibitem[Ye et~al.(2024)Ye, Yang, Pang, Wang, Wong, Yilmaz, Shi, and Tu]{ye2024benchmarking_arxiv}
Fanghua Ye, Mingming Yang, Jianhui Pang, Longyue Wang, Derek~F. Wong, Emine Yilmaz, Shuming Shi, and Zhaopeng Tu.
\newblock Benchmarking {LLMs} via uncertainty quantification.
\newblock 2024.

\bibitem[Yessenov(2005)]{yessenov2005descents}
Kuat Yessenov.
\newblock Descents and the weak bruhat order, 2005.

\bibitem[Zhang and Hyvarinen(2008)]{zhang_distinguishing_2008}
Kun Zhang and Aapo Hyvarinen.
\newblock Distinguishing {Causes} from {Eﬀects} using {Nonlinear} {Acyclic} {Causal} {Models}.
\newblock 2008.

\bibitem[Zhang and Hyvarinen(2009)]{zhang_identiability_2009}
Kun Zhang and Aapo Hyvarinen.
\newblock On the {Identiﬁability} of the {Post}-{Nonlinear} {Causal} {Model}.
\newblock \emph{Uncertainty in Artificial Intelligence}, 2009.

\bibitem[Zhu et~al.(2024)Zhu, Benos, and Chikina]{zhu_hybrid_2024}
Yuehua Zhu, Panayiotis~V. Benos, and Maria Chikina.
\newblock A hybrid constrained continuous optimization approach for optimal causal discovery from biological data.
\newblock \emph{Bioinformatics}, 40\penalty0 (Suppl 2):\penalty0 ii87--ii97, 2024.
\newblock \doi{10.1093/bioinformatics/btae411}.

\bibitem[Zuk et~al.(2012)Zuk, Margel, and Domany]{zuk_samples_2012}
Or~Zuk, Shiri Margel, and Eytan Domany.
\newblock On the number of samples needed to learn the correct structure of a bayesian network.
\newblock In \emph{Proceedings of the 28th Conference on Uncertainty in Artificial Intelligence (UAI)}, pages 560--569, 2012.

\end{thebibliography}

\clearpage
\appendix
\thispagestyle{empty}

\renewcommand{\thefigure}{\Alph{section}.\arabic{figure}}
\makeatletter
\@addtoreset{figure}{section}
\makeatother

\renewcommand{\thealgorithm}{\Alph{section}.\arabic{algorithm}}
\makeatletter
\@addtoreset{algorithm}{section}
\makeatother

% Supplementary material: To improve readability, you must use a single-column format for the supplementary material.
\onecolumn
\aistatstitle{Appendix}
% new appendix stuffs
\section{Introduction}\label{appendix: introduction}

\subsection{Unbounded Error Example}\label{appendix: unbounded error}

To illustrate why hard constraints from unreliable experts can cause unbounded error, we provide a concrete example using the PC algorithm. Recall that PC discovers the skeleton of the underlying DAG by running conditional independence tests, removing edges between pairs of variables when there exists at least one conditioning set that renders them conditionally independent.

\textbf{Example construction.} Consider the true DAG on four variables $\{x_1, x_2, x_3, x_4\}$ where $x_1$ is a common cause: $x_1 \rightarrow x_2$, $x_1 \rightarrow x_3$, $x_1 \rightarrow x_4$. In this structure, $x_1$ confounds the relationships between $x_2$, $x_3$, and $x_4$. Now suppose an expert provides a hard constraint specifying that $x_1$ is an isolated node with no edges, implying $x_1 \perp\!\!\!\perp x_2$, $x_1 \perp\!\!\!\perp x_3$, $x_1 \perp\!\!\!\perp x_4$.

\textbf{Cascading failure.} Under this incorrect hard constraint, PC with oracle independence tests will incorrectly infer edges between all pairs in $\{x_2, x_3, x_4\}$: specifically, edges $(x_2, x_3)$, $(x_2, x_4)$, and $(x_3, x_4)$. This occurs because the only conditioning set that renders these variables pairwise independent is $\{x_1\}$. However, PC's adjacency-based testing procedure means that when testing whether to remove edge $(x_i, x_j)$, the algorithm only considers conditioning sets composed of variables adjacent to at least one of $x_i$ or $x_j$ in the current skeleton. Since the expert's hard constraint excludes $x_1$ from all adjacency sets by decree, the critical conditioning set $\{x_1\}$ is never tested, and all three spurious edges are incorrectly retained.

\textbf{Unbounded error.} The resulting skeleton exhibits maximum possible error: every edge in the recovered skeleton is absent from the true graph (three false positives), and every edge in the true graph is absent from the recovered skeleton (three false negatives). This demonstrates unbounded error in the sense that expert misinformation leads to arbitrarily poor performance relative to running PC without any expert guidance—even with oracle conditional independence tests and infinite samples.

\textbf{General mechanism of error propagation.} This example illustrates a fundamental vulnerability of hard-constraint methods: causal discovery algorithms typically leverage results from early tests to determine which tests are run (and how their results are interpreted) at later stages. This sequential dependence is often necessary to efficiently navigate the super-exponential search space of DAGs. However, it creates a pathway for error propagation, where mistakes in discovering one part of the causal structure (such as incorrectly excluding edges based on expert constraints) cause: (i) crucial tests to never be performed, and (ii) results of later tests to be misinterpreted or leveraged incorrectly. When expert advice in the form of a hard constraint incorrectly determines parts of the causal graph, it initiates a cascade of errors that propagates throughout the algorithm, potentially leading to catastrophic failure even in the large-sample limit. This motivates our approach of using expert guidance to optimize test sequences rather than to replace statistical testing with hard constraints.

\subsection{Sequential Testing in Causal Discovery}\label{appendix: sequential test issue}

Here we discuss how sequential statistical testing underlies the three major paradigms of causal discovery: constraint-based, score-based, and functional causal model (FCM) based methods. In each paradigm, the tests performed at later stages fundamentally depend on the outcomes of earlier tests, creating a sequential dependence structure susceptible to error propagation.

\textbf{Constraint-based methods.} Constraint-based algorithms operate by performing conditional independence tests (CITs) to iteratively refine a candidate graph structure. Critically, which CITs are performed at any given stage depends on the results of CITs performed in earlier stages. This sequential dependence manifests primarily through \emph{adjacency sets}: when testing whether to remove an edge between variables $x_i$ and $x_j$, algorithms such as PC \citep{spirtes_anytime_2001} only condition on subsets of variables currently adjacent to $x_i$ or $x_j$ in the working skeleton. When earlier tests incorrectly remove (or retain) edges, the adjacency sets used in subsequent tests are corrupted. For instance, if a true edge $x_i \rightarrow x_k$ is incorrectly removed early in the algorithm, then $x_k$ will be excluded from the adjacency set when later testing edges involving $x_i$, potentially causing the algorithm to miss the conditioning set needed to correctly identify other edges. This cascading effect means that errors in early tests propagate through the entire discovery process via their influence on which variables are considered in later conditioning sets.

\textbf{Score-based methods.} Score-based algorithms attempt to optimize a scoring function (such as BIC) by iteratively adding, deleting, or reversing edges in the candidate graph \citep{chickering_optimal_2002}. At each step, algorithm evaluates the score change associated with each possible modification and selects the move that most improves the score. Crucially, the score assigned to any proposed edge modification depends on the \emph{current parent sets} of the variables involved, which are themselves determined by all prior edge additions and deletions. For example, when considering whether to add edge $x_i \rightarrow x_j$, the score change is computed based on $x_j$'s current parent set $\text{Pa}(x_j)$; if previous iterations incorrectly modified $\text{Pa}(x_j)$, the score for this new edge will be miscalculated. Recent work has formalized score-based methods as performing implicit conditional independence tests \citep{chickering_statistically_2020}, making the connection to sequential testing even more explicit. Like constraint-based methods, score-based approaches exhibit cascading errors where early mistakes in edge selection corrupt the parent sets used for scoring future edge modifications.

\textbf{FCM-based methods.} Functional causal model (FCM) based methods, specifically those based on additive noise models (ANM) \citep{zhang_distinguishing_2008}, typically operate by constructing a topological ordering of variables through recursive identification of roots (variables with no parents) or leaves (variables with no children \citep{montagna_scalable_2023,montagna_causal_2023, suj_2024, suj_2025_losam}). The standard approach maintains a set of unsorted variables and, at each iteration, tests which of these variables satisfy the root or leaf condition by checking whether they are independent of other unsorted variables conditional on already-sorted variables. For instance, when identifying leaves, a variable $x_i$ is classified as a leaf if its residuals (after regressing on each other unsorted variable individually) are independent of those regressors. The key sequential dependence arises because: (i) the set of candidate variables tested for root/leaf status at each iteration depends on which variables were identified in previous iterations, and (ii) the conditioning sets used in these independence tests (the already-sorted variables) are built up incrementally based on prior test results. If an earlier iteration incorrectly identifies a non-leaf as a leaf, this error propagates by corrupting both the candidate set and the conditioning sets used in all subsequent iterations. This can be viewed as a series of tests for root/leaf status, where the outcome of each test determines which variables and conditioning sets are used in future tests.

\textbf{Implications for expert guidance.} This universal reliance on sequential testing across all three paradigms creates both a challenge and an opportunity. The challenge is that error propagation can cause early mistakes to cascade throughout the algorithm. The opportunity is that by optimizing the \emph{sequence} in which tests are performed—without changing the tests themselves—we can improve finite-sample performance while preserving asymptotic correctness. Our G2G framework exploits this insight by using expert predictions to guide test sequences rather than to replace statistical testing, making it broadly applicable across constraint-based, score-based, and FCM-based methods.

\subsection{Expert Error Violating Guarantees for Expert-Aided Discovery with Soft Constraints}\label{appendix: expert error violation}

Traditional expert-aided discovery methods integrate expert predictions through either hard constraints (which enforce expert beliefs directly) or soft constraints (which bias statistical procedures toward expert beliefs). In Appendix~\ref{appendix: unbounded error}, we demonstrated how hard constraints can lead to unbounded error through error propagation. Here, we discuss how soft constraint approaches fundamentally compromise the theoretical guarantees of causal discovery algorithms in both score-based and constraint-based methods. In what follows we summarize the analysis of \citet{wu2025llm}, who demonstrate that soft constraint methods break the mathematical foundations underlying both paradigms.

\textbf{Soft constraints in score-based methods.} Score-based methods incorporating soft constraints typically modify the scoring function by adding a prior term based on expert predictions:
$$\sigma(G; D, \lambda) = \sigma(G; D) + \sigma(G; \lambda)$$
where $\sigma(G; D)$ evaluates how well graph $G$ fits the data $D$, and $\sigma(G; \lambda)$ rewards or penalizes structures based on expert constraints $\lambda$. This direct summation creates three fundamental problems that undermine theoretical guarantees:

First, \emph{the terms operate in incompatible probability spaces}. The data term $\sigma(G; D)$ reflects log-likelihood under sample data, while the prior term $\sigma(G; \lambda)$ encodes expert beliefs independent of data. These quantities have different statistical foundations and cannot be meaningfully combined through simple addition. The scale mismatch is severe: for example, in the BIC score $\sigma_{BIC}(G; D) = \log P(D|G) - \frac{k}{2}\log N$, the sample size $N$ dramatically affects the magnitude of the data term, while $\sigma(G; \lambda)$ is independent of $N$. Without principled normalization, one term typically dominates, either drowning out expert guidance or allowing poor expert predictions to overwhelm statistical evidence.

Second, \emph{existing scoring functions already contain implicit priors that conflict with expert priors}. For instance, the BDeu score \citep{scanagatta2014minbdeu} assumes Dirichlet uniform priors and prior independence of all variables—assumptions that directly contradict typical expert predictions about causal relationships. Introducing $\sigma(G; \lambda)$ creates competing prior specifications with no principled mechanism for reconciliation, requiring ad-hoc hyperparameters to balance multiple conflicting priors and data fit.

Third, \emph{adding expert priors breaks critical structural properties}. Traditional scoring functions satisfy \emph{decomposability}, allowing the score to be written as $\sigma(G; D) = \sum_{i=1}^n \sigma(x_i, \text{Pa}(x_i); D)$, which enables efficient local optimization. Expert priors $\sigma(G; \lambda)$ often impose global constraints across multiple variables that cannot be decomposed into local contributions, rendering many optimization algorithms inapplicable. Additionally, soft constraints violate \emph{score local consistency}, which guarantees that improving local fit to data improves the overall score. When expert priors incorrectly emphasize certain relationships, local changes that better reflect data may receive lower scores, breaking the connection between score optimization and causal structure recovery.

\textbf{Soft constraints in constraint-based methods.} While hard constraints in constraint-based methods create error propagation pathways (Appendix~\ref{appendix: unbounded error}), soft constraint approaches that incorporate edge priors from experts into conditional independence testing face a different but equally fundamental issue: they invalidate the distributional theory underlying hypothesis tests.

Some approaches modify conditional independence test statistics by incorporating expert edge priors. For example, methods that integrate soft constraints into the $G^2$ test adjust the statistic as:
$$G^2(X, Y | Z) - p > \chi^2_{\alpha, f}$$
where $p$ encodes prior beliefs about the strength or likelihood of the causal relationship between $X$ and $Y$. While this appears to simply modulate the rejection region based on expert confidence, it fundamentally invalidates the statistical theory underlying the test.

The core issue is that \emph{subtracting $p$ distorts the asymptotic distribution}. The $G^2$ statistic has well-established asymptotic properties—specifically, it follows a chi-squared distribution under the null hypothesis of conditional independence. Subtracting an arbitrary prior term $p$ based on expert beliefs destroys these properties: the modified statistic $G^2(X, Y | Z) - p$ no longer follows a chi-squared distribution, invalidating hypothesis tests based on comparing to $\chi^2_{\alpha, f}$ critical values. Even if the modified statistic could be shown to asymptotically follow some chi-squared distribution, the degrees of freedom $f$ and critical values would need to be rederived from first principles—a non-trivial task that existing methods do not address. Without valid asymptotic theory, there are no guarantees about Type I error rates, power, or consistency.

This problem extends beyond the $G^2$ test to any approach that incorporates expert edge priors by modifying test statistics. Whether using Fisher's Z-test, permutation tests, or other conditional independence tests, directly adjusting the test statistic or critical values based on soft expert constraints breaks the calibration that ensures valid statistical inference. The result is that even with infinite data, these methods cannot guarantee asymptotic correctness when expert predictions are inaccurate.

\textbf{Implications.} Both score-based and constraint-based soft constraint approaches sacrifice the rigorous statistical foundations that provide guarantees for purely data-driven methods. In score-based methods, soft constraints introduce mathematically inconsistent score combinations and break structural properties needed for optimization. In constraint-based methods, soft constraints invalidate the distributional theory underlying hypothesis tests. These are not merely technical concerns—they represent fundamental incompatibilities between integrating unreliable expert predictions through soft constraints and maintaining statistical guarantees. Our G2G framework addresses these issues by using expert predictions to guide test \emph{sequences} rather than test \emph{outcomes}, preserving the statistical validity of each individual test while leveraging expert knowledge to improve finite-sample efficiency.

\newpage
\section{Definitions}\label{appendix: problem setup}
This section details the core assumptions underlying our causal structure learning framework.

\begin{definition}[Causal Markov Condition, \citealt{spirtes_anytime_2001}]\label{def:markov}
    A causal graph \(\mathcal{G} = (V, E)\) satisfies the Causal Markov Condition if and only if every variable \(x_i \in V\) is independent of its non-descendants given its parents \(\mathrm{Pa}(x_i)\). This implies that the joint distribution \(p(V)\) factorizes as:
    \begin{equation}
        p(V) = \prod_{i=1}^{d} p(x_i \mid \mathrm{Pa}(x_i)).
    \end{equation}
\end{definition}

\begin{definition}[Acyclicity, \citealt{spirtes_anytime_2001}]\label{def:acyclicity}
    A causal graph \(\mathcal{G}\) is acyclic if it contains no directed paths starting and ending at the same node.
\end{definition}

\begin{definition}[Faithfulness, \citealt{spirtes_anytime_2001}]\label{def:faithfulness}
    A distribution \(p\) is faithful to a graph \(\mathcal{G}\) if every conditional independence relation present in \(p\) is entailed by the Causal Markov Condition applied to \(\mathcal{G}\). That is, for any disjoint sets of variables \(x, y, Z\):
    \[
    x \perp\!\!\!\perp y \mid Z \text{ in } p \quad \rightarrow \quad Z \text{ d-separates } x \text{ and } y \text{ in } \mathcal{G}.
    \]
\end{definition}

\begin{definition}[Causal Sufficiency, \citealt{spirtes_anytime_2001}]\label{def:sufficiency}
    A set of variables \(V\) is causally sufficient if there exist no unobserved confounders for any pair of variables in \(V\). Formally, for any \(x_i, x_j \in V\), there is no unmeasured variable \(U \notin V\) that is a direct cause of both \(x_i\) and \(x_j\) in the true causal graph.
\end{definition}

\begin{definition}[d-Separation, \citealt{spirtes_anytime_2001}]\label{def:d_separation}
    A set of variables \(\mathbf{Z}\) d-separates variables \(x\) and \(y\) in a graph \(\mathcal{G}\) if and only if \(\mathbf{Z}\) blocks all paths between \(x\) and \(y\). This graphical condition implies the conditional independence \(x \perp\!\!\!\perp y \mid \mathbf{Z}\) in every distribution that is Markov with respect to \(\mathcal{G}\).
\end{definition}

\begin{definition}[Markov Equivalence Class, \citealt{spirtes_anytime_2001}]\label{def:markov_equivalence}
    The Markov equivalence class (MEC) of a DAG \(\mathcal{G}\) is the set of all DAGs that imply the same set of conditional independence statements via d-separation. Under the assumptions of causal Markov, faithfulness, and causal sufficiency, the MEC can be identified from perfect conditional independence tests and is typically represented by a Completed Partially Directed Acyclic Graph (CPDAG).
\end{definition}

\begin{definition}[Sufficient Power and Mutual Independence of CITs]\label{assumption:cit_power_independence}
    We assume that conditional independence tests satisfy two properties:

  \begin{enumerate}
    \item \textbf{Sufficient Specificity}: CITs are adequately powered such that $1 - \alpha > \beta$, where $\alpha$ is the Type I error rate (false positive rate) and $\beta$ is the Type II error rate (false negative rate) of each CIT. 
    
    \item \textbf{Mutual Independence}: The outcomes of different CITs are mutually independent. Formally, for CITs indexed by $i \in \{1, \ldots, m\}$:
    \[
    \mathbb{P}(\text{CIT}_1, \ldots, \text{CIT}_m) = \prod_{i=1}^m \mathbb{P}(\text{CIT}_i)
    \]

\end{enumerate}

    The `sufficient specificity' assumption claims that each CIT has a false positive and false negative rate of $\alpha,\beta$ respectively, and ensures the true negative rate exceeds the false negative rate, i.e., that there is enough data collected such that CITs that condition on d-separating sets return independence with higher probability than those that do not.

      The `mutual independence' assumption is a simplifying assumption for theoretical analysis commonly made in causal discovery \citep{brown2008strategy, li2009controlling, strobl2019estimating, cooper1999causal}. The statement holds exactly with infinite data or when using data-splitting (where each CIT uses an independent sample), but is an approximation when CITs reuse the same finite dataset, as is common in causal discovery. 
      
      We invoke these assumptions for our runtime analysis (Subroutine \ref{algo: expert-subset-ordering1}, Lemma \ref{lemma: optimal-sequence}), but note that they are both \textit{not} used in our accuracy analysis (Subroutine \ref{algo: expert-edge-ordering1}, Lemma \ref{lemma: edge-accuracy-monotonic}).
\end{definition}

\newpage
\section{CD-GUESS Framework and Application to Constraint-Based Discovery}\label{appendix: framework and CB discovery}

% \subsection{CD-GUESS Framework Algorithm}\label{appendix: cd-guess algorithm}

% \begin{algorithm}[h!]
% \caption{CD-GUESS: Causal Discovery Guided by Expert Search Strategy}
% \label{algo: cd-guess}
% \begin{algorithmic}[1]
% \Require Dataset $\mathcal{D}$ over variables $V$, Expert $\mathcal{E}$, Base algorithm $\mathcal{A}$
% \Ensure Learned causal structure $C$
% \State \textbf{Phase 1: Expert Prediction}
% \State Query expert $\mathcal{E}$ for causal structure prediction $\hat{G}$ 
% \State \Comment{$\hat{G}$ can be a DAG, skeleton, ordering, etc.}
% \State 
% \State \textbf{Phase 2: Search Strategy Extraction}
% \State Extract search strategies $\mathcal{SS} = \{\text{SS}_1, \text{SS}_2, \ldots\}$ from $\hat{G}$
% \State \Comment{e.g., edge ordering, subset ordering, etc.}
% \State 
% \State \textbf{Phase 3: Guided Discovery}
% \State Initialize $C$ according to base algorithm $\mathcal{A}$
% \For{each subroutine $S_i$ in $\mathcal{A}$ that uses uniform sampling}
%     \If{search strategy $\text{SS}_i \in \mathcal{SS}$ exists for $S_i$}
%         \State Replace uniform sampling in $S_i$ with sampling guided by $\text{SS}_i$
%     \Else
%         \State Keep uniform sampling in $S_i$
%     \EndIf
% \EndFor
% \State Run modified algorithm $\mathcal{A}$ with guided subroutines on $\mathcal{D}$
% \State \Return learned structure $C$
% \end{algorithmic}
% \end{algorithm}

\subsection{Extensions of CD-GUESS Framework}\label{appendix: cd-guess framework extensions}
Here we outline a few possible extensions of the CD-GUESS Framework.

\subsubsection{Heuristic Selection of Experts and Pruning of Guesses}

When integrating LLM predictions with data-driven causal discovery, we identify three distinct regimes: World 1 where data-driven methods alone perform best (LLM guidance degrades performance), World 2 where hybrid approaches excel (combining data and LLM guidance), and World 3 where the LLM guess alone suffices (outperforming empirical methods). Which world applies depends on the interplay between LLM guess quality, sample size, and the underlying graph structure—with higher-quality guesses and smaller samples favoring Worlds 2 and 3. Our primary concern is preventing World 1 scenarios by detecting when LLM predictions are harmful, as this represents the most critical failure mode where expert guidance actively degrades baseline performance. There are two main approaches to trying to prevent World 1 from occurring. The first is validate whether the specific guess given by an expert is better than random (Checking Guess Quality). If this is infeasible, we can attempt to identify whether the expert is any good in this domain (Checking Expert Competency).

\textbf{Guess Quality.} This approach evaluates specific causal structure predictions by assigning fitness scores that correlate with graph accuracy—higher scores indicate more accurate graphs. To assess whether a given score is meaningful, we construct a null distribution by sampling random graphs and testing whether the expert's guess scores significantly above this baseline. Two primary methodologies exist: likelihood-based and nonparametric approaches, which differ in computational complexity and modeling flexibility.

Likelihood-based methods \citep{huang_generalized_2018} assume parametric models for the underlying functional relationships (e.g., residing in specific kernel spaces or exponential families) and derive scores that measure graph-data compatibility. Random graphs of similar density provide the comparison baseline. Notably, likelihood scores do not monotonically increase with graph accuracy as measured by skeleton or edge metrics, serving instead as relative comparison measures between candidate structures.

Nonparametric approaches, exemplified by permutation-based methods \citep{eulig2025toward}, evaluate graphs by counting how many conditional independence tests in the data align with the d-separation statements implied by each graph. When sampling comparison graphs, we can either generate them uniformly at random or preserve specific properties of the expert guess, such as sparsity constraints or causal ordering structure. While these methods provide monotonic relationships between scores and accuracy (measured by satisfied CI statements), they incur higher computational costs and may exhibit greater sensitivity to finite-sample effects.

Formally, we suggest to compute a $p$-value as $p = \frac{1}{m}\sum_{i=1}^m \mathbbm{1}[\text{score}(\widehat{\mathcal{G}}_{\mathcal{R}_i}) \geq \text{score}(\widehat{\mathcal{G}})]$, where $\{\widehat{\mathcal{G}}_{\mathcal{R}_i}\}_{i=1}^m$ are randomly sampled graphs preserving relevant properties (e.g., edge density), and $\widehat{\mathcal{G}}$ is the expert prediction. If $p < \alpha_{\text{val}}$ for significance level $\alpha_{\text{val}}$ (e.g., 0.05), indicating the expert guess scores significantly better than random, we incorporate the guidance into test prioritization; otherwise, we proceed without expert guidance to maintain baseline performance guarantees.

We note that there are roughly 3 possible scenarios where it is difficult to directly assess the quality of a guess. The first is that data-driven methods may require strong assumptions on the DGP, such as parametric functional assumptions, that may not be satisfied. Additionally, some types of graph structures, such as dense graphs, provide relatively few CIs for nonparametric scores to validate, leading to low power in those regimes. Additionally, while there are many methods for assigning a score to causal graphs, there are fewer methods for assessing the fittingness of causal orderings, a fundamentally more difficult problem. Therefore, in these cases where it is hard to leverage the data itself, we propose to rely on a data-independent measure: the self-consistency of the expert.

\textbf{Expert Self-Consistency.} As suggested by \cite{faller2024self}, we assess expert reliability by evaluating self-consistency across predictions on overlapping variable subsets. This data-free approach requires the expert to produce causal graphs for multiple variable subsets and measures agreement where these subsets overlap—paralleling recent findings that LLM response consistency under prompt variations correlates with reliability. We suggest to compute a consistency score $\mathcal{C}(\psi) = \frac{1}{|\mathcal{P}|} \sum_{(S_1, S_2) \in \mathcal{P}} \text{sim}(\widehat{\mathcal{G}}_\psi(S_1 \cup S_2)|_{S_1 \cap S_2}, \widehat{\mathcal{G}}_\psi(S_1 \cap S_2))$, where $\psi$ is the expert, $\mathcal{P}$ is a collection of overlapping variable subset pairs, and $\text{sim}(\cdot, \cdot)$ is some sort of measure of structural similarity on the shared variables. This metric provides theoretical guarantees: a perfectly accurate expert must exhibit perfect self-consistency. If $\mathcal{C}(\psi) > \tau_c$ for a pre-specified threshold $\tau_c$ (e.g., 0.7), we proceed with expert-guided discovery; otherwise, we default to vanilla PC to preserve baseline performance. While this approach effectively screens for domain competence, it faces two limitations: the computational cost of querying numerous subsets may be prohibitive, and it evaluates general expert reliability rather than the quality of specific structural predictions.

\subsubsection{Integrating Uncertainty Quantification}

We note that PC-Guess can be extended to leverage confidence estimates, for use in situations where an expert can do uncertainty quantification for their predictions. For example, the ordering algorithms (Subroutines \ref{algo: expert-edge-ordering1}, \ref{algo: expert-subset-ordering1})
can be modified such that the orderings are not drawn uniformly from two buckets, but rather according to confidence score. More formally, when the expert provides edge confidence scores $w_\psi(x_i, x_j) \in [0,1]$ and ordering confidence $\pi_\psi(x_i \prec x_j) \in [0,1]$, we suggest to modify test prioritization to sample probabilistically rather than deterministically. For edge testing order, we sample the next pair $(x_i, x_j)$ with probability proportional to $1 - w_\psi(x_i, x_j)$, favoring likely non-edges. For conditioning set selection, we sample set $S$ with probability proportional to $\prod_{z \in S} \pi_\psi(z \prec \{x_i, x_j\})$, favoring sets containing likely ancestors. This probabilistic approach naturally interpolates between expert-guided and random search based on prediction confidence, maintains theoretical guarantees (all tests have non-zero probability), and becomes exploratory under high uncertainty. With uniform uncertainties, it reduces to standard PC.

\subsubsection{Leveraging Causal Reasoning}

Additionally, when economically or computationally feasible, after a number of structural decisions (e.g., edge confirmation/removal) are made, we could in theory reprompt the expert to provide a new guess using the updated partial structure, repeating the process until the algorithm terminates. This iterative refinement would then leverage the compositional reasoning abilities of experts to produce progressively better predictions as structural information is confirmed, allowing them to correct earlier mistakes and incorporate validated constraints. To fully exploit expert capabilities, we reprompt the expert for updated predictions after each edge decision during discovery. This iterative approach offers two key advantages: first, as the structure is progressively verified, the problem dimensionality decreases, where LLMs demonstrably achieve higher accuracy on smaller graphs; second, the confirmed structural information provides additional context that allows the expert to correct earlier mistakes and make more informed predictions about remaining edges. Thus, rather than using a static initial guess, we dynamically update expert guidance throughout the discovery process.

\subsection{Extension to Score-Based Methods}\label{appendix: cd-guess score extension}

\textbf{Method description.} Score-based algorithms like GES \citep{chickering_optimal_2002} discover causal structure by greedily optimizing a scoring function, iteratively adding or removing edges that maximize the score improvement. A variant called SE-GES \citep{chickering_statistically_2020} restricts the search to statistically efficient operators—those conditioning on the fewest variables—to improve finite-sample performance:

{
\renewcommand{\thealgorithm}{A}
\begin{algorithm}[h!]
\caption{Statistically Efficient Greedy Equivalence Search (SE-GES)}
\label{algo: se-ges}
\begin{algorithmic}[1]
\State \textbf{Input}: Data $D$
\State \textbf{Output}: CPDAG $\mathcal{C}$
\State $\mathcal{C} \leftarrow$ FINDIMAP($D$)
\State $M^{-1} \leftarrow$ UNDEFINED for all node pairs
\State $k \leftarrow 0$
\Repeat
    \State $M^k \leftarrow$ UPDATESEPARATORS($M^{k-1}$, $k$)
    \State $\mathcal{C} \leftarrow$ SE-BES($\mathcal{C}$, $k$)
    \If{every node in $\mathcal{C}$ has $\leq k$ parents}
        \State \Return $\mathcal{C}$
    \Else
        \State $k \leftarrow k + 1$
    \EndIf
\Until{convergence}
\end{algorithmic}
\end{algorithm}
\addtocounter{algorithm}{-1}
}

SE-GES (Algorithm 2) operates as follows. First (Line 1), it applies FINDIMAP to identify an initial graph structure which contains the conditional independencies in the data. It then progressively iterates through increasing values of bound $k$ on the conditioning set size, starting from $k=0$. At each iteration, two key steps occur:

\textit{Line 5 - UPDATESEPARATORS}: This subroutine identifies all order-$k$ separators $M^k$ by testing conditional independence statements of size $k$. These separators capture which variable pairs are conditionally independent given $k$ conditioning variables, forming the statistical basis for edge removal decisions. Specifically, UPDATESEPARATORS tests CI statements of the form CIT($x_i, x_j \mid W$) where $|W| = k$, storing those that return independent as order-$k$ separators. This collection $M^k$ grows as $k$ increases, accumulating evidence about which edges can be safely removed.

\textit{Line 6 - SE-BES}: Using the separators $M^k$, SE-BES performs backward elimination by evaluating deletion operators (edge removals) while restricting to those of order at most $k$, optimizing a scoring function that balances fit and complexity.

After SE-BES completes, Line 7 checks whether the resulting CPDAG is consistent with bound $k$—whether every node has at most $k$ parents. If consistent, the algorithm terminates; otherwise, it increments $k$ and repeats, testing higher-order conditional independencies.

\textbf{Computational challenge and proposed modification.} As graph size grows, exhaustively testing all $k$-order CI statements becomes computationally prohibitive. This necessitates imposing a budget $Q$ on the number of CI tests per edge per round, inducing a preference over which tests to prioritize and in what order to execute them.

We propose using G2G-guided test ordering as a subroutine within SE-GES. Specifically, for each conditioning set size $k$, rather than testing CI statements uniformly at random, we apply the ordering principles from PC-Guess (Section~\ref{sec: g2g cb theory}) to prioritize tests: (1) extract edge predictions from expert graph $\widehat{\mathcal{G}}$, (2) prioritize testing edges the expert believes are false (likely to yield independence), and (3) within each edge's tests, prioritize conditioning sets the expert predicts are d-separating.

\textbf{Potential benefits.} This modification offers several advantages while preserving SE-GES's theoretical guarantees:
\begin{itemize}
    \item \emph{Budget efficiency}: Under limited testing budgets, prioritizing tests more likely to reveal true independencies increases the probability of correctly identifying removable edges within the allocated $Q$ tests per edge
    \item \emph{Search space reduction}: Earlier identification of true non-edges prunes the graph faster, reducing downstream computational costs
    \item \emph{Preserved guarantees}: Since SE-GES's correctness depends only on testing sufficient CI statements (not their order), expert-guided prioritization cannot compromise asymptotic correctness
\end{itemize}

\textbf{Conjectured theoretical properties.} Analogous to our constraint-based results (Theorems~\ref{theorem: guess acc correct}, \ref{theorem: sgs performance}), we conjecture that expert-guided SE-GES satisfies similar guarantees. Specifically, we expect that as sample size $n \to \infty$ and testing budget $Q \to \infty$, the algorithm maintains statistical consistency and recovers the true skeleton regardless of expert quality, since all edges eventually receive their full budget of tests. For fixed $Q$ and $n$, we conjecture that the expected probability of correct skeleton recovery increases monotonically with expert edge prediction accuracy $p_\psi$—better experts identify true non-edges earlier, increasing the probability that budget-limited tests discover removable edges. Finally, we expect finite-sample robustness: when $p_\psi \geq 0.5$, expert-guided test ordering should achieve performance at least as good as random ordering, since experts performing at chance produce orderings distributionally equivalent to random, with strictly better performance when $p_\psi > 0.5$. While we have not completed formal proofs of these properties, the intuition directly mirrors our constraint-based analysis, and we believe the results could be established using similar coupling arguments we use in our proofs.

Prior work has proposed expert-augmented score-based methods through heuristics like initialization with expert guesses or preferring operations agreeing with expert predictions \citep{constantinou2023impact, ejaz_less_2025}, but without formal guarantees. Our G2G framework suggests a principled approach with conjectured theoretical properties worth investigating in score-based methods.

\subsection{Extension to ANM-Based Methods}\label{appendix: cd-guess anm extension}

\textbf{Method description.} ANM-based algorithms discover causal structure by exploiting asymmetries in functional relationships under additive noise models. A prominent example is RESIT \citep{peters_causal_2014}, which identifies the causal DAG by iteratively finding leaf nodes (variables with no children):

{
\renewcommand{\thealgorithm}{B}
\begin{algorithm}[h!]
\caption{RESIT (Simplified)}
\label{algo: resit}
\begin{algorithmic}[1]
\State \textbf{Input}: Data $D$ on variables $V = \{x_1, \ldots, x_d\}$
\State \textbf{Output}: Parent sets $\{\text{Pa}(x_i)\}_{i=1}^d$
\State $S \leftarrow V$, $\pi \leftarrow []$ \Comment{$S$ tracks remaining variables, $\pi$ builds topological order}
\State \textbf{Phase 1: Determine topological order}
\Repeat
    \For{each $x_k \in S$}
        \State Regress $x_k$ on $S \setminus \{x_k\}$
        \State Test independence between residuals and $S \setminus \{x_k\}$
        \State Record dependence strength
    \EndFor
    \State Let $x_{k^*}$ be variable with weakest dependence
    \State $\text{Pa}(x_{k^*}) \leftarrow S \setminus \{x_{k^*}\}$
    \State $S \leftarrow S \setminus \{x_{k^*}\}$
    \State $\pi \leftarrow [x_{k^*}, \pi]$ \Comment{Prepend to topological order}
\Until{$S = \emptyset$}
\State \textbf{Phase 2: Prune superfluous edges} (analogous to constraint-based edge removal)
\State \Return $\{\text{Pa}(x_i)\}_{i=1}^d$
\end{algorithmic}
\end{algorithm}
\addtocounter{algorithm}{-1}
}

RESIT (Algorithm B) operates in two phases. Phase 1 (Lines 4-13) iteratively identifies leaf nodes by testing all remaining variables to find which has residuals most independent of the others—this variable is a leaf and removed from consideration. The algorithm builds a topological ordering $\pi$ by prepending each identified leaf. Phase 2 (Line 14) then prunes unnecessary edges from the identified parent sets.

The critical computational bottleneck is Line 5: at each iteration with $|S|$ remaining variables, RESIT must test all $|S|$ candidates, requiring $O(d^2)$ total tests across iterations. Moreover, the order in which variables are tested affects accuracy—if a non-leaf is incorrectly identified as a leaf due to test errors (false positive), this error propagates through subsequent iterations.

\textbf{Computational challenge and proposed modification.} The key insight is that the order in which variables are tested in Line 5 critically impacts both efficiency and accuracy. Testing true leaves first reduces the number of tests (since the leaf is found immediately) and prevents false positive errors from incorrectly identifying non-leaves.

We propose modifying RESIT to use G2G-guided variable ordering. Specifically: (1) query expert $\psi$ for a predicted topological ordering $\hat{\pi}$ over variables $V$, (2) in Line 5, test variables in the order specified by $\hat{\pi}$ rather than arbitrary order, prioritizing variables the expert predicts appear later in topological order (more likely to be leaves). This modification directly parallels our edge ordering principle from Section~\ref{sec: g2g cb theory}—prioritizing tests more likely to succeed.

\textbf{Potential benefits.} This modification offers several advantages:
\begin{itemize}
    \item \emph{Computational efficiency}: Testing true leaves first reduces the expected number of independence tests per iteration from $O(|S|)$ to $O(1)$ when expert predictions are accurate
    \item \emph{Error reduction}: Identifying true leaves early prevents false positives on non-leaves, reducing error propagation through the topological ordering
    \item \emph{Preserved guarantees}: Since RESIT's correctness depends on testing all variables (not their order), expert-guided prioritization cannot compromise asymptotic correctness
\end{itemize}

\textbf{Conjectured theoretical properties.} Analogous to our constraint-based results (Theorems~\ref{theorem: guess acc correct}, \ref{theorem: sgs performance}), we conjecture that expert-guided RESIT satisfies similar guarantees. Let $p_\pi$ denote the expert's accuracy in predicting topological orderings (probability that for a random pair $(x_i, x_j)$, if $x_i$ is an ancestor of $x_j$, then $\hat{\pi}(x_i) < \hat{\pi}(x_j)$). We expect that as sample size $n \to \infty$, the algorithm maintains statistical consistency and recovers the true DAG structure regardless of $p_\pi$, since all variables are eventually tested. For fixed $n$, we conjecture that the expected probability of correct DAG recovery increases monotonically with $p_\pi$—better topological ordering predictions place true leaves earlier in the test sequence, increasing the probability of correct identification before errors accumulate. We also expect finite-sample robustness: when $p_\pi \geq 0.5$, expert-guided ordering should achieve performance at least as good as random variable ordering, since random orderings provide the baseline ($p_\pi = 0.5$) with improvement when $p_\pi > 0.5$. Finally, we conjecture that the expected number of tests decreases monotonically with $p_\pi$ because testing true leaves first terminates the inner loop earlier, providing computational efficiency gains. While we have not completed formal proofs of these properties, the intuition follows directly from our core insights, and we believe the results are straightforward to establish using similar analytical techniques.

This extension demonstrates how G2G principles—using expert predictions to guide test sequences—apply beyond constraint-based methods to functional causal model approaches, suggesting potential broad applicability of the framework.

\subsection{Decomposition of Constraint-Based Discovery into Edge Prune and Edge Loop}\label{appendix: cb subroutine decomposition}

We demonstrate how two constraint-based algorithms—PC and rPC-approx—decompose into the Edge Loop (EL) and Edge Prune (EP) subroutines defined in Section~\ref{sec: g2g cb theory}.

\textbf{PC Algorithm.} The PC algorithm iteratively tests edges with increasing conditioning set sizes $\ell$, removing edges when independence is found.

{
\renewcommand{\thealgorithm}{C}
\begin{algorithm}[h!]
\caption{PC (Skeleton Discovery)}
\label{algo: pc-decomposed}
\begin{algorithmic}[1]
\State \textbf{Input}: Complete skeleton $\mathcal{C}$ over variables $V$
\State Sample $\mathsf{O}$ uniformly, sample $\mathsf{L}$ uniformly
\State Define $R_{PC}^{\ell}(\mathcal{C}, e_{i,j}) = \mathbbm{1}[n_{i,j} \in \mathcal{C}] \land \mathbbm{1}[|\adj_{-j}(\mathcal{C}, x_i)| \geq \ell]$
\For{$\ell = 0$ to $|V|-1$}
    \State $\mathcal{C} \leftarrow$ EL($\mathcal{C}$, $\mathsf{O}$, [$\ell$, $\ell$], $\mathsf{L}$, $R_{PC}^{\ell}$, EP)
    \State \textbf{if} no edges satisfy $R_{PC}^{\ell}$ \textbf{then break}
\EndFor
\State \Return $\mathcal{C}$
\end{algorithmic}
\end{algorithm}
\addtocounter{algorithm}{-1}
}

\textbf{rPC-approx Algorithm.} The rPC-approx algorithm \citep{sondhi2019reduced} bounds conditioning set sizes to $\eta < |V|-1$ and modifies the conditioning set search space to use local neighborhoods.

{
\renewcommand{\thealgorithm}{D}
\begin{algorithm}[h!]
\caption{rPC-approx (Skeleton Discovery)}
\label{algo: rpc-decomposed}
\begin{algorithmic}[1]
\State \textbf{Input}: Complete skeleton $\mathcal{C}$ over variables $V$, maximum size $\eta$
\State Sample $\mathsf{O}$ uniformly, sample $\mathsf{L}$ uniformly
\State Define $R_{\text{rPC}}(\mathcal{C}, e_{i,j}) = \mathbbm{1}[n_{i,j} \in \mathcal{C}]$
\For{$\ell = 0$ to $\eta$}
    \State $\mathcal{C} \leftarrow$ EL($\mathcal{C}$, $\mathsf{O}$, [$\ell$, $\ell$], $\mathsf{L}$, $R_{\text{rPC}}$, EP)
\EndFor
\State \Return $\mathcal{C}$
\end{algorithmic}
\end{algorithm}
\addtocounter{algorithm}{-1}
}

\textbf{Key differences in decomposition:}
\begin{itemize}
    \item \emph{Conditioning set size bound}: PC iterates up to $|V|-1$; rPC-approx stops at $\eta$.
    \item \emph{Validity rule}: PC's $R_{PC}^{\ell}$ requires $|\adj_{-j}(\mathcal{C}, x_i)| \geq \ell$ (sufficient adjacency); rPC-approx's $R_{\text{rPC}}$ only checks edge presence.
    \item \emph{Conditioning set source}: In EP, PC draws subsets from $\adj_{-j}(\mathcal{C}, x_i)$; rPC-approx draws from $\mathrm{adj}(\mathcal{C}, x_i) \cup \mathrm{adj}(\mathcal{C}, x_j) \setminus \{x_i, x_j\}$ (union of both endpoints' neighborhoods).
\end{itemize}

Both algorithms follow the same template—iteratively calling EL with increasing $\ell$—differing only in their validity rules, conditioning set size bounds, and the search space for conditioning sets within EP.

\subsection{Complexities of Error Propagation in Edge Loop}\label{appendix: edge loop error propogation}

We explain why error-tolerant metrics—such as expected number of correct edges $\mathbb{E}[\sum_{n_{i,j}} Y_{n_{i,j}}]$—are analytically intractable for ordering optimization, motivating our focus on perfect recovery probability $\Phi$ in Section~\ref{subsec:metric}.

\textbf{Notation.} Recall that true edges $n_{i,j} \in \mathcal{S}^*$ are edges present in the true skeleton, while false edges $n_{i,j} \notin \mathcal{S}^*$ are edges absent from the true skeleton but present in the current candidate skeleton $\mathcal{C}$ being tested.

\textbf{Requirements for analyzing error-tolerant metrics.} To optimize orderings under error-tolerant metrics, we must:
\begin{enumerate}
    \item[(a)] \emph{Identify qualitatively different error types}: Constraint-based algorithms make two types of errors—false positives (incorrectly retaining false edges) and false negatives (incorrectly removing true edges)
    \item[(b)] \emph{Determine the magnitude of each error type's impact}: Quantify how each error affects the probability of correctly classifying subsequent edges
    \item[(c)] \emph{Characterize how sequence changes affect errors}: Predict which sequence modifications reduce overall error rates
\end{enumerate}

\textbf{Opposing downstream effects (Challenge for (a)).} False positive and false negative errors have conflicting impacts on future tests. Consider an edge ordering where edge $n_k$ is incorrectly decided at position $k$. For any subsequent edge $n_j$ where $j > k$:

\emph{False Positive} (incorrectly retaining false edge $n_k \notin \mathcal{S}^*$):
\begin{itemize}
    \item The retained false edge artificially inflates $|\mathrm{adj}(\mathcal{C}, x_m)|$ for vertices $x_m$ adjacent to $n_k$
    \item By Lemma~\ref{lemma: edge-asymmetry}, larger adjacency sets increase the number of CI tests, which:
    \begin{itemize}
        \item \emph{Helps} remove subsequent false edges (more tests $\rightarrow$ higher chance of finding independence)
        \item \emph{Hurts} retention of subsequent true edges (more tests $\rightarrow$ higher chance of spurious independence due to finite-sample errors)
    \end{itemize}
\end{itemize}

\emph{False Negative} (incorrectly removing true edge $n_k \in \mathcal{S}^*$):
\begin{itemize}
    \item The removed true edge artificially deflates $|\mathrm{adj}(\mathcal{C}, x_m)|$ for vertices $x_m$ adjacent to $n_k$
    \item By Lemma~\ref{lemma: edge-asymmetry}, smaller adjacency sets decrease the number of CI tests, which:
    \begin{itemize}
        \item \emph{Helps} retention of subsequent true edges (fewer tests $\rightarrow$ lower chance of spurious independence)
        \item \emph{Hurts} removal of subsequent false edges (fewer tests $\rightarrow$ lower chance of finding true independence; may also exclude crucial d-separating variables from adjacency sets)
    \end{itemize}
\end{itemize}

Since the two error types have opposing effects—false positives benefit false edge detection but harm true edge detection, while false negatives do the reverse— its not clear whether a single ordering strategy dominates without knowing how the effect of the true graph structure, i.e. whether there are many true edges or false edges.

\textbf{Unknown Error Magnitudes (Challenge for (b)).} The magnitude of error propagation depends critically on graph topology. Consider testing edges $n_{i,j}$ and $n_{g,h}$:
\begin{itemize}
    \item If $n_{i,j}$ and $n_{g,h}$ share vertices, errors on $n_{i,j}$ directly modify the adjacency sets used when testing $n_{g,h}$
    \item The severity depends on whether the removed/retained edge contains d-separating variables for $n_{g,h}$
    \item The true/false edge ratio in the graph determines whether false positive or false negative errors dominate overall performance
    \item The error rate values ($\alpha,\beta$ play a role in the likliehood of different types of errors; if the $\beta$ is high or low, this might affect whether we optimize the sequence towards preventing false positives or false negatives.
\end{itemize}

\textbf{Analytical Intractability (Challenge for (c)).} The challenges for (a) and (b) make it difficult to analyze theoretically whether a one sequence dominates another, as (a) presents a challenge to the notion that, even with perfect information of the ground truth, its not clear how to determine the effect of swapping edges in a sequence, and (b) shows that there is important information such as graph structure and error rates missing that may be crucial to the analysis.

\textbf{Conclusion.} This graph-dependent complexity motivates our restriction to perfect recovery probability $\Phi = \mathbb{P}[\prod Y_{n_{i,j}} = 1]$ in Section~\ref{subsec:metric}. By conditioning on no prior errors (perfect recovery up to position $i-1$), we eliminate error propagation from the analysis, enabling the clean characterization in Lemma~\ref{lemma: false-before-true}: placing false edges before true edges is universally beneficial regardless of graph structure.
\newpage
\section{Lemmas, Theorems, and Proofs}\label{appendix: proofs}

\subsection{Proof of Lemma \ref{lemma: EL perfect}}\label{proof: EL perfect}
\begin{restatable}[]{lemma}{ELperfect}\label{lemma: EL perfect}
Under oracle CITs, when $\mathcal{C}$ is the complete graph over $V$, EL($\mathcal{C}$, $\mathsf{O}$, $[0, |V|-1]$, $\mathsf{L}$, $R$, EP) returns the same skeleton for any edge ordering $\mathsf{O}$ (and any $\mathsf{L}$ by Lemma~\ref{lemma: EP perfect}).
\end{restatable}
\begin{proof}
We prove that when given oracle CITs, if  $\mathcal{C}$ is complete then the set of edges removed by EL is independent of $\mathsf{O}$. We show that each edge's fate is determined solely by the existence of d-separating sets, not by the order of processing.

\textbf{All true edges are retained:} Consider any edge $n_{i,j} \in \mathcal{C}$ that is a true edge in the skeleton of $\mathcal{G}$. By faithfulness, $x_i \not\!\perp\!\!\!\perp x_j \mid S$ for all $S \subseteq V \setminus \{x_i, x_j\}$. When EL processes this edge (at any position in $\mathsf{O}$), all CI tests return dependent under perfect tests, so $n_{i,j}$ is never removed. Thus, all true edges remain in $\mathcal{C}$ throughout execution regardless of $\mathsf{O}$.

\textbf{All false edges are removed:} Consider any edge $n_{i,j} \in \mathcal{C}$ that is not in the true skeleton. By faithfulness and causal sufficiency, there exists $S \subseteq V \setminus \{x_i, x_j\}$ such that $x_i \perp\!\!\!\perp x_j \mid S$. By the Markov Condition (Definition \ref{def:markov}), one such set must exist, for example at least one of $\text{Pa}(x_i)$ or $\text{Pa}(x_j)$. Let $S$ be such a d-separating set.

Since $\mathcal{C}$ is initially complete and all parent edges are true edges in the skeleton, these parent edges remain in $\mathcal{C}$ throughout execution (as all true edges are retained). Therefore, regardless of when EL processes edge $n_{i,j}$ in the ordering $\mathsf{O}$, we have $S \subseteq \adj_{-j}(\mathcal{C}, x_i)$ when this edge is tested. EP will test with conditioning set $S$, the test returns independent under perfect CI tests, and $n_{i,j}$ is removed.

Since both the retention of true edges and removal of false edges are independent of $\mathsf{O}$ when $\mathcal{C}$ is complete, EL returns the same skeleton for any edge ordering.
\end{proof}

\subsection{Proof of Lemma \ref{lemma: EP perfect}}\label{proof: EP perfect}
\begin{restatable}[]{lemma}{EPperfect}\label{lemma: EP perfect}
Under oracle CITs, for any edge $e_{i,j}$ and conditioning set size $k$, EP($\mathcal{C}$, $e_{i,j}$, $k$, $\mathsf{L}$) returns the same result for any subset ordering $\mathsf{L}$.
\end{restatable}
\begin{proof}
EP tests all subsets $s \subseteq [\adj_{-j}(\mathcal{C}, x_i)]_k$ and removes edge $n_{i,j}$ if any CIT($x_i, x_j \mid s$) returns independent. With perfect CITs, the outcome of each test is deterministic and depends only on the conditioning set $s$, not on when it is tested. 

If there exists $s \subseteq [\adj_{-j}(\mathcal{C}, x_i)]_k$ such that $x_i \perp\!\!\!\perp x_j \mid s$, then EP will remove $n_{i,j}$ when it tests this conditioning set, regardless of which ordering $\mathsf{L}$ is used to sequence the tests. 

If no such $s$ exists, then all tests return dependent and $n_{i,j}$ remains, again regardless of $\mathsf{L}$.

Since the decision to remove or retain the edge depends only on the existence of a d-separating set among the tested conditioning sets, not the order in which they are tested, EP returns the same result for any $\mathsf{L}$ under perfect CI tests.
\end{proof}

\subsection{Proof Lemma \ref{lemma: edge-asymmetry}}\label{proof: edge-asymmetry}
\begin{restatable}[]{lemma}{edgeasymmetry}\label{lemma: edge-asymmetry}
    For a true edge $n_{i,j}$ (where $n_{i,j} \in \mathcal{S}^*$), adding vertices to the adjacency set $\mathrm{adj}(\mathcal{C}, x_i)$ strictly decreases $\mathbb{P}(Y_{n_{i,j}}=1)$, while for a false edge $n_{i,j}$ (where $n_{i,j} \notin \mathcal{S}^*$) adding vertices strictly increases $\mathbb{P}(Y_{n_{i,j}}=1)$.
\end{restatable}

\begin{proof}
We say an edge $n_{i,j}$ is a \emph{true edge} if $n_{i,j} \in \mathcal{S}^*$ and a \emph{false edge} if $n_{i,j} \notin \mathcal{S}^*$.

Note in the $k^{th}$ step of Subroutine \ref{algo: subset-ordering}, a true edge $n_{i,j}$ is kept if no CIT conditioning on a subset of size $k$ from adjacency set $\mathrm{adj}_{-j}(\mathcal{C}, x_i)$ returns independent. Given that $n_{i,j}$ is a true edge, all CITs of size $k$ return independent with the same false negative rate $\beta$. Then, increasing the size of $\mathrm{adj}(\mathcal{C}, x_i)$ can only increase the number of CITs run, which can only increase the probability that $n_{i,j}$ is incorrectly removed, which therefore decreases $\mathbb{P}(Y_{n_{i,j}}=1)$.

For a false edge $n_{i,j}$, correct specification (i.e., $Y_{n_{i,j}}=1$) occurs when the edge is removed, which happens if and only if at least one CIT in the $k^{th}$ step returns independent. Let $A_1 \subset A_2$ where $A_1 = \mathrm{adj}_{-j}(\mathcal{C}_1, x_i)$ and $A_2 = \mathrm{adj}_{-j}(\mathcal{C}_2, x_i)$ denote two adjacency sets differing by the inclusion of additional vertices in $A_2$. Denote by $\mathcal{T}_1 = [A_1]_k$ and $\mathcal{T}_2 = [A_2]_k$ the corresponding sets of size-$k$ conditioning sets. Since $A_1 \subset A_2$, we have $\mathcal{T}_1 \subset \mathcal{T}_2$. The probability of correct removal is:
$$\mathbb{P}(Y_{n_{i,j}}=1 \mid A) = \mathbb{P}\left(\bigcup_{W \in [A]_k} \{\text{CIT}(x_i, x_j \mid W) = \text{independent}\}\right)$$
For any fixed false edge $n_{i,j} \notin \mathcal{S}^*$, each CIT has a positive probability of returning independent. Since $\mathcal{T}_1 \subset \mathcal{T}_2$, the union over $\mathcal{T}_2$ includes all events in the union over $\mathcal{T}_1$ plus additional events corresponding to tests on conditioning sets in $\mathcal{T}_2 \setminus \mathcal{T}_1$. Therefore:
$$\mathbb{P}(Y_{n_{i,j}}=1 \mid A_1) < \mathbb{P}(Y_{n_{i,j}}=1 \mid A_2)$$.
Thus, increasing $|\mathrm{adj}(\mathcal{C}, x_i)|$ strictly increases $\mathbb{P}(Y_{n_{i,j}}=1)$ for false edges.
\end{proof}

\subsection{Proof of Lemma \ref{lemma: edge-ordering-asymmetry}}\label{proof: edge-ordering-asymmetry}
\begin{restatable}[]{lemma}{EdgeOrderingAsymmetry}\label{lemma: edge-ordering-asymmetry}
    For an ordering $\mathsf{O}$ containing true edge $n_{g,h}$ (where $n_{g,h} \in \mathcal{S}^*$) and false edge $n_{i,j}$ (where $n_{i,j} \notin \mathcal{S}^*$), we have that $\mathbb{P}(Y_{n_{i,j}}=1|Y_{n_{g,h}}=1) = \mathbb{P}(Y_{n_{i,j}}=1)$ while $\mathbb{P}(Y_{n_{g,h}}=1|Y_{n_{i,j}}=1) \geq \mathbb{P}(Y_{n_{g,h}}=1)$.    
    Further, the inequality is strict when $n_{i,j}$ and $n_{g,h}$ share a vertex.
\end{restatable}
\begin{proof}
    We say an edge is a \emph{true edge} if it is in $\mathcal{S}^*$ and a \emph{false edge} if it is not in $\mathcal{S}^*$.
    
    We define $\mathbb{P}(Y_{n_{i,j}}=1|Y_{n_{g,h}}=1)$ as the probability of removing a false edge $n_{i,j}$ after correctly retaining a true edge $n_{g,h}$, and $\mathbb{P}(Y_{n_{i,j}}=1)$ as the probability of removing the false edge $n_{i,j}$ without running a test on $n_{g,h}$.

    If $n_{g,h}$ does not share a vertex with $n_{i,j}$ in $\mathcal{C}$, then
    the adjacency sets of $n_{g,h}$ and $n_{i,j}$ do not overlap. This implies that the event of correctly removing a false edge $n_{i,j}$ is independent of the event of correctly retaining a true edge $n_{g,h}$, as these two events depend only on the data on the node itself and its neighbors. However, if at least one vertex in $n_{g,h}$ is already in the adjacency set of one of the vertices $n_{i,j}$ in $\mathcal{C}$, then correctly retaining $n_{g,h}$ does not change the adjacency set of either vertex in $n_{i,j}$, since retaining an edge means keeping it in the skeleton without modification, which preserves all existing adjacency relationships and therefore does not alter which vertices are available for conditioning. This means which tests are run for $n_{i,j}$ don't change after correctly testing $n_{g,h}$, which means the probability of removing $n_{i,j}$ doesn't change. 
    
    We now define $\mathbb{P}(Y_{n_{g,h}}=1|Y_{n_{i,j}}=1)$ as the probability of retaining true edge $n_{g,h}$ after correctly removing false edge $n_{i,j}$, while $\mathbb{P}(Y_{n_{g,h}}=1)$ as the probability of correctly retaining true edge $n_{g,h}$ without testing $n_{i,j}$. Note that if $n_{i,j}$ does not share any vertex with $n_{g,h}$ in $\mathcal{C}$, then correctly removing $n_{i,j}$ does not affect the probability of retaining $n_{g,h}$. However, if they do share at least one vertex, then correctly removing $n_{i,j}$ reduces the adjacency set of at least one vertex of $n_{g,h}$. Then by Lemma \ref{lemma: edge-asymmetry} the probability of correctly retaining $n_{g,h}$ strictly increases.
\end{proof}

\subsection{Proof of Lemma \ref{lemma: false-before-true}}\label{proof: false-before-true}
\begin{restatable}[]{lemma}{FalseEdgeBeforeTrueEdge}\label{lemma: false-before-true}
    Given a sequence of edges $\mathsf{O}$ that are either true edges (in $\mathcal{S}^*$) or false edges (not in $\mathcal{S}^*$), for any pair of adjacent edges consisting of a true edge followed by a false edge, the sequence generated by swapping the pair is no worse in terms of $\mathbb{P}\left[\bigcap_{n \in \text{sequence}} Y_n = 1\right]$, and strictly better if the false edge and true edge share a node.
\end{restatable}
\begin{proof}
To establish the lemma, it suffices to show that for any adjacent pair where a true edge $n_t$ precedes a false edge $n_f$, swapping their order weakly improves the joint probability $\mathbb{P}\left[\bigcap_{n \in \text{sequence}} Y_n = 1\right]$, with strict improvement when $n_t$ and $n_f$ share a vertex.

We say an edge is a \emph{true edge} if it is in $\mathcal{S}^*$ and a \emph{false edge} if it is not in $\mathcal{S}^*$. Consider a sequence $\mathsf{O} = \ldots, n_i, n_t, n_f, n_j, \ldots$ and the swapped sequence $\mathsf{O}' = \ldots, n_i, n_f, n_t, n_j, \ldots$, where $n_t$ is a true edge and $n_f$ is a false edge.

The key observation is that we need only compare how the swap affects $\mathbb{P}(Y_{n_t} = 1, Y_{n_f} = 1 \mid \text{all prior edges in $\mathsf{O}$ correctly classified})$. This is because: (1) edges processed before $n_t$ and $n_f$ are unaffected by their relative ordering, and (2) conditioned on both $n_t$ and $n_f$ being correctly classified, the resulting skeleton state is identical regardless of their processing order—correctly retaining $n_t$ preserves adjacency sets while correctly removing $n_f$ removes it from adjacency sets, and these operations commute. Therefore, the conditional probabilities for all subsequent edges $n_k$ remain unchanged given correct classification of both $n_t$ and $n_f$.

By Lemma~\ref{lemma: edge-ordering-asymmetry}, we have:
$$\mathbb{P}(Y_{n_f} = 1 | Y_{n_t} = 1, \text{prior edges correct}) = \mathbb{P}(Y_{n_f} = 1 | \text{prior edges correct})$$
since correctly retaining true edge $n_t$ does not modify adjacency sets and therefore does not affect the probability of removing false edge $n_f$. However, also by Lemma~\ref{lemma: edge-ordering-asymmetry}:
$$\mathbb{P}(Y_{n_t} = 1 | Y_{n_f} = 1, \text{prior edges correct}) \geq \mathbb{P}(Y_{n_t} = 1 | \text{prior edges correct})$$
with strict inequality when $n_t$ and $n_f$ share a vertex, since correctly removing $n_f$ reduces the adjacency set of $n_t$, which by Lemma~\ref{lemma: edge-asymmetry} strictly increases the probability of correctly retaining $n_t$.

Therefore, placing $n_f$ before $n_t$ (sequence $\mathsf{O}'$) achieves weakly better joint probability than the original ordering (sequence $\mathsf{O}$), with strict improvement when $n_t$ and $n_f$ share a vertex.
\end{proof}

\subsection{Proof of Lemma \ref{lemma: edge-accuracy-monotonic}}\label{proof: edge-accuracy-monotonic}
\EdgeOrderingAccuracyMonotonicity*
\begin{proof}

We establish that the expected perfect recovery probability $\mathbb{E}[\Phi_{\text{EL-G}}(p_\psi)]$ increases monotonically with expert accuracy $p_\psi$.

\textbf{Setup and notation.} For a fixed partial skeleton $\mathcal{C}$ and true DAG $\mathcal{G}^*$, consider running EL-G with expert accuracy $p_\psi$. The randomness in this process comes from three sources: (1) the expert's prediction $\widehat{\mathcal{G}}$ sampled according to accuracy $p_\psi$, (2) the finite-sample data sampled for use in conditional independence tests, and (3) the random shuffling used when EL-G generates the initial permutation of edges in $\mathcal{C}$ within each partition. Let $\Phi_{p_\psi}$ denote the probability that EL-G produces the true skeleton $\mathcal{S}^*$ after sampling from each of the three sources of randomness. The expectation $\mathbb{E}[\Phi_{p_\psi}]$ is taken over all three sources of randomness.

\textbf{Goal.} To show monotonicity, we must prove that for $p_{\psi_2} > p_{\psi_1}$, we have $\mathbb{E}[\Phi_{p_{\psi_2}}] \geq \mathbb{E}[\Phi_{p_{\psi_1}}]$. It suffices to show that $\Phi_{p_{\psi_2}}$ stochastically dominates $\Phi_{p_{\psi_1}}$.

\textbf{Coupling and stochastic dominance.} We employ a coupling argument. The following classical result provides our main tool (see \citep{lindvall1999strassen} for a more abstract discussion of the result, and see Theorem 4.2.3. in \citep{levin2023coupling} for a more direct formulation):

\begin{theorem}[Strassen's Coupling Theorem]
The real random variable $X$ stochastically dominates $Y$ if and only if there exists a coupling $(\widehat{X}, \widehat{Y})$ of $X$ and $Y$ such that $\mathbb{P}[\widehat{X} \geq \widehat{Y}] = 1$. We refer to $(\widehat{X}, \widehat{Y})$ as a monotone coupling of $X$ and $Y$.
\end{theorem}

A coupling of random variables $X$ and $Y$ is a joint distribution $(\widehat{X}, \widehat{Y})$ where $\widehat{X}$ and $\widehat{Y}$ are two entirely different random variables whose marginal distributions coincide with the distributions of $X$ and $Y$, respectively. More formally, a coupling is a probability measure on the product space whose projections onto each coordinate recover the original distributions. In simpler terms, $(\widehat{X}, \widehat{Y})$ is constructed such that $\widehat{X}$ has the same distribution as $X$, $\widehat{Y}$ has the same distribution as $Y$, but $\widehat{X}$ and $\widehat{Y}$ may be dependent.

\textbf{Example: Bernoulli couplings.} Consider Bernoulli random variables $X$ and $Y$ with $\mathbb{P}[X=1] = q$ and $\mathbb{P}[Y=1] = r$ where $0 \leq q < r \leq 1$.
\begin{itemize}
    \item \emph{Independent coupling}: We can construct $(\widehat{X}, \widehat{Y})$ where $\widehat{X}$ has the same distribution as $X$ and $\widehat{Y}$ has the same distribution as $Y$ and they are independent. This gives joint probabilities $\mathbb{P}[(\widehat{X}, \widehat{Y}) = (i,j)] = \mathbb{P}[\widehat{X}=i]\mathbb{P}[\widehat{Y}=j]$ for $i,j \in \{0,1\}$.
    \item \emph{Monotone coupling}: Alternatively, sample $U$ uniformly from $[0,1]$, and set $\widehat{X} = \mathbb{1}\{U \leq q\}$ and $\widehat{Y} = \mathbb{1}\{U \leq r\}$. Then $(\widehat{X}, \widehat{Y})$ is a coupling with $\mathbb{P}[\widehat{X} \leq \widehat{Y}] = 1$, where $\widehat{X}$ and $\widehat{Y}$ still follow the same Bernoulli distributions as $X$ and $Y$ respectively. This demonstrates that a single source of randomness can induce dependence while preserving marginals.
\end{itemize}

\textbf{Proof strategy.} Our proof proceeds in four steps:
\begin{enumerate}
    \item Describe a hypothetical process for generating a monotone coupling $(\widehat{\Phi}_{p_{\psi_2}}, \widehat{\Phi}_{p_{\psi_1}})$ using shared randomness (analogous to the monotone Bernoulli coupling above). This construction assumes access to the ground truth skeleton $\mathcal{S}^*$ and is purely for theoretical analysis. Importantly, Strassen's theorem only requires showing that such a monotone coupling is \emph{possible} to construct, not that we can construct it in practice with knowledge of only finite samples.
    \item Verify that the marginal distributions coincide with the original distributions: $\widehat{\Phi}_{p_{\psi_2}} \stackrel{d}{=} \Phi_{p_{\psi_2}}$ and $\widehat{\Phi}_{p_{\psi_1}} \stackrel{d}{=} \Phi_{p_{\psi_1}}$.
    \item Show the coupling is monotone: $\mathbb{P}[\widehat{\Phi}_{p_{\psi_2}} \geq \widehat{\Phi}_{p_{\psi_1}}] = 1$.
    \item Conclude from Strassen's theorem that $\mathbb{E}[\Phi_{p_{\psi_2}}] \geq \mathbb{E}[\Phi_{p_{\psi_1}}]$.
\end{enumerate}

\subsubsection{Step 1: Describing the Process for Generating the Montone Coupling $(\widehat{\Phi}_{p_{\psi_2}}, \widehat{\Phi}_{p_{\psi_1}})$}

We construct a coupling between two random variables $\widehat{\Phi}_{p_{\psi_1}}$ and $\widehat{\Phi}_{p_{\psi_2}}$ by describing a hypothetical generative process that uses shared randomness. In our case, both $\widehat{\Phi}_{p_{\psi_1}}$ and $\widehat{\Phi}_{p_{\psi_2}}$ are random variables taking values in $[0,1]$ representing the probability of perfect skeleton recovery, and will be designed to have the same marginal distributions as $\Phi_{p_{\psi_1}}$ and $\Phi_{p_{\psi_2}}$ respectively.

\textbf{Sources of shared randomness.} We fix the expert accuracies $p_{\psi_1}$ and $p_{\psi_2}$, but allow the expert predictions $\widehat{\mathcal{G}}$ to vary. Our coupling uses three sources of shared randomness:
\begin{enumerate}
    \item Let $c$ be the number of edges in $\mathcal{C}$. Suppose there are $k$ true edges in $\mathcal{C}$ (edges in $\mathcal{S}^*$ and $\mathcal{C}$) and $c - k$ false edges (edges in $\mathcal{C}$ but not in $\mathcal{S}^*$). Let $L = [l_1, l_2, \ldots, l_{c}]$ be a random variable corresponding to a uniformly sampled permutation of $k$ ones and $c - k$ zeros. That is, $L$ contains exactly $k$ entries equal to 1 and $c - k$ entries equal to 0, where the ordering is uniformly random among all such permutations.
    \item A collection of independent uniform random variables $\mathcal{R} = \{r_1, r_2, \ldots, r_{c}\}$ where each $r_i \sim \text{Uniform}[0,1]$. These will determine, for each position in $L$ independently, whether the expert correctly classifies the corresponding edge.
    \item Two uniform random permutations: $\pi_T$ over the $k$ true edges in $\mathcal{S}^*$, and $\pi_F$ over the $c - k$ false edges not in $\mathcal{S}^*$. These determine the relative ordering within each partition.
\end{enumerate}

\textbf{Generating the expert prediction at accuracy $p_{\psi_1}$.} We now describe how to sample an expert graph $\widehat{\mathcal{S}}^{(p_{\psi_1})}$ with accuracy $p_{\psi_1}$ using the shared randomness $(L, \mathcal{R}, \pi_T, \pi_F)$. For each position $i$ in $L$, the expert independently classifies the corresponding edge correctly with probability $p_{\psi_1}$:
\begin{itemize}
    \item If $l_i = 1$ (corresponds to a true edge in $\mathcal{S}^*$): the expert correctly predicts this edge exists if $r_i \leq p_{\psi_1}$, otherwise incorrectly predicts it does not exist.
    \item If $l_i = 0$ (corresponds to a false edge not in $\mathcal{S}^*$): the expert correctly predicts this edge does not exist if $r_i \leq p_{\psi_1}$, otherwise incorrectly predicts it exists.
\end{itemize}

\textbf{Constructing the edge ordering $\widehat{\mathsf{O}}^{p_{\psi_1}}$.} Given the expert prediction $\widehat{\mathcal{S}}^{(p_{\psi_1})}$, we construct the edge ordering as specified by EL-G (Subroutine \ref{algo: expert-edge-ordering1}). Initialize two empty lists $B_1 = [], B_2 = []$. For each position $i$ in $L$ (going in the order specified by $L$)
\begin{itemize}
    \item If $l_i = 1$ (true edge) and $r_i \leq p_{\psi_1}$ (correctly classified): add position $i$ to the end of $B_2$.
    \item If $l_i = 1$ (true edge) and $r_i > p_{\psi_1}$ (incorrectly classified): add position $i$ to the end of $B_1$.
    \item If $l_i = 0$ (false edge) and $r_i \leq p_{\psi_1}$ (correctly classified): add position $i$ to the end of $B_1$.
    \item If $l_i = 0$ (false edge) and $r_i > p_{\psi_1}$ (incorrectly classified): add position $i$ to the end of $B_2$.
\end{itemize}
Concatenate the buckets: $L_F^{p_{\psi_1}} = B_1 + B_2$. Within $L_F^{p_{\psi_1}}$, assign relative ordering among true edges using $\pi_T$ and among false edges using $\pi_F$ to obtain the final edge ordering $\widehat{\mathsf{O}}^{p_{\psi_1}}$.

\textbf{Computing $\widehat{\Phi}_{p_{\psi_1}}$.} Given the edge ordering $\widehat{\mathsf{O}}^{p_{\psi_1}}$, let $\widehat{\Phi}_{p_{\psi_1}}$ denote the probability of perfect skeleton recovery by EL (Subroutine \ref{algo: edge-ordering}) if were to randomly draw $n$ finite-samples of the variables $V$ from the DGP $\mathcal{G}^*$.

% we now draw $n$ finite-samples of the variables $V$ from the DGP $\mathcal{G}^*$. Let $\widehat{\Phi}_{p_{\psi_1}}$ denote this probability of perfect skeleton recovery by EL (Subroutine \ref{algo: edge-ordering}) given the specific realization of the edge ordering $\widehat{\mathsf{O}}^{p_{\psi_1}}$ and the finite-sample data.

\textbf{Generating $\widehat{\Phi}_{p_{\psi_2}}$ using the same randomness.} We follow the exact same procedure as above, crucially reusing the same shared randomness $(L, \mathcal{R}, \pi_T, \pi_F)$. The only difference is that we use accuracy $p_{\psi_2}$ instead of $p_{\psi_1}$ when determining expert classifications, and redraw a new batch of finite-sample data. This yields a potentially different expert graph $\widehat{\mathcal{S}}^{(p_{\psi_2})}$, a potentially different edge ordering $\widehat{\mathsf{O}}^{p_{\psi_2}}$, and a potentially different recovery probability $\widehat{\Phi}_{p_{\psi_2}}$.

By this coupling procedure, we generate the joint random variable $(\widehat{\Phi}_{p_{\psi_1}}, \widehat{\Phi}_{p_{\psi_2}})$.

\subsubsection{Step 2: Showing the Marginals of the Two Variables Coincide with Original Distributions}

Our goal is to verify that $\widehat{\Phi}_{p_{\psi_1}} \stackrel{d}{=} \Phi_{p_{\psi_1}}$ and $\widehat{\Phi}_{p_{\psi_2}} \stackrel{d}{=} \Phi_{p_{\psi_2}}$. We focus on showing $\widehat{\Phi}_{p_{\psi_1}} \stackrel{d}{=} \Phi_{p_{\psi_1}}$; the argument for $p_{\psi_2}$ follows identically by symmetry.

\textbf{Reduction to orderings.} Let $\mathsf{O}^{p_{\psi_1}}$ denote the random edge ordering generated by Subroutine \ref{algo: expert-edge-ordering1} using expert accuracy $p_{\psi_1}$, and let $\widehat{\mathsf{O}}^{p_{\psi_1}}$ denote the random edge ordering generated in our coupling procedure (Step 1) using accuracy $p_{\psi_1}$. Given any fixed edge ordering, the probability of perfect skeleton recovery when randomly drawing $n$ finite samples from $\mathcal{G}^*$ is deterministically fixed. Therefore, the distributions of $\Phi_{p_{\psi_1}}$ and $\widehat{\Phi}_{p_{\psi_1}}$ are determined entirely by the distributions of $\mathsf{O}^{p_{\psi_1}}$ and $\widehat{\mathsf{O}}^{p_{\psi_1}}$ respectively. To show $\widehat{\Phi}_{p_{\psi_1}} \stackrel{d}{=} \Phi_{p_{\psi_1}}$, it suffices to show that $\widehat{\mathsf{O}}^{p_{\psi_1}} \stackrel{d}{=} \mathsf{O}^{p_{\psi_1}}$.

\textbf{Decomposition of orderings.} Any edge ordering $\mathsf{O}$ over all $c$ edges can be uniquely decomposed into three components:
\begin{enumerate}
    \item $\pi_T$: the relative ordering (permutation) among the $k$ true edges
    \item $\pi_F$: the relative ordering (permutation) among the $c-k$ false edges  
    \item $\pi_{TF}$: the relative placement of true edges versus false edges (which false/true edges come before which other true/false edges)
\end{enumerate}
Therefore, the distribution over edge orderings is uniquely determined by the joint distribution over $(\pi_T, \pi_F, \pi_{TF})$. We will show that this joint distribution is identical for both $\mathsf{O}^{p_{\psi_1}}$ and $\widehat{\mathsf{O}}^{p_{\psi_1}}$.

\textbf{Analysis of the coupling procedure.} In our coupling construction (Step 1), we explicitly sampled $\pi_T$ and $\pi_F$ uniformly and independently. The relative placement $\pi_{TF}$ is then determined by: for each position $i$ in $L$, whether $l_i$ is correctly classified (with probability $p_{\psi_1}$) determines whether that edge goes to $B_1$ or $B_2$, and then $\pi_{TF}$ corresponds to the partition structure $(B_1, B_2)$. Since each edge's classification is independent and depends only on its true label and $p_{\psi_1}$, we have:
\begin{itemize}
    \item $\widehat{\pi}_T$ is uniform over all permutations of $k$ true edges
    \item $\widehat{\pi}_F$ is uniform over all permutations of $\binom{d}{2}-k$ false edges
    \item $\widehat{\pi}_{TF}$ has distribution determined by $p_{\psi_1}$ (probability of correct classification)
    \item $\widehat{\pi}_T, \widehat{\pi}_F, \widehat{\pi}_{TF}$ are mutually independent
\end{itemize}

\textbf{Analysis of Subroutine \ref{algo: expert-edge-ordering1}.} We now show that $\mathsf{O}^{p_{\psi_1}}$ has the same distributional structure. By construction of Subroutine \ref{algo: expert-edge-ordering1}:
\begin{itemize}
    \item Edges in $\mathcal{C}$ are initially shuffled uniformly at random
    \item Each edge is independently classified as being in $\widehat{\mathcal{S}}$ or not with probability $p_{\psi_1}$ of correct classification
    \item Edges are partitioned into $B_1 = \mathcal{C} \setminus \widehat{\mathcal{S}}$ and $B_2 = \mathcal{C} \cap \widehat{\mathcal{S}}$, preserving their random ordering within each bucket
\end{itemize}

To show $\pi_T$ is uniform: Start with an initial uniform random permutation $\pi_T^{\text{init}}$ of all $k$ true edges. Some subset of size $m \sim \text{Binomial}(k, 1-p_{\psi_1})$ are misclassified and placed in $B_1$, while the remaining $k-m$ are placed in $B_2$. For any fixed $m$ and any fixed choice of which $m$ edges are misclassified, this operation defines a bijection from $\pi_T^{\text{init}}$ to the resulting permutation: given the final permutation and knowing which edges went to which bucket, we can uniquely recover $\pi_T^{\text{init}}$, and vice versa. We note that as bijections preserve uniformity, the marginal distribution of $\pi_T$ (after marginalizing over $m$ and the choice of which edges) remains uniform. By the same argument, $\pi_F$ is uniform.

For independence: The distribution of $\pi_{TF}$ is determined solely by which edges are correctly classified (controlled by $p_{\psi_1}$). For any fixed realization of $\pi_{TF}$ (i.e., fixed partition $(B_1, B_2)$), the above uniformity argument shows that $\pi_T$ and $\pi_F$ remain uniform. Therefore $(\pi_T, \pi_F)$ are independent of $\pi_{TF}$.

\textbf{Conclusion.} Both $\widehat{\mathsf{O}}^{p_{\psi_1}}$ and $\mathsf{O}^{p_{\psi_1}}$ decompose into $(\pi_T, \pi_F, \pi_{TF})$ where each component has identical marginal distributions and the same independence structure. Therefore $\widehat{\mathsf{O}}^{p_{\psi_1}} \stackrel{d}{=} \mathsf{O}^{p_{\psi_1}}$, which implies $\widehat{\Phi}_{p_{\psi_1}} \stackrel{d}{=} \Phi_{p_{\psi_1}}$. By the same argument, $\widehat{\Phi}_{p_{\psi_2}} \stackrel{d}{=} \Phi_{p_{\psi_2}}$.

\subsubsection{Step 3: Showing that the Coupling is Monotone.}

Our goal is to show that $\mathbb{P}[\widehat{\Phi}_{p_{\psi_2}} \geq \widehat{\Phi}_{p_{\psi_1}}] = 1$. We establish this by showing that the edge ordering $\widehat{\mathsf{O}}^{p_{\psi_2}}$ can be obtained from $\widehat{\mathsf{O}}^{p_{\psi_1}}$ through a sequence of beneficial swaps.

\textbf{Structure of the two orderings.} Recall from Step 1 that both $\widehat{\mathsf{O}}^{p_{\psi_1}}$ and $\widehat{\mathsf{O}}^{p_{\psi_2}}$ are generated using the same shared randomness $(L, \mathcal{R}, \pi_T, \pi_F)$. Crucially, the permutations $\pi_T$ (ordering among true edges) and $\pi_F$ (ordering among false edges) are identical in both orderings. The only difference lies in the relative placement $\pi_{TF}$ of true edges versus false edges, which is determined by which edges are correctly classified.

\textbf{More edges correctly classified at higher accuracy.} Since $p_{\psi_1} < p_{\psi_2}$, for each position $i$ in $L$:
\begin{itemize}
    \item If edge $i$ is correctly classified under accuracy $p_{\psi_1}$ (i.e., $r_i \leq p_{\psi_1}$), then it is also correctly classified under accuracy $p_{\psi_2}$ (since $r_i \leq p_{\psi_1} < p_{\psi_2}$)
    \item Additional edges may be correctly classified under $p_{\psi_2}$ that were misclassified under $p_{\psi_1}$ (those with $p_{\psi_1} < r_i \leq p_{\psi_2}$)
\end{itemize}

This means:
\begin{itemize}
    \item True edges that were correctly placed in $B_2$ in $L_F^{p_{\psi_1}}$ remain in $B_2$ in $L_F^{p_{\psi_2}}$
    \item Some true edges that were incorrectly placed in $B_1$ in $L_F^{p_{\psi_1}}$ may move to $B_2$ in $L_F^{p_{\psi_2}}$ (moving rightward)
    \item False edges that were correctly placed in $B_1$ in $L_F^{p_{\psi_1}}$ remain in $B_1$ in $L_F^{p_{\psi_2}}$
    \item Some false edges that were incorrectly placed in $B_2$ in $L_F^{p_{\psi_1}}$ may move to $B_1$ in $L_F^{p_{\psi_2}}$ (moving leftward)
\end{itemize}

\textbf{Characterizing the transformation via inversions.} To formalize how $\widehat{\mathsf{O}}^{p_{\psi_2}}$ relates to $\widehat{\mathsf{O}}^{p_{\psi_1}}$, we introduce a numerical ordering. Assign integers to edges as follows:
\begin{itemize}
    \item Assign $\{1, 2, \ldots, c-k\}$ to the false edges according to their fixed relative order $\pi_F$
    \item Assign $\{c-k+1, \ldots, \binom{d}{2}\}$ to the true edges according to their fixed relative order $\pi_T$
\end{itemize}

Under this assignment, every false edge has a smaller numerical label than every true edge. An \textbf{inversion} is any pair of edges that appears out of numerical order—specifically, a true edge appearing before a false edge in the ordering. Since $\pi_T$ and $\pi_F$ are fixed, edges of the same type (both true or both false) never form inversions relative to each other.

The transformation from $\widehat{\mathsf{O}}^{p_{\psi_1}}$ to $\widehat{\mathsf{O}}^{p_{\psi_2}}$ only moves true edges rightward and false edges leftward. This process cannot create new inversions: if a true edge precedes a false edge in $\widehat{\mathsf{O}}^{p_{\psi_2}}$, that pair must have been in the same order in $\widehat{\mathsf{O}}^{p_{\psi_1}}$. Therefore, the set of inversions in $\widehat{\mathsf{O}}^{p_{\psi_2}}$ is a subset of the inversions in $\widehat{\mathsf{O}}^{p_{\psi_1}}$.

\textbf{Weak Bruhat order and reachability.} The weak Bruhat order on permutations provides a formal framework for this relationship. A standard result (\citep{yessenov2005descents}, Proposition 2.2) states:

\begin{proposition}
For permutations $v$ and $w$, $v \leq w$ in weak Bruhat order if and only if $\mathrm{Inv}(v) \subseteq \mathrm{Inv}(w)$, where $\mathrm{Inv}(\cdot)$ denotes the set of inversions.
\end{proposition}

Equivalently, $v$ can be obtained from $w$ by a sequence of adjacent transpositions $(w_i, w_{i+1})$ where $w_i > w_{i+1}$ in the numerical ordering. In our setting, moving a false edge leftward past an adjacent true edge (moving $w_{i+1}$ left of $w_i$) is precisely such an inversion-reducing swap, since false edges have smaller numerical labels than true edges.

Since $\mathrm{Inv}(\widehat{\mathsf{O}}^{p_{\psi_2}}) \subseteq \mathrm{Inv}(\widehat{\mathsf{O}}^{p_{\psi_1}})$, it follows that $\widehat{\mathsf{O}}^{p_{\psi_2}}$ can be obtained from $\widehat{\mathsf{O}}^{p_{\psi_1}}$ through a sequence of adjacent swaps that move false edges leftward past true edges.

\textbf{Monotonicity via beneficial swaps.} By Lemma \ref{lemma: false-before-true}, each such swap (placing a false edge before a true edge) weakly improves the probability of perfect skeleton recovery, with strict improvement when the swapped edges share a vertex. Since $\widehat{\mathsf{O}}^{p_{\psi_2}}$ is reachable from $\widehat{\mathsf{O}}^{p_{\psi_1}}$ through a sequence of such beneficial swaps, we have $\widehat{\Phi}_{p_{\psi_1}} \leq \widehat{\Phi}_{p_{\psi_2}}$, establishing that $\mathbb{P}[\widehat{\Phi}_{p_{\psi_2}} \geq \widehat{\Phi}_{p_{\psi_1}}] = 1$.

\subsubsection{Step 4: Conclusion.} It follows from Strassen's Coupling Theorem that $\mathbb{E}[\Phi_{p_{\psi_2}}] \geq \mathbb{E}[\Phi_{p_{\psi_1}}]$, with strict inequality when $\mathcal{C}$ contains at least one true edge adjacent to a false edge. We note that when $\mathcal{C}$ contains false edges adjacent to true edges, there are at least one pair of orderings $\widehat{\mathsf{O}}^{p_{\psi_1}}, \widehat{\mathsf{O}}^{p_{\psi_2}}$ that can be created through the process outlined in Step 1 such that a false edge adjacent to a true edge is swapped to be earlier in the sequence. In that case we have 
$\widehat{\Phi}_{p_{\psi_1}} < \widehat{\Phi}_{p_{\psi_2}}$, which implies by the Coupling Theorem that  $\mathbb{E}[\Phi_{p_{\psi_2}}] > \mathbb{E}[\Phi_{p_{\psi_1}}]$.

\end{proof}

\subsection{Proof of Lemma \ref{lemma: orderingindaccuracy}}
\label{proof: orderindaccuracy}
% \orderingindaccuracy*
\begin{restatable}[]{lemma}{orderingindaccuracy}\label{lemma: orderingindaccuracy}
Accuracy $\mathbb{P}(Y_{n_{i,j}}=1)$ is constant for all possible orderings $\mathsf{L}_1, \mathsf{L}_2, \ldots$ of $\text{CIT}^k_{\adj_{-j}(\calC, x_i)}$.
\end{restatable}
\begin{proof}
    Note that in EP (Subroutine \ref{algo: subset-ordering}), edge $n_{i,j}$ is removed when a subset $s$ is found to render $x_i,x_j$ independent, i.e. $x_i\ind x_j |s$. This implies that for $n_{i,j}$ to be removed at least one CI test must return independent, and if no CI tests return independent $n_{i,j}$ will not be removed. Whether at least one CI test returns independent depends only on which CI tests are run (i.e., what the adjacent set $\adj_{-j}(\mathcal{C}, x_i)$ is), implying that the order $\mathsf{L}$ in which the CI tests are run is irrelevant to whether the edge is removed or not. This implies that the probability of accuracy $\mathbb{P}(Y_{n_{i,j}}=1)$ is also independent of ordering.
\end{proof}
\subsection{Proof of Lemma \ref{lemma: no-dseps}}\label{proof: no-dseps}
\begin{restatable}[]{lemma}{nodseps}\label{lemma: no-dseps} 
    Under the assumptions of Definition \ref{assumption:cit_power_independence}, if $\not\exists$ a size $k$ subset of $\adj_{-j}(\mathcal{C}, x_i)$ $s_i$ such that $x_i \perp\!\!\!\perp x_j |s_i$, then any pair of orderings $\mathsf{L},\mathsf{L}'$ achieves the same $\mathbb{E}[t_{e_{i,j}}]$, where $t_{e_{i,j}}$ is the number of tests conducted by EP using either $\mathsf{L}$ or $\mathsf{L'}$.
\end{restatable}
\begin{proof}
    Suppose $\not\exists$ a size $k$ subset $s_i$ of $\adj_{-j}(\mathcal{C}, x_i)$ such that $x_i \perp\!\!\!\perp x_j |s_i$. Then for all $\text{CIT}(x_i, x_j \mid s_u) \in \text{CIT}^k_{\adj_{-j}(\mathcal{C}, x_i)}$, the subset $s_u$ is not a d-separating set of $x_i,x_j$, meaning all tests should return dependent under an oracle.
    
    The runtime is determined by when the first test returns independent (or when all tests have been run). By Definition \ref{assumption:cit_power_independence}, since no d-separating sets exist, each CIT falsely returns independent with the same false negative rate $\beta$, and the outcomes of these tests are mutually independent. Therefore, the timing of when a test will return independent is the same no matter the order of the tests, implying that $\mathbb{E}[t_{e_{i,j}}]$ remains constant for all orderings of tests.
\end{proof}

\subsection{Proof of Lemma \ref{lemma: optimal-sequence}}\label{proof: optimal-sequence}
\begin{restatable}[]{lemma}{OptimalSequenceOrdering}\label{lemma: optimal-sequence} 
Under the assumptions of Definition \ref{assumption:cit_power_independence}, given a sequence of CITs $\mathsf{L}$ for edge $n_{i,j}$, for any pair of adjacent CITs consisting of a test on a non-d-separating set $s_1$ followed by a test on a d-separating set $s_2$, the sequence $\mathsf{L}'$ generated by swapping the pair to place the d-separating test $s_2$ first achieves strictly better runtime, i.e., $\mathbb{E}[t_{e_{i,j}}^{\mathsf{L}'}] < \mathbb{E}[t_{e_{i,j}}^{\mathsf{L}}]$, where $t_{e_{i,j}}^{\mathsf{L}}$ is the number of tests conducted by EP under ordering $\mathsf{L}$.
\end{restatable}

\begin{proof}
We say a CIT$=I$ if it returns independent, and CIT$=N$ if it returns dependent (not independent). For edge $n_{i,j}$, let $\text{CIT}^{d-sep}_i$ denote a CIT at position $i$ in the sequence that tests a d-separating set of $x_i, x_j$, and let $\text{CIT}^{non}_i$ denote a CIT at position $i$ that tests a non-d-separating set of $x_i, x_j$.

Consider a pair of orderings $\mathsf{L}, \mathsf{L}'$ that differ in only two positions, where the tests are swapped: ordering $\mathsf{L} = \ldots, \text{CIT}^{non}_{i}, \text{CIT}^{d-sep}_{i+1}, \ldots$, while ordering $\mathsf{L}'= \ldots, \text{CIT}^{d-sep}_{i}, \text{CIT}^{non}_{i+1}, \ldots$. Under the mutual independence assumption (Definition \ref{assumption:cit_power_independence}), we can simplify the difference between the expected runtimes of the orderings as:
\begin{align}
    \mathbb{E}[t_{e_{i,j}}^{\mathsf{L}'}] - \mathbb{E}[t_{e_{i,j}}^{\mathsf{L}}] &= \left(\prod_{j=1}^{i-1}\mathbb{P}(\text{CIT}_{j}=N)\right)\big[ \mathbb{P}(\text{CIT}^{d-sep}_{i}=I) \cdot i \notag\\
    &\quad +\mathbb{P}(\text{CIT}^{d-sep}_{i}=N) \mathbb{P}(\text{CIT}^{non}_{i+1}=I) \cdot (i+1)\big] \notag\\
    &\quad - \left(\prod_{j=1}^{i-1}\mathbb{P}(\text{CIT}_{j}=N)\right)\big[ \mathbb{P}(\text{CIT}^{non}_{i}=I) \cdot i \notag\\
    &\quad +\mathbb{P}(\text{CIT}^{non}_{i}=N) \mathbb{P}(\text{CIT}^{d-sep}_{i+1}=I) \cdot (i+1)\big]
\end{align}
Dividing both sides by the positive value $\left(\prod_{j=1}^{i-1}\mathbb{P}(\text{CIT}_{j}=N)\right)$ yields the following on the RHS:
\begin{align}
    &= \big[ \mathbb{P}(\text{CIT}^{d-sep}_{i}=I) \cdot i +\mathbb{P}(\text{CIT}^{d-sep}_{i}=N) \mathbb{P}(\text{CIT}^{non}_{i+1}=I) \cdot (i+1)\big] \notag\\
    &\quad - \big[ \mathbb{P}(\text{CIT}^{non}_{i}=I) \cdot i +\mathbb{P}(\text{CIT}^{non}_{i}=N) \mathbb{P}(\text{CIT}^{d-sep}_{i+1}=I) \cdot (i+1)\big].
\end{align}
Filling in the probabilities with the correct false positive and false negative rates (where $\alpha$ is the false positive rate and $\beta$ is the false negative rate) yields:
\begin{align}
    &= [(1-\alpha) \cdot i + \alpha \beta (i+1)] - [\beta \cdot i + (1-\beta)(1-\alpha) \cdot (i+1)] \notag\\
    &= [(1-\alpha) \cdot i + \alpha\beta (i+1)] - [\beta \cdot i + (1-\beta -\alpha+\alpha\beta) \cdot (i+1)]
\end{align}
We note that both terms in the above equation can be rewritten as weighted sums of $i, i+1$, with the weights adding up to $1-\alpha+\alpha\beta$. Under the assumption that $1-\alpha > \beta$, this implies that the weight on $i$ is higher in the first term, which implies the weight on $i+1$ is higher in the second term. As $i < i+1$, under the rearrangement inequality this implies that the sum is negative, and therefore $\mathbb{E}[t_{e_{i,j}}^{\mathsf{L}'}] - \mathbb{E}[t_{e_{i,j}}^{\mathsf{L}}]<0$.
\end{proof}

\subsection{Proof of Lemma \ref{lemma: expert-accuracy}}\label{proof: expert-accuracy}

\begin{restatable}[Monotonicity of EP Runtime in Expert Accuracy]{lemma}{ExpertAccuracyMonotonicity}\label{lemma: expert-accuracy}
Under Assumption~\ref{assumption:cit_power_independence}, let $T_{\text{EP-G}}(p_{\text{d-sep}})$ denote the number of tests executed by EP-Guess (Subroutine~\ref{algo: expert-subset-ordering1}) when testing edge $n_{i,j}$ at conditioning set size $k$ with expert $\psi$ having d-separation accuracy $p_{\text{d-sep}}$. Then:
\begin{enumerate}
    \item[(a)] $\mathbb{E}[T_{\text{EP-G}}(p_{\text{d-sep}})]$ decreases monotonically with $p_{\text{d-sep}}$, strictly decreasing when $[A]_k$ contains both d-separating and non-d-separating sets (where $A = \adj_{-j}(\calC, x_i)$)
    \item[(b)] When $p_{\text{d-sep}} \geq 0.5$, $\mathbb{E}[T_{\text{EP-G}}(p_{\text{d-sep}})] \leq \mathbb{E}[T_{\text{random}}]$ where $T_{\text{random}}$ denotes runtime under random ordering
\end{enumerate}
\end{restatable}

\begin{proof}

We establish that the expected runtime $\mathbb{E}[T_{\text{EP-G}}(p_{\text{d-sep}})]$ decreases monotonically with expert d-separation accuracy $p_{\text{d-sep}}$.

\textbf{Setup and notation.} For a fixed partial skeleton $\mathcal{C}$, true DAG $\mathcal{G}^*$, edge $n_{i,j} \in \mathcal{C}$, and conditioning set size $k$, consider running EP-Guess with expert accuracy $p_{\text{d-sep}}$. EP-Guess tests conditional independence CIT($x_i, x_j \mid W$) for all $W \in [A]_k$ where $A = \adj_{-j}(\calC, x_i)$, terminating when the first test returns independence. The randomness in this process comes from three sources: (1) the expert's prediction $\widehat{\mathcal{G}}$ sampled according to accuracy $p_{\text{d-sep}}$, (2) the finite-sample data $D$ sampled for use in conditional independence tests, and (3) the random shuffling used when EP-Guess generates the initial permutation of conditioning sets within each partition. Let $T_{p_{\text{d-sep}}}$ denote the number of tests executed by EP before termination after sampling from each of the three sources of randomness. The expectation $\mathbb{E}[T_{p_{\text{d-sep}}}]$ is taken over all three sources of randomness.

\textbf{Goal.} To show monotonicity, we must prove that for $p_{\text{d-sep}_2} > p_{\text{d-sep}_1}$, we have $\mathbb{E}[T_{p_{\text{d-sep}_1}}] \geq \mathbb{E}[T_{p_{\text{d-sep}_2}}]$. By Strassen's theorem, it suffices to show that $T_{p_{\text{d-sep}_1}}$ stochastically dominates $T_{p_{\text{d-sep}_2}}$.

\textbf{Coupling and stochastic dominance.} We employ a coupling argument. The following classical result provides our main tool (see \citep{lindvall1999strassen} for a more abstract discussion of the result, and see Theorem 4.2.3. in \citep{levin2023coupling} for a more direct formulation):

\begin{theorem}[Strassen's Coupling Theorem]
The real random variable $X$ stochastically dominates $Y$ if and only if there exists a coupling $(\widehat{X}, \widehat{Y})$ of $X$ and $Y$ such that $\mathbb{P}[\widehat{X} \geq \widehat{Y}] = 1$. We refer to $(\widehat{X}, \widehat{Y})$ as a monotone coupling of $X$ and $Y$.
\end{theorem}

A coupling of random variables $X$ and $Y$ is a joint distribution $(\widehat{X}, \widehat{Y})$ where $\widehat{X}$ and $\widehat{Y}$ are two entirely different random variables whose marginal distributions coincide with the distributions of $X$ and $Y$, respectively. More formally, a coupling is a probability measure on the product space whose projections onto each coordinate recover the original distributions. In simpler terms, $(\widehat{X}, \widehat{Y})$ is constructed such that $\widehat{X}$ has the same distribution as $X$, $\widehat{Y}$ has the same distribution as $Y$, but $\widehat{X}$ and $\widehat{Y}$ may be dependent.

\textbf{Example: Bernoulli couplings.} Consider Bernoulli random variables $X$ and $Y$ with $\mathbb{P}[X=1] = q$ and $\mathbb{P}[Y=1] = r$ where $0 \leq q < r \leq 1$.
\begin{itemize}
    \item \emph{Independent coupling}: We can construct $(\widehat{X}, \widehat{Y})$ where $\widehat{X}$ has the same distribution as $X$ and $\widehat{Y}$ has the same distribution as $Y$ and they are independent. This gives joint probabilities $\mathbb{P}[(\widehat{X}, \widehat{Y}) = (i,j)] = \mathbb{P}[\widehat{X}=i]\mathbb{P}[\widehat{Y}=j]$ for $i,j \in \{0,1\}$.
    \item \emph{Monotone coupling}: Alternatively, sample $U$ uniformly from $[0,1]$, and set $\widehat{X} = \mathbb{1}\{U \leq q\}$ and $\widehat{Y} = \mathbb{1}\{U \leq r\}$. Then $(\widehat{X}, \widehat{Y})$ is a coupling with $\mathbb{P}[\widehat{X} \leq \widehat{Y}] = 1$, where $\widehat{X}$ and $\widehat{Y}$ still follow the same Bernoulli distributions as $X$ and $Y$ respectively. This demonstrates that a single source of randomness can induce dependence while preserving marginals.
\end{itemize}

\textbf{Proof strategy.} Our proof proceeds in four steps:
\begin{enumerate}
    \item Describe a hypothetical process for generating a monotone coupling $(\widehat{T}_{p_{\text{d-sep}_1}}, \widehat{T}_{p_{\text{d-sep}_2}})$ using shared randomness (analogous to the monotone Bernoulli coupling above). This construction assumes access to the ground truth DAG $\mathcal{G}^*$ and is purely for theoretical analysis. Importantly, Strassen's theorem only requires showing that such a monotone coupling is \emph{possible} to construct, not that we can construct it in practice with knowledge of only finite samples.
    \item Verify that the marginal distributions coincide with the original distributions: $\widehat{T}_{p_{\text{d-sep}_1}} \stackrel{d}{=} T_{p_{\text{d-sep}_1}}$ and $\widehat{T}_{p_{\text{d-sep}_2}} \stackrel{d}{=} T_{p_{\text{d-sep}_2}}$.
    \item Show the coupling is monotone: $\mathbb{P}[\widehat{T}_{p_{\text{d-sep}_1}} \geq \widehat{T}_{p_{\text{d-sep}_2}}] = 1$.
    \item Conclude from Strassen's theorem that $\mathbb{E}[T_{p_{\text{d-sep}_1}}] \geq \mathbb{E}[T_{p_{\text{d-sep}_2}}]$.
\end{enumerate}

\subsubsection{Step 1: Describing the Process for Generating the Monotone Coupling $(\widehat{T}_{p_{\text{d-sep}_1}}, \widehat{T}_{p_{\text{d-sep}_2}})$}

We construct a coupling between two random variables $\widehat{T}_{p_{\text{d-sep}_1}}$ and $\widehat{T}_{p_{\text{d-sep}_2}}$ by describing a hypothetical generative process that uses shared randomness. Both random variables take values in positive integers representing the number of tests executed before EP terminates, and will be designed to have the same marginal distributions as $T_{p_{\text{d-sep}_1}}$ and $T_{p_{\text{d-sep}_2}}$ respectively.

\textbf{Sources of shared randomness.} We fix the expert accuracies $p_{\text{d-sep}_1}$ and $p_{\text{d-sep}_2}$, the conditioning set size $k$, and allow the expert predictions $\widehat{\mathcal{G}}$ to vary. For edge $n_{i,j}$ with adjacency set $A = \adj_{-j}(\calC, x_i)$, EP-Guess considers all size-$k$ subsets of $A$, i.e., all $W \in [A]_k$. Let $c = |[A]_k|$ denote the total number of such conditioning sets. Our coupling uses three sources of shared randomness:
\begin{enumerate}
    \item We partition the size-$k$ conditioning sets in $[A]_k$ based on true d-separation in $\mathcal{G}^*$. Suppose there are $m$ sets in $[A]_k$ that d-separate $x_i, x_j$ in $\mathcal{G}^*$ (d-separating sets) and $c - m$ sets that do not (non-d-separating sets). Let $L = [l_1, l_2, \ldots, l_c]$ be a random variable corresponding to a uniformly sampled permutation of $m$ ones and $c - m$ zeros. That is, $L$ contains exactly $m$ entries equal to 1 (indicating d-separating sets) and $c - m$ entries equal to 0 (indicating non-d-separating sets), where the ordering is uniformly random among all such permutations.
    \item A collection of independent uniform random variables $\mathcal{R} = \{r_1, r_2, \ldots, r_c\}$ where each $r_i \sim \text{Uniform}[0,1]$. These will determine, for each position in $L$ independently, whether the expert correctly classifies the corresponding conditioning set.
    \item Two uniform random permutations: $\pi_{\text{d-sep}}$ over the $m$ d-separating sets in $[A]_k$, and $\pi_{\text{non-d-sep}}$ over the $c - m$ non-d-separating sets in $[A]_k$. These determine the relative ordering within each partition.
\end{enumerate}

\textbf{Generating the expert prediction at accuracy $p_{\text{d-sep}_1}$.} We now describe how to sample an expert graph $\widehat{\mathcal{G}}^{(p_{\text{d-sep}_1})}$ with d-separation accuracy $p_{\text{d-sep}_1}$ using the shared randomness $(L, \mathcal{R}, \pi_{\text{d-sep}}, \pi_{\text{non-d-sep}})$. For each position $i$ in $L$, the expert independently classifies the corresponding conditioning set correctly with probability $p_{\text{d-sep}_1}$:
\begin{itemize}
    \item If $l_i = 1$ (corresponds to a true d-separating set in $\mathcal{G}^*$): the expert correctly predicts this set d-separates $x_i, x_j$ if $r_i \leq p_{\text{d-sep}_1}$, otherwise incorrectly predicts it does not d-separate.
    \item If $l_i = 0$ (corresponds to a non-d-separating set): the expert correctly predicts this set does not d-separate $x_i, x_j$ if $r_i \leq p_{\text{d-sep}_1}$, otherwise incorrectly predicts it d-separates.
\end{itemize}

\textbf{Constructing the conditioning set ordering $\widehat{\mathsf{L}}^{p_{\text{d-sep}_1}}$.} Given the expert prediction $\widehat{\mathcal{G}}^{(p_{\text{d-sep}_1})}$, we construct the conditioning set ordering as specified by EP-Guess (Subroutine~\ref{algo: expert-subset-ordering1}). Initialize two empty lists $B_1 = [], B_2 = []$. For each position $i$ in $L$ (going in the order specified by $L$):
\begin{itemize}
    \item If $l_i = 1$ (d-separating set) and $r_i \leq p_{\text{d-sep}_1}$ (correctly classified): add position $i$ to the end of $B_1$ (predicted d-separating sets tested first).
    \item If $l_i = 1$ (d-separating set) and $r_i > p_{\text{d-sep}_1}$ (incorrectly classified): add position $i$ to the end of $B_2$.
    \item If $l_i = 0$ (non-d-separating set) and $r_i \leq p_{\text{d-sep}_1}$ (correctly classified): add position $i$ to the end of $B_2$.
    \item If $l_i = 0$ (non-d-separating set) and $r_i > p_{\text{d-sep}_1}$ (incorrectly classified): add position $i$ to the end of $B_1$.
\end{itemize}
Concatenate the buckets: $L_F^{p_{\text{d-sep}_1}} = B_1 + B_2$. Within $L_F^{p_{\text{d-sep}_1}}$, assign relative ordering among d-separating sets using $\pi_{\text{d-sep}}$ and among non-d-separating sets using $\pi_{\text{non-d-sep}}$ to obtain the final conditioning set ordering $\widehat{\mathsf{L}}^{p_{\text{d-sep}_1}}$.

\textbf{Computing $\widehat{T}_{p_{\text{d-sep}_1}}$.} Given the conditioning set ordering $\widehat{\mathsf{L}}^{p_{\text{d-sep}_1}}$, let $\widehat{T}_{p_{\text{d-sep}_1}}$ denote the expected number of tests executed by EP (Subroutine~\ref{algo: subset-ordering}) before termination if we were to randomly draw $n$ finite samples from the DGP $\mathcal{G}^*$ and run tests according to $\widehat{\mathsf{L}}^{p_{\text{d-sep}_1}}$.

\textbf{Generating $\widehat{T}_{p_{\text{d-sep}_2}}$ using the same randomness.} We follow the exact same procedure as above, crucially reusing the same shared randomness $(L, \mathcal{R}, \pi_{\text{d-sep}}, \pi_{\text{non-d-sep}})$. The only difference is that we use accuracy $p_{\text{d-sep}_2}$ instead of $p_{\text{d-sep}_1}$ when determining expert classifications. This yields a potentially different expert graph $\widehat{\mathcal{G}}^{(p_{\text{d-sep}_2})}$, a potentially different conditioning set ordering $\widehat{\mathsf{L}}^{p_{\text{d-sep}_2}}$, and a potentially different expected runtime $\widehat{T}_{p_{\text{d-sep}_2}}$.

By this coupling procedure, we generate the joint random variable $(\widehat{T}_{p_{\text{d-sep}_1}}, \widehat{T}_{p_{\text{d-sep}_2}})$.

\subsubsection{Step 2: Showing the Marginals of the Two Variables Coincide with Original Distributions}

Our goal is to verify that $\widehat{T}_{p_{\text{d-sep}_1}} \stackrel{d}{=} T_{p_{\text{d-sep}_1}}$ and $\widehat{T}_{p_{\text{d-sep}_2}} \stackrel{d}{=} T_{p_{\text{d-sep}_2}}$. We focus on showing $\widehat{T}_{p_{\text{d-sep}_1}} \stackrel{d}{=} T_{p_{\text{d-sep}_1}}$; the argument for $p_{\text{d-sep}_2}$ follows identically by symmetry.

\textbf{Reduction to orderings.} Let $\mathsf{L}^{p_{\text{d-sep}_1}}$ denote the random conditioning set ordering generated by Subroutine~\ref{algo: expert-subset-ordering1} using expert accuracy $p_{\text{d-sep}_1}$, and let $\widehat{\mathsf{L}}^{p_{\text{d-sep}_1}}$ denote the random conditioning set ordering generated in our coupling procedure (Step 1) using accuracy $p_{\text{d-sep}_1}$. Given any fixed conditioning set ordering, the expected number of tests executed when randomly drawing $n$ finite samples from $\mathcal{G}^*$ is deterministically fixed. Therefore, the distributions of $T_{p_{\text{d-sep}_1}}$ and $\widehat{T}_{p_{\text{d-sep}_1}}$ are determined entirely by the distributions of $\mathsf{L}^{p_{\text{d-sep}_1}}$ and $\widehat{\mathsf{L}}^{p_{\text{d-sep}_1}}$ respectively. To show $\widehat{T}_{p_{\text{d-sep}_1}} \stackrel{d}{=} T_{p_{\text{d-sep}_1}}$, it suffices to show that $\widehat{\mathsf{L}}^{p_{\text{d-sep}_1}} \stackrel{d}{=} \mathsf{L}^{p_{\text{d-sep}_1}}$.

\textbf{Decomposition of orderings.} Any conditioning set ordering $\mathsf{L}$ over all $c$ conditioning sets can be uniquely decomposed into three components:
\begin{enumerate}
    \item $\pi_{\text{d-sep}}$: the relative ordering (permutation) among the $m$ d-separating sets
    \item $\pi_{\text{non-d-sep}}$: the relative ordering (permutation) among the $c-m$ non-d-separating sets
    \item $\pi_{\text{rel}}$: the relative placement of d-separating sets versus non-d-separating sets (which sets come before which other sets)
\end{enumerate}
Therefore, the distribution over conditioning set orderings is uniquely determined by the joint distribution over $(\pi_{\text{d-sep}}, \pi_{\text{non-d-sep}}, \pi_{\text{rel}})$. We will show that this joint distribution is identical for both $\mathsf{L}^{p_{\text{d-sep}_1}}$ and $\widehat{\mathsf{L}}^{p_{\text{d-sep}_1}}$.

\textbf{Analysis of the coupling procedure.} In our coupling construction (Step 1), we explicitly sampled $\pi_{\text{d-sep}}$ and $\pi_{\text{non-d-sep}}$ uniformly and independently. The relative placement $\pi_{\text{rel}}$ is then determined by: for each position $i$ in $L$, whether $l_i$ is correctly classified (with probability $p_{\text{d-sep}_1}$) determines whether that conditioning set goes to $B_1$ or $B_2$, and then $\pi_{\text{rel}}$ corresponds to the partition structure $(B_1, B_2)$. Since each conditioning set's classification is independent and depends only on its true d-separation status and $p_{\text{d-sep}_1}$, we have:
\begin{itemize}
    \item $\widehat{\pi}_{\text{d-sep}}$ is uniform over all permutations of $m$ d-separating sets
    \item $\widehat{\pi}_{\text{non-d-sep}}$ is uniform over all permutations of $c-m$ non-d-separating sets
    \item $\widehat{\pi}_{\text{rel}}$ has distribution determined by $p_{\text{d-sep}_1}$ (probability of correct classification)
    \item $\widehat{\pi}_{\text{d-sep}}, \widehat{\pi}_{\text{non-d-sep}}, \widehat{\pi}_{\text{rel}}$ are mutually independent
\end{itemize}

\textbf{Analysis of Subroutine~\ref{algo: expert-subset-ordering1}.} We now show that $\mathsf{L}^{p_{\text{d-sep}_1}}$ has the same distributional structure. By construction of Subroutine~\ref{algo: expert-subset-ordering1}:
\begin{itemize}
    \item Conditioning sets in $[A]_k$ are initially shuffled uniformly at random
    \item Each conditioning set is independently classified based on whether it d-separates in $\widehat{\mathcal{G}}$ with probability $p_{\text{d-sep}_1}$ of correct classification
    \item Sets are partitioned into $B_1$ (predicted d-separating) and $B_2$ (predicted non-d-separating), preserving their random ordering within each bucket
\end{itemize}

To show $\pi_{\text{d-sep}}$ is uniform: Start with an initial uniform random permutation $\pi_{\text{d-sep}}^{\text{init}}$ of all $m$ d-separating sets. Some subset of size $\ell \sim \text{Binomial}(m, 1-p_{\text{d-sep}_1})$ are misclassified and placed in $B_2$, while the remaining $m-\ell$ are placed in $B_1$. For any fixed $\ell$ and any fixed choice of which $\ell$ sets are misclassified, this operation defines a bijection from $\pi_{\text{d-sep}}^{\text{init}}$ to the resulting permutation: given the final permutation and knowing which sets went to which bucket, we can uniquely recover $\pi_{\text{d-sep}}^{\text{init}}$, and vice versa. Since bijections preserve uniformity, the marginal distribution of $\pi_{\text{d-sep}}$ (after marginalizing over $\ell$ and the choice of which sets) remains uniform. By the same argument, $\pi_{\text{non-d-sep}}$ is uniform.

For independence: The distribution of $\pi_{\text{rel}}$ is determined solely by which conditioning sets are correctly classified (controlled by $p_{\text{d-sep}_1}$). For any fixed realization of $\pi_{\text{rel}}$ (i.e., fixed partition $(B_1, B_2)$), the above uniformity argument shows that $\pi_{\text{d-sep}}$ and $\pi_{\text{non-d-sep}}$ remain uniform. Therefore $(\pi_{\text{d-sep}}, \pi_{\text{non-d-sep}})$ are independent of $\pi_{\text{rel}}$.

\textbf{Conclusion.} Both $\widehat{\mathsf{L}}^{p_{\text{d-sep}_1}}$ and $\mathsf{L}^{p_{\text{d-sep}_1}}$ decompose into $(\pi_{\text{d-sep}}, \pi_{\text{non-d-sep}}, \pi_{\text{rel}})$ where each component has identical marginal distributions and the same independence structure. Therefore $\widehat{\mathsf{L}}^{p_{\text{d-sep}_1}} \stackrel{d}{=} \mathsf{L}^{p_{\text{d-sep}_1}}$, which implies $\widehat{T}_{p_{\text{d-sep}_1}} \stackrel{d}{=} T_{p_{\text{d-sep}_1}}$. By the same argument, $\widehat{T}_{p_{\text{d-sep}_2}} \stackrel{d}{=} T_{p_{\text{d-sep}_2}}$.

\subsubsection{Step 3: Showing that the Coupling is Monotone}

Our goal is to show that $\mathbb{P}[\widehat{T}_{p_{\text{d-sep}_1}} \geq \widehat{T}_{p_{\text{d-sep}_2}}] = 1$. We establish this by showing that the conditioning set ordering $\widehat{\mathsf{L}}^{p_{\text{d-sep}_2}}$ can be obtained from $\widehat{\mathsf{L}}^{p_{\text{d-sep}_1}}$ through a sequence of runtime-reducing swaps.

\textbf{Structure of the two orderings.} Recall from Step 1 that both $\widehat{\mathsf{L}}^{p_{\text{d-sep}_1}}$ and $\widehat{\mathsf{L}}^{p_{\text{d-sep}_2}}$ are generated using the same shared randomness $(L, \mathcal{R}, \pi_{\text{d-sep}}, \pi_{\text{non-d-sep}})$. Crucially, the permutations $\pi_{\text{d-sep}}$ (ordering among d-separating sets) and $\pi_{\text{non-d-sep}}$ (ordering among non-d-separating sets) are identical in both orderings. The only difference lies in the relative placement $\pi_{\text{rel}}$ of d-separating sets versus non-d-separating sets, which is determined by which conditioning sets are correctly classified.

\textbf{More sets correctly classified at higher accuracy.} Since $p_{\text{d-sep}_1} < p_{\text{d-sep}_2}$, for each position $i$ in $L$:
\begin{itemize}
    \item If conditioning set $i$ is correctly classified under accuracy $p_{\text{d-sep}_1}$ (i.e., $r_i \leq p_{\text{d-sep}_1}$), then it is also correctly classified under accuracy $p_{\text{d-sep}_2}$ (since $r_i \leq p_{\text{d-sep}_1} < p_{\text{d-sep}_2}$)
    \item Additional sets may be correctly classified under $p_{\text{d-sep}_2}$ that were misclassified under $p_{\text{d-sep}_1}$ (those with $p_{\text{d-sep}_1} < r_i \leq p_{\text{d-sep}_2}$)
\end{itemize}

This means:
\begin{itemize}
    \item D-separating sets that were correctly placed in $B_1$ in $L_F^{p_{\text{d-sep}_1}}$ remain in $B_1$ in $L_F^{p_{\text{d-sep}_2}}$
    \item Some d-separating sets that were incorrectly placed in $B_2$ in $L_F^{p_{\text{d-sep}_1}}$ may move to $B_1$ in $L_F^{p_{\text{d-sep}_2}}$ (moving leftward)
    \item Non-d-separating sets that were correctly placed in $B_2$ in $L_F^{p_{\text{d-sep}_1}}$ remain in $B_2$ in $L_F^{p_{\text{d-sep}_2}}$
    \item Some non-d-separating sets that were incorrectly placed in $B_1$ in $L_F^{p_{\text{d-sep}_1}}$ may move to $B_2$ in $L_F^{p_{\text{d-sep}_2}}$ (moving rightward)
\end{itemize}

\textbf{Characterizing the transformation via inversions.} To formalize how $\widehat{\mathsf{L}}^{p_{\text{d-sep}_2}}$ relates to $\widehat{\mathsf{L}}^{p_{\text{d-sep}_1}}$, we introduce a numerical ordering. Assign integers to conditioning sets as follows:
\begin{itemize}
    \item Assign $\{1, 2, \ldots, m\}$ to the d-separating sets according to their fixed relative order $\pi_{\text{d-sep}}$
    \item Assign $\{m+1, \ldots, c\}$ to the non-d-separating sets according to their fixed relative order $\pi_{\text{non-d-sep}}$
\end{itemize}

Under this assignment, every d-separating set has a smaller numerical label than every non-d-separating set. An \textbf{inversion} is any pair of conditioning sets that appears out of numerical order—specifically, a non-d-separating set appearing before a d-separating set in the ordering. Since $\pi_{\text{d-sep}}$ and $\pi_{\text{non-d-sep}}$ are fixed, sets of the same type (both d-separating or both non-d-separating) never form inversions relative to each other.

The transformation from $\widehat{\mathsf{L}}^{p_{\text{d-sep}_1}}$ to $\widehat{\mathsf{L}}^{p_{\text{d-sep}_2}}$ only moves d-separating sets leftward and non-d-separating sets rightward. This process cannot create new inversions: if a non-d-separating set precedes a d-separating set in $\widehat{\mathsf{L}}^{p_{\text{d-sep}_2}}$, that pair must have been in the same order in $\widehat{\mathsf{L}}^{p_{\text{d-sep}_1}}$. Therefore, the set of inversions in $\widehat{\mathsf{L}}^{p_{\text{d-sep}_2}}$ is a subset of the inversions in $\widehat{\mathsf{L}}^{p_{\text{d-sep}_1}}$.

\textbf{Weak Bruhat order and reachability.} The weak Bruhat order on permutations provides a formal framework for this relationship. A standard result (\citep{yessenov2005descents}, Proposition 2.2) states:

\begin{proposition}
For permutations $v$ and $w$, $v \leq w$ in weak Bruhat order if and only if $\mathrm{Inv}(v) \subseteq \mathrm{Inv}(w)$, where $\mathrm{Inv}(\cdot)$ denotes the set of inversions.
\end{proposition}

Equivalently, $v$ can be obtained from $w$ by a sequence of adjacent transpositions $(w_i, w_{i+1})$ where $w_i > w_{i+1}$ in the numerical ordering. In our setting, moving a d-separating set leftward past an adjacent non-d-separating set (moving $w_i$ left of $w_{i+1}$) is precisely such an inversion-reducing swap, since d-separating sets have smaller numerical labels than non-d-separating sets.

Since $\mathrm{Inv}(\widehat{\mathsf{L}}^{p_{\text{d-sep}_2}}) \subseteq \mathrm{Inv}(\widehat{\mathsf{L}}^{p_{\text{d-sep}_1}})$, it follows that $\widehat{\mathsf{L}}^{p_{\text{d-sep}_2}}$ can be obtained from $\widehat{\mathsf{L}}^{p_{\text{d-sep}_1}}$ through a sequence of adjacent swaps that move d-separating sets leftward past non-d-separating sets.

\textbf{Monotonicity via runtime-reducing swaps.} By Lemma~\ref{lemma: optimal-sequence}, each such swap (placing a d-separating set before a non-d-separating set) strictly decreases the expected number of tests executed. Since $\widehat{\mathsf{L}}^{p_{\text{d-sep}_2}}$ is reachable from $\widehat{\mathsf{L}}^{p_{\text{d-sep}_1}}$ through a sequence of such runtime-reducing swaps, we have $\widehat{T}_{p_{\text{d-sep}_1}} \geq \widehat{T}_{p_{\text{d-sep}_2}}$, establishing that $\mathbb{P}[\widehat{T}_{p_{\text{d-sep}_1}} \geq \widehat{T}_{p_{\text{d-sep}_2}}] = 1$.

\subsubsection{Step 4: Conclusion}

It follows from Strassen's Coupling Theorem that $\mathbb{E}[T_{p_{\text{d-sep}_1}}] \geq \mathbb{E}[T_{p_{\text{d-sep}_2}}]$, with strict inequality when $\text{CIT}^k_{\adj_{-j}(\calC, x_i)}$ contains at least one d-separating set and one non-d-separating set. When such mixed sets exist, there is at least one pair of orderings $\widehat{\mathsf{L}}^{p_{\text{d-sep}_1}}, \widehat{\mathsf{L}}^{p_{\text{d-sep}_2}}$ that can be created through the process outlined in Step 1 such that a d-separating set is swapped to be earlier in the sequence. In that case we have $\widehat{T}_{p_{\text{d-sep}_1}} > \widehat{T}_{p_{\text{d-sep}_2}}$, which implies by the Coupling Theorem that $\mathbb{E}[T_{p_{\text{d-sep}_1}}] > \mathbb{E}[T_{p_{\text{d-sep}_2}}]$.

For part (b), observe that when $p_{\text{d-sep}} = 0.5$, the expert's predictions are independent of the true d-separation structure, producing orderings distributionally equivalent to random orderings. Therefore $\mathbb{E}[T_{\text{EP-G}}(0.5)] = \mathbb{E}[T_{\text{random}}]$. Part (a) then implies $\mathbb{E}[T_{\text{EP-G}}(p_{\text{d-sep}})] \leq \mathbb{E}[T_{\text{random}}]$ for all $p_{\text{d-sep}} \geq 0.5$.

\end{proof}

\subsection{Proof of Theorem \ref{theorem: guess acc correct}}\label{proof: guess acc correct}
% add label above
\GuessPCCorrectness*
\begin{proof}
We will first show that, under the assumption of oracle CITs (CITs that always return independence when conditioning on a d-separating set and dependence if not), both PC-Guess and gPC-Guess return the correct graph. We will then note that the probability that either method returns the incorrect graph is upper bounded by the probability that at least one CIT run returns an incorrect result. We will conclude by noting that, under a consistent CIT, the probability that any CIT test returns independence goes to $0$, yielding $\lim_{n \to \infty} \mathbb{P}(\widetilde{\mathcal{G}} = \mathcal{G}^*) = 1$.

We note that, with access to oracle CIT, PC has previously been shown to always return the correct graph when using any ordering of the vertices any ordering $\mathsf{O}$ to guide EL, and by Lemma \citep{spirtes_anytime_2001}, and suborderings $\mathsf{L}$ given to EP do not affect whether EP will remove an edge or retain it (Lemma \ref{lemma: orderingindaccuracy}). Therefore, PC-Guess returns the true graph with correct (in)dependence results from tests.

We note that gPC-Guess performs a single pass of the EL, running EP with CITs conditioning on subsets of size $0$ to $|V|-1$. Then, the correctness of gPC-Guess with access to oracle CITs follows directly from Lemmas \ref{lemma: EL perfect}, \ref{lemma: EP perfect}).

Now we establish the asymptotic result with consistent CITs. Let $K$ denote the maximum number of CIT tests that could possibly be run by either PC-Guess or gPC-Guess. Since the number of vertices $|V|$ is fixed, $K$ is finite---specifically, $K \leq |V|^2 \cdot 2^{|V|-2}$, as each test involves choosing two variables and a conditioning set from the remaining vertices.

For each possible test $\tau$ involving variables $X, Y$ and conditioning set $S$, let $E_\tau^{(n)}$ denote the event that test $\tau$ returns an incorrect result when run on sample size $n$. By the oracle correctness established above, the algorithm returns the incorrect graph only if at least one test returns an incorrect result. Therefore:
\begin{align*}
\mathbb{P}(\tilde{\mathcal{G}} \neq \mathcal{G}^*) \leq \mathbb{P}\left(\bigcup_{\tau \in \text{tests run}} E_\tau^{(n)}\right)
\end{align*}

Since the set of tests actually run is a subset of all $K$ possible tests, we have:
\begin{align*}
\mathbb{P}(\tilde{\mathcal{G}} \neq \mathcal{G}^*) \leq \mathbb{P}\left(\bigcup_{\tau=1}^{K} E_\tau^{(n)}\right) \leq \sum_{\tau=1}^{K} \mathbb{P}(E_\tau^{(n)})
\end{align*}
where the final inequality follows from the union bound.

By the consistency assumption, for each test $\tau$, we have $\lim_{n \to \infty} \mathbb{P}(E_\tau^{(n)}) = 0$. Since $K$ is finite:
\begin{align*}
\lim_{n \to \infty} \mathbb{P}(\tilde{\mathcal{G}} \neq \mathcal{G}^*) \leq \lim_{n \to \infty} \sum_{\tau=1}^{K} \mathbb{P}(E_\tau^{(n)}) = \sum_{\tau=1}^{K} \lim_{n \to \infty} \mathbb{P}(E_\tau^{(n)}) = 0
\end{align*}

Therefore, $\lim_{n \to \infty} \mathbb{P}(\tilde{\mathcal{G}} = \mathcal{G}^*) = 1$. \

% given any single ordering $\mathsf{B}$ to guide both EL and EP $$

% the probability of any test 

% The correctness follows for PC from that fact that correctness results for PC have been shown to hold when the EL subroutine algorithm is given any possible ordering \citep{spirtes_anytime_2001}, and by the fact that suborderings given to EP do not affect whether EP will remove an edge or retain it (Lemma \ref{lemma: orderingindaccuracy}).

\end{proof}

% \subsection{Proof of Corollary \ref{theorem: guess acc worst}}
% \GuessPCWorstCaseAccuracy*
% \begin{proof}
% Follows directly from proof of Theorem \ref{theorem: guess acc correct}, and because PC-GUESS inherits the guarantees of PC.
% \end{proof}

\subsection{Proof of Theorem \ref{theorem: pc performance}}\label{proof: pc performance}
% add label above
\PCGuessIterationPerformance*
\begin{proof}
We prove each part separately.

\textbf{Part (a):} The monotonic increase in $\mathbb{E}[\Phi_\ell]$ with $p_\psi$ follows directly from Lemma \ref{lemma: edge-accuracy-monotonic}. For any iteration $\ell$ with partial skeleton $\mathcal{C}$, EL-G (Subroutine \ref{algo: expert-edge-ordering1}) partitions edges based on expert predictions and processes false edges before true edges. By Lemma \ref{lemma: edge-accuracy-monotonic}, as expert accuracy $p_\psi$ increases, the expected perfect recovery probability $\mathbb{E}[\Phi_\ell]$ increases monotonically. By Lemma \ref{lemma: edge-accuracy-monotonic} this inequality is strict when $\mathcal{C}$ contains both true edges (edges in $\mathcal{S}^*$) and false edges (edges not in $\mathcal{S}^*$) and $\mathcal{G}$ is a nonempty and non-fully connected graph.

\textbf{Part (b):} Under the assumption of Definition \ref{assumption:cit_power_independence}, the monotonic decrease in $\mathbb{E}[t_\ell]$ with $p_{\text{d-sep}}$ follows directly from Lemma \ref{lemma: expert-accuracy}. For fixed expert edge accuracy $p_\psi$, as d-separation prediction accuracy $p_{\text{d-sep}}$ increases, when testing any edge using its adjacency set, EP-G (Subroutine \ref{algo: expert-subset-ordering1}) places d-separating sets earlier in the test sequence, reducing the expected number of tests $\mathbb{E}[t_\ell]$ conducted before finding a d-separating set or exhausting all tests. Note that the distribution over how likely an edge will be tested with a particular adjacency set remains fixed due to $p_\psi$ remaining fixed. The strict inequality occurs when at least one edge in $\mathcal{C}$ is tested with a sequence of CITs where at least one CIT conditions on a subset which d-separates that edge, and one other CIT does not condition on a subset which d-separates that edge. This occurs exactly when $\mathcal{C}$ contains at least one false edge where at least one vertex in the edge has an adjacency set that contains at least one d-separating subset and one non-d-separating subset.

\textbf{Part (c):} When $p_\psi = 0.5$, the expert classifies each edge as true or false with equal probability. The procedure in EL-G (Subroutine \ref{algo: expert-edge-ordering1}) then partitions the edges into two sets and orders them as $\mathsf{O} = \mathcal{C} \setminus \widehat{\mathcal{S}} + \mathcal{C} \cap \widehat{\mathcal{S}}$, with random ordering within each partition. Since each edge is assigned to each partition with probability $0.5$ independently, and edges within each partition are randomly ordered, this is equivalent to sampling a uniformly random ordering of all edges in $\mathcal{C}$. Therefore, when $p_\psi = 0.5$, PC-Guess has the same distribution over orderings as the baseline PC (which uses uniformly random orderings), implying $\mathbb{E}[\Phi_\ell] = \mathbb{E}[\bar{\Phi}_\ell]$ at $p_\psi = 0.5$. Combined with part (a), which establishes that $\mathbb{E}[\Phi_\ell]$ increases monotonically with $p_\psi$, we have $\mathbb{E}[\Phi_\ell] \geq \mathbb{E}[\bar{\Phi}_\ell]$ for all $p_\psi \geq 0.5$.
\end{proof}

\subsection{Proof of Theorem \ref{theorem: sgs performance}}\label{proof: sgs performance}
% \begin{restatable}[Performance of gPC-Guess]{theorem}{gPCGuessPerformance}\label{theorem: sgs performance}
% gPC-Guess satisfies \ref{criterion:monotone}-\ref{criterion:finite-sample}: \textbf{(a)} $\mathbb{E}[\Phi]$ increases monotonically with $p_\psi$; \textbf{(b)} For fixed $p_\psi$, $\mathbb{E}[t]$ decreases monotonically with $p_{\text{d-sep}}$; \textbf{(c)} When $p_\psi \geq 0.5$, $\mathbb{E}[\Phi] \geq \mathbb{E}[\bar{\Phi}]$.
% \end{restatable}
\gPCGuessPerformance*

\begin{proof}
We prove each part separately.

\textbf{Part (a):} The monotonic increase in $\mathbb{E}[\Phi]$ with $p_\psi$ follows directly from Lemma \ref{lemma: edge-accuracy-monotonic}. gPC-Guess runs EL once on the complete skeleton $\mathcal{C}$, where EL-G (Subroutine \ref{algo: expert-edge-ordering1}) partitions edges based on expert predictions and processes false edges before true edges. By Lemma \ref{lemma: edge-accuracy-monotonic}, as expert accuracy $p_\psi$ increases, the expected perfect recovery probability $\mathbb{E}[\Phi]$ increases monotonically. By Lemma \ref{lemma: edge-accuracy-monotonic} this inequality is strict when $\mathcal{C}$ contains both true edges (edges in $\mathcal{S}^*$) and false edges (edges not in $\mathcal{S}^*$). Since gPC-Guess starts with the complete skeleton, $\mathcal{C}$ contains both true and false edges whenever $\mathcal{G}$ is a nonempty and non-fully connected graph.

\textbf{Part (b):} Under the assumption of Definition \ref{assumption:cit_power_independence}, the monotonic decrease in $\mathbb{E}[t]$ with $p_{\text{d-sep}}$ follows directly from Lemma \ref{lemma: expert-accuracy}. For fixed expert edge accuracy $p_\psi$, as d-separation prediction accuracy $p_{\text{d-sep}}$ increases, when testing any edge using its adjacency set, EP-G (Subroutine \ref{algo: expert-subset-ordering1}) places d-separating sets earlier in the test sequence, reducing the expected number of tests $\mathbb{E}[t]$ conducted before finding a d-separating set or exhausting all tests. Note that the distribution over how likely an edge will be tested with a particular adjacency set remains fixed due to $p_\psi$ remaining fixed. The strict inequality occurs when at least one edge in $\mathcal{C}$ is tested with a sequence of CITs where at least one CIT conditions on a subset which d-separates that edge, and one other CIT does not condition on a subset which d-separates that edge. This occurs exactly when $G^*$ is not entirely empty (all subsets are d-separating), or entirely connected (no subsets are d-separating).

\textbf{Part (c):} When $p_\psi = 0.5$, the expert classifies each edge as true or false with equal probability. The procedure in EL-G (Subroutine \ref{algo: expert-edge-ordering1}) then partitions the edges into two sets and orders them as $\mathsf{O} = \mathcal{C} \setminus \widehat{\mathcal{S}} + \mathcal{C} \cap \widehat{\mathcal{S}}$, with random ordering within each partition. Since each edge is assigned to each partition with probability $0.5$ independently, and edges within each partition are randomly ordered, this is equivalent to sampling a uniformly random ordering of all edges in $\mathcal{C}$. Therefore, when $p_\psi = 0.5$, gPC-Guess has the same distribution over orderings as the baseline gPC (which uses uniformly random orderings), implying $\mathbb{E}[\Phi] = \mathbb{E}[\bar{\Phi}]$ at $p_\psi = 0.5$. Combined with part (a), which establishes that $\mathbb{E}[\Phi]$ increases monotonically with $p_\psi$, we have $\mathbb{E}[\Phi] \geq \mathbb{E}[\bar{\Phi}]$ for all $p_\psi \geq 0.5$.
\end{proof}

% \subsection{Proof of Lemma \ref{lemma: probedgeaccuracy}}\label{proof: probedgeaccuracy}
% % add label above
% \ProbEdgeAccuracy*
% \begin{proof}
% \sujai{proof writeup in progress.}
% \end{proof}

\newpage
\section{Theoretical Details Concerning EP-G and its Guarantees}\label{appendix: subset ordering details}

In this appendix, we provide the complete theoretical analysis for guiding the Edge Prune (EP) subroutine with expert predictions. While Section~\ref{subsec: subset-ordering} established that EP orderings cannot affect accuracy (only runtime), we formalize here the full analysis: when and how different orderings impact computational efficiency, what ordering principles are optimal, and how expert accuracy translates to monotonic runtime improvements.

\subsection{The Edge Prune Subroutine and Ordering Choices}

Recall that the EP subroutine (Subroutine~\ref{algo: subset-ordering}) is called by constraint-based algorithms to test individual edges $e_{i,j}$ at a fixed conditioning set size $k$. Given the current skeleton $\mathcal{C}$, EP computes the adjacency set $A = \adj_{-j}(\mathcal{C}, x_i)$ and tests conditional independence for all size-$k$ subsets $W \subseteq A$. The subroutine terminates as soon as any test CIT($x_i, x_j \mid W$) returns independence, at which point edge $n_{i,j}$ is removed from the skeleton.

Let $\text{CIT}^k_A := \{\text{CIT}(x_i, x_j \mid W): W \subseteq A, |W|=k\}$ denote the collection of all possible size-$k$ conditional independence tests for this edge. The EP subroutine requires an ordering $\mathsf{L}$ that specifies the sequence in which these tests are executed. Our goal is to understand how the choice of $\mathsf{L}$ affects algorithm performance.

\subsection{Accuracy is Invariant to Ordering}

We begin by establishing that EP orderings cannot affect the probability of correctly deciding edge $n_{i,j}$:

\orderingindaccuracy*

\textbf{Proof sketch.} Edge $n_{i,j}$ is removed if and only if at least one test in $\text{CIT}^k_A$ returns independence. Since each test's outcome depends only on the data, the variables tested, and the conditioning set—not the execution order—the probability of correct removal or retainment of the edge depends only on which tests are run, not their sequence. See Appendix~\ref{proof: orderindaccuracy} for full proof.

This result implies that unlike EL (where edge ordering affects accuracy via adjacency set modifications), EP orderings leave the correctness probability unchanged. Our focus therefore shifts to computational efficiency.

\subsection{When Orderings Affect Runtime}

While accuracy is invariant, the computational cost—measured by the number of tests executed before EP terminates—varies across orderings. To understand when and why, we partition $\text{CIT}^k_A$ based on the true d-separation structure of $\mathcal{G}^*$:

\begin{itemize}
    \item $\text{CIT}^{\text{d-sep}}_A := \{\text{CIT}(x_i, x_j \mid W): W \text{ d-separates } x_i, x_j \text{ in } \mathcal{G}^*\}$
    \item $\text{CIT}^{\text{non-d-sep}}_A := \text{CIT}^k_A \setminus \text{CIT}^{\text{d-sep}}_A$
\end{itemize}

Tests in $\text{CIT}^{\text{d-sep}}_A$ have high probability of returning independence (specifically $1-\alpha$, where $\alpha$ is the Type I error rate), while tests in $\text{CIT}^{\text{non-d-sep}}_A$ have low probability of returning independence (specifically $\beta$, the Type II error rate or power deficit).

\nodseps*

\textbf{Proof sketch.} When no d-separating sets exist, all tests return independence with identical probability $\beta$ (the false negative rate). Under mutual independence of test outcomes, the expected stopping time follows a geometric distribution with parameter $\beta$, which is invariant to the ordering of tests. See Appendix~\ref{proof: no-dseps} for full proof.

This lemma reveals the key insight: runtime optimization is only possible when $\text{CIT}^k_A$ contains both d-separating and non-d-separating sets. In such cases, strategic ordering can significantly reduce computational cost.

\subsection{Optimal Ordering Principles}

We now characterize orderings that minimize expected runtime. Our analysis relies on standard technical assumptions from the conditional independence testing literature:

\begin{assumption}[CIT Specificity and Independence]\label{assumption:cit_power_independence}
We assume:
\begin{enumerate}
    \item[(i)] \textbf{Adequate Specificity:} $1-\alpha > \beta$, i.e., the true negative rate (probability of correctly detecting independence when it exists) exceeds the false negative rate (probability of failing to detect dependence when it exists). This holds asymptotically as sample size $n \to \infty$ under standard regularity conditions.
    \item[(ii)] \textbf{Conditional independence of tests:} For any two distinct conditioning sets $W_1, W_2 \subseteq A$, the outcomes of CIT($x_i, x_j \mid W_1$) and CIT($x_i, x_j \mid W_2$) are (conditionally) independent given the data. This holds asymptotically under faithfulness and sufficient sample size.
\end{enumerate}
\end{assumption}

\textbf{Remark on assumptions.} Assumption~\ref{assumption:cit_power_independence}(ii) is a technical simplification that ensures tractability. In practice, test outcomes are not strictly independent due to shared data and overlapping conditioning sets. However, our main result—that d-separating sets should be placed first—likely holds under weaker conditions. Specifically, we conjecture that the monotonicity guarantee (Lemma~\ref{lemma: expert-accuracy}) extends to any setting where: (a) tests on d-separating sets have strictly higher independence probability than tests on non-d-separating sets, and (b) test outcomes exhibit limited positive dependence. A rigorous proof under these relaxed conditions remains an open problem, but the intuition is clear: placing high-probability tests first reduces expected runtime regardless of the specific dependence structure among tests.

Under Assumption~\ref{assumption:cit_power_independence}, we can precisely characterize optimal orderings:

\OptimalSequenceOrdering*

\textbf{Proof sketch.} Consider orderings $\mathsf{L}$ and $\mathsf{L}'$ differing only in positions $i$ and $i+1$, where $\mathsf{L}$ places a non-d-separating test before a d-separating test, and $\mathsf{L}'$ swaps them. The expected runtime difference equals $\mathbb{P}(\text{all prior tests fail}) \cdot [(1-\alpha) \cdot i + \alpha\beta(i+1)] - [\beta \cdot i + (1-\beta)(1-\alpha)(i+1)]$. Since $1-\alpha > \beta$, the first term places more weight on $i$ while the second places more weight on $i+1$. By the rearrangement inequality, this difference is negative, establishing $\mathbb{E}[t_{e_{i,j}}^{\mathsf{L}'}] < \mathbb{E}[t_{e_{i,j}}^{\mathsf{L}}]$. See Appendix~\ref{proof: optimal-sequence} for full proof.

This lemma establishes a clear ordering principle: placing d-separating sets before non-d-separating sets minimizes expected runtime. Any ordering that violates this principle can be improved by swapping adjacent ``inversions."

\subsection{Expert-Guided Algorithm with Monotonicity Guarantees}

Building on Lemma~\ref{lemma: optimal-sequence}, we design EP-Guess (Subroutine~\ref{algo: expert-subset-ordering1}) to leverage expert predictions of d-separating sets. The algorithm extracts all d-separating sets $\widehat{\mathcal{D}}_{ij}$ from the expert graph $\widehat{\mathcal{G}}$ and constructs ordering $\mathsf{L}$ by placing predicted d-separating sets ($[A]_k \cap \widehat{\mathcal{D}}_{ij}$) before predicted non-d-separating sets ($[A]_k \setminus \widehat{\mathcal{D}}_{ij}$), with random order within each partition.

We model expert $\psi$'s d-separation predictions using the same framework as edge predictions: for any pair $(x_i, x_j)$ and conditioning set $W$, the expert independently predicts whether $W$ d-separates $x_i, x_j$ in $\mathcal{G}^*$ with accuracy $p_{\text{d-sep}}$. This can be formalized as a binary symmetric channel where the expert observes the true d-separation status and reports it correctly with probability $p_{\text{d-sep}}$.

\ExpertAccuracyMonotonicity*

\textbf{Proof sketch.} The proof (App.~\ref{proof: expert-accuracy}) establishes monotonicity via a coupling argument between experts with accuracies $p_{\text{d-sep}_1} < p_{\text{d-sep}_2}$. Both experts observe the same true d-separating sets and use identical randomness for classification, but the higher-accuracy expert makes fewer errors. This ensures that every conditioning set correctly classified by the weaker expert is also correctly classified by the stronger expert.

Consequently, the better expert's ordering has fewer ``inversions" (non-d-separating sets incorrectly placed before d-separating sets). The better ordering can be obtained from the weaker ordering through a sequence of adjacent swaps that move d-separating sets leftward past non-d-separating sets. By Lemma~\ref{lemma: optimal-sequence}, each such swap strictly decreases expected runtime.

Since the better expert's ordering is reachable through runtime-reducing swaps, it achieves pointwise improvement for any fixed realization of data and expert predictions. Strassen's Coupling Theorem then implies that $T_{\text{EP-G}}(p_{\text{d-sep}_2})$ stochastically dominates $T_{\text{EP-G}}(p_{\text{d-sep}_1})$ in the first-order sense, yielding monotonicity in expectation. Part (b) follows from observing that $p_{\text{d-sep}} = 0.5$ produces orderings distributionally equivalent to random orderings.

\subsection{Discussion and Comparison to Edge Loop Guidance}

These results establish that EP-Guess provides complementary benefits to EL-Guess:

\begin{itemize}
    \item \textbf{EL-Guess} (Section~\ref{subsec: edge-ordering}): Improves \emph{accuracy} by prioritizing false edges, thereby reducing adjacency set inflation and improving the probability of correct edge decisions. The benefit is measured by increased perfect recovery probability $\Phi$.
    
    \item \textbf{EP-Guess} (this section): Improves \emph{computational efficiency} by prioritizing d-separating sets, thereby increasing early termination probability. The benefit is measured by decreased expected runtime $\mathbb{E}[T]$.
\end{itemize}

Both forms of guidance share the same underlying structure:
\begin{enumerate}
    \item Identify an ordering principle that universally improves performance (false edges first for EL; d-separating sets first for EP)
    \item Model the expert as a binary symmetric channel predicting the relevant classification
    \item Use coupling arguments to prove monotonic improvement with expert accuracy
    \item Establish robustness: performance is never worse than random when expert accuracy $\geq 0.5$
\end{enumerate}

The main technical difference is that EL guidance affects accuracy (via adjacency set modifications), while EP guidance affects only runtime (leaving accuracy invariant). This difference arises because EP operates at a fixed conditioning set size, testing all relevant subsets regardless of order, whereas EL modifies the graph structure dynamically, affecting which tests are even possible for subsequent edges.

In practice, both forms of guidance can be applied simultaneously: EL-Guess sequences edges optimally (prioritizing false edges), and within each edge's testing, EP-Guess sequences conditioning sets optimally (prioritizing predicted d-separating sets). The combined algorithm achieves both accuracy improvements and computational speedups when expert predictions are accurate.

\newpage
\section{Extract Ordering Subroutine}\label{appendix: extractorderingsubroutine}

\floatname{algorithm}{Subroutine}
\begin{algorithm}[h]
\caption{Extract Orderings from Expert (EOE)}\label{algo: extract-orderings}
\begin{algorithmic}[1]
\State \textbf{Inputs}: Expert $\psi$, complete skeleton $\calC$
\State Obtain $\hat{\mathcal{G}}$ from $\psi$. Extract skeleton $\hat{\mathcal{S}}$ and d-separating sets $\hat{\mathcal{D}}$ from $\hat{\mathcal{G}}$
\State Randomly order $\calC$ and $[V]_{1:|V|-1}$. Set $\mathsf{O} = \calC \setminus \hat{\mathcal{S}} + \calC \cap \hat{\mathcal{S}}$ and $\mathsf{L} = \left([V]_{1:|V|-1} \cap \hat{\mathcal{D}}\right) + \left([V]_{1:|V|-1} \setminus \hat{\mathcal{D}}\right)$
\State \Return $\mathsf{O}$, $\mathsf{L}$
\end{algorithmic}
\end{algorithm}

Subroutine \ref{algo: extract-orderings} extracts orderings $\mathsf{O},\mathsf{L}$ from an expert's guess of the causal DAG $\widehat{\mathcal{G}}$. It combines the first few steps of both Subroutine \ref{algo: expert-edge-ordering1} and \ref{algo: expert-subset-ordering1}.

\newpage
\section{Suboptimality of PC-Guess under Perfect Guidance}\label{appendix: pcg2g suboptimal}

We demonstrate that PC-Guess's level-by-level structure can prevent it from fully exploiting expert guidance, even when the expert provides perfect predictions. This inefficiency stems from PC's statistical conditioning bias, which prioritizes testing smaller conditioning sets before larger ones regardless of expert predictions.

\textbf{The core limitation.} When an expert correctly identifies a false edge $n_{i,j}$ and places it early in the edge ordering $\mathsf{O}$, PC-Guess can only remove this edge once it reaches conditioning set size $\ell = k$, where $k$ is the size of the minimal d-separating set for $x_i$ and $x_j$. At all prior levels $\ell \in \{0, 1, \ldots, k-1\}$, PC-Guess must test the edge and find dependence, leaving the false edge in the skeleton. During these early levels, this retained false edge inflates the adjacency sets of $x_i$ and $x_j$, forcing unnecessary conditioning set tests for all other edges incident to these vertices.

\textbf{Concrete example: 4-node chain.} Consider the true causal graph $\mathcal{G}^*: x_1 \rightarrow x_2 \rightarrow x_3 \rightarrow x_4$. The complete initial skeleton $\mathcal{C}$ contains 6 edges: three true edges $\{n_{1,2}, n_{2,3}, n_{3,4}\}$ and three false edges $\{n_{1,3}, n_{1,4}, n_{2,4}\}$. The false edge $n_{1,4}$ has minimal d-separating set $\{x_2\}$ of size $k=1$. Suppose the expert provides perfect guidance by placing $n_{1,4}$ first in the ordering $\mathsf{O}$. At level $\ell=0$, PC-Guess tests $n_{1,4}$ with conditioning set $\emptyset$ and finds dependence (correct, due to the chain path), leaving the false edge in the skeleton. Because $n_{1,4}$ remains, all subsequent level-0 tests of edges involving $x_1$ or $x_4$ use inflated adjacency sets: when testing $n_{1,2}$, we have $\mathrm{adj}_{-2}(\mathcal{C}, x_1) = \{x_3, x_4\}$ instead of $\{x_3\}$; when testing $n_{3,4}$, we have $\mathrm{adj}_{-4}(\mathcal{C}, x_3) = \{x_1, x_2\}$ instead of $\{x_2\}$. Only at level $\ell=1$ does PC-Guess test $n_{1,4}$ with $\{x_2\}$ and successfully remove it. The inefficiency: PC-Guess performed one guaranteed-to-fail test of $n_{1,4}$ at level 0, plus all level-0 tests of edges incident to $x_1$ or $x_4$ used unnecessarily large adjacency sets—waste that scales as $O(d^2)$ for chains of length $d$, all avoidable if the algorithm could immediately test $n_{1,4}$ with its minimal d-separator.

% \textbf{Scaling to larger graphs.} This inefficiency becomes more severe in longer chains. For a chain of length $d$: $x_1 \rightarrow x_2 \rightarrow \cdots \rightarrow x_d$, the false edge $n_{1,d}$ has minimal d-separating set $\{x_2\}$ of size 1, but the complete skeleton initially contains $\binom{d}{2}$ edges. At level 0, PC-Guess tests all $\binom{d}{2}$ edges with empty conditioning sets. The false edge $n_{1,d}$ remains, keeping $|\mathrm{adj}(\mathcal{C}, x_1)| = d-1$ and $|\mathrm{adj}(\mathcal{C}, x_d)| = d-1$. Every edge incident to $x_1$ or $x_d$ (there are $2d-3$ such edges) must be tested with inflated adjacency sets. The total number of unnecessary tests grows as $O(d^2)$ at level 0 alone.

\textbf{Conclusion.} This example demonstrates that PC-Guess's adherence to the level-by-level structure prevents it from fully exploiting perfect expert guidance. The algorithm must exhaust all smaller conditioning set sizes before testing the conditioning set size that would actually remove the false edge, wasting both direct tests on the false edge and indirect tests on neighboring edges with inflated adjacency sets. This motivates gPC-Guess (Section~\ref{subsec: lapcg2g}), which eliminates the level-by-level constraint and allows immediate testing with conditioning sets of any size, enabling the algorithm to act on expert predictions without delay.

\newpage

\section{Experimental Details}\label{appendix: experimental details}

\subsection{Synthetic Data Generation Parameters}

We generate synthetic data using linear Gaussian structural equation models on Erdős-Rényi (ER) random graphs. The data generation process follows these steps:

\begin{itemize}
    \item \textbf{Graph Structure}: We generate $d$-dimensional DAGs using the Erdős-Rényi model where edges are added independently with probability $p_{edge}$. The DAG property is ensured by making the adjacency matrix lower triangular, preventing cycles and self-loops.
    
    \item \textbf{Edge Weights}: For each edge in the binary adjacency matrix, we assign weights sampled uniformly from $[-2.5, -1.5] \cup [1.5, 2.5]$ to avoid faithfulness violations that occur when weights are near zero.
    
    \item \textbf{Data Generation}: Each variable $x_i$ follows the linear structural equation model:
    $$x_i = \sum_{j \in \text{Pa}(x_i)} w_{ji} x_j + \varepsilon_i$$
    where $w_{ji}$ are the edge weights and $\varepsilon_i \sim \mathcal{N}(0, 1)$ are independent Gaussian noise terms with mean 0 and variance 1.0. All noise terms are generated independently across variables and samples.
    
    \item \textbf{Standardization}: All generated data is standardized to zero mean and unit variance to ensure fair comparison across different graph structures.
    
     \item \textbf{Variable Randomization}: Before feeding data to any causal discovery method, we randomly permute the order of variables to prevent information leakage from variable ordering.
\end{itemize}

We consider two sparsity levels based on edge probability:
\begin{itemize}
    \item \textbf{Sparse graphs (ER1)}: Edge probability $p_{edge} = 1/(d-1)/2$, yielding approximately $d$ edges in expectation
    \item \textbf{Dense graphs (ER3)}: Edge probability $p_{edge} = 3/(d-1)/2$, yielding approximately $3d$ edges in expectation
\end{itemize}
In the main text results, we focus on sparse ER3 graphs with $d = 20$ variables and $n = 100$ samples.

\subsection{Real-World Data}

We use the discrete version of the Sachs protein signaling dataset from the BNLearn repository, which contains measurements of 11 phosphoproteins and phospholipids in human immune system cells. The dataset details include:

\begin{itemize}
    \item \textbf{Variables}: 11 proteins/phospholipids: Raf, Mek, Plcg, PIP2, PIP3, Erk, Akt, PKA, PKC, P38, Jnk.
    \item \textbf{Ground Truth}: Known causal DAG structure based on established biological pathways.
    \item \textbf{Sample Sizes}: We subsample the original dataset to create experiments with $n = 100$ samples.
    \item \textbf{Preprocessing}: Data is standardized to zero mean and unit variance.
\end{itemize}

The Sachs dataset provides a realistic benchmark for evaluating causal discovery methods on real biological networks with known ground truth structure.

\subsection{Algorithm Input and Baseline Methods}

\textbf{Algorithm input and ordering choices.} For all experiments reported in Section~\ref{sec: exps} and Appendix~\ref{appendix: additional experimental results} (except Appendix~\ref{appendix: varying d-seperation results}), algorithms receive only a predicted \emph{skeleton}—an undirected graph over the variables—not a full DAG. This skeleton guides the Edge Loop (EL) subroutine in determining which edges to test and in what order (see Section~\ref{sec: g2g cb theory}). The ordering used to guide Edge Prune (EP)—the sequence of conditioning sets tested for each individual edge—is \emph{always generated uniformly at random} for all methods. This design choice reflects our focus on accuracy improvements rather than runtime gains: by Lemma~\ref{lemma: orderingindaccuracy}, EP ordering does not affect edge decision accuracy, only computational cost.

\textbf{Baseline methods.} PC and gPC serve as baseline versions of PC-Guess and gPC-Guess respectively. These baselines always receive skeletons generated with expert edge prediction accuracy $p_\psi = 0.5$ (equivalent to random guessing) and use random EP ordering (equivalent to $p_{\text{d-sep}} = 0.5$). By comparing PC-Guess and gPC-Guess against PC and gPC, we isolate the benefit of expert guidance over random ordering.

\textbf{Number of trials.} All results reported in Section~\ref{sec: exps} and Appendix~\ref{appendix: additional experimental results} reflect averages over 30 independent trials, each with a different random seed controlling data generation, expert prediction sampling, and algorithmic randomness.

\subsection{Simulated Expert Implementation}

We implement simulated experts that generate skeleton predictions for evaluating algorithm performance across controlled accuracy levels.

\textbf{Skeleton generation for simulated experts.} For experiments with simulated experts (Figures~\ref{fig:first}, \ref{fig:second}, and most experiments in Appendix~\ref{appendix: additional experimental results}), we generate predicted skeletons as follows:
\begin{enumerate}
    \item For each potential edge pair $(x_i, x_j)$ where $i < j$, we check whether the edge exists in the true skeleton $\mathcal{S}^*$
    \item With probability $p_\psi$ (the expert edge prediction accuracy), we correctly classify the edge (add it to $\widehat{\mathcal{S}}$ if it exists in $\mathcal{S}^*$, exclude it otherwise)
    \item With probability $1-p_\psi$, we misclassify the edge (add it to $\widehat{\mathcal{S}}$ if it does \emph{not} exist in $\mathcal{S}^*$, exclude it otherwise)
\end{enumerate}
This process simulates a binary symmetric channel and allows systematic evaluation across controlled accuracy levels $p_\psi \in [0.3, 1.0]$. Each of the 30 trials in each experiment with simulated experts uses an independently sampled skeleton prediction.

\textbf{D-separation prediction (Appendix~\ref{appendix: varying d-seperation results} only).} The experiments in Appendix~\ref{appendix: varying d-seperation results} differ from all other experiments because they focus on evaluating how d-separation prediction accuracy $p_{\text{d-sep}}$ affects runtime (not accuracy). For these experiments:
\begin{enumerate}
    \item The skeleton $\widehat{\mathcal{S}}$ is generated using a simulated expert with accuracy $p_\psi = 0.5$ (random edge prediction), ensuring the skeleton provides no informative signal about true edge structure
    \item We use a simulated expert to predict d-separating sets: for each edge $e_{i,j}$ being tested, and for each candidate conditioning set $W$, the expert correctly identifies whether $W$ d-separates $x_i, x_j$ in $\mathcal{G}^*$ with probability $p_{\text{d-sep}}$
    \item This d-separation prediction guides EP ordering according to Subroutine~\ref{algo: expert-subset-ordering1}: predicted d-separating sets are tested before predicted non-d-separating sets
    \item We vary $p_{\text{d-sep}} \in [0.5, 1.0]$ to measure how d-separation accuracy affects expected runtime (Lemma~\ref{lemma: expert-accuracy})
\end{enumerate}

\subsection{LLM Expert Implementation}

We use Claude Opus 4.1 as our LLM expert, accessed through Amazon Bedrock. For experiments with the LLM expert (Figure~\ref{fig:third}), the skeleton generation process is:

\begin{itemize}
    \item \textbf{Prompting Strategy}: For each of the 30 trials, we prompt Claude Opus 4.1 once using the following prompt template, with variable names randomly shuffled for that trial (to prevent prompt bias):
    
    \begin{quote}\itshape
    "Analyze this protein signaling network step by step: \{var\_names\_str\}
    
    Step 1: Consider each protein's known biological functions\\
    Step 2: Identify which proteins can directly interact with each other\\
    Step 3: Look for signaling pathways and cascades\\
    Step 4: Include regulatory relationships (activation/inhibition)
    
    For each pair of proteins, ask: Can protein A directly influence protein B's activity or state?
    
    List ALL direct causal relationships as pairs. Be comprehensive - missing edges is worse than including uncertain ones. Return your answer as a list of variable name pairs that have direct edges between them. Format your response as pairs of variable names in parentheses, separated by commas. Start your final answer with the tag EDGES: followed by your list."
    \end{quote}
    
    This prompt uses step-by-step reasoning to systematically guide the LLM through the causal discovery process. It employs chain-of-thought prompting by breaking down the analysis into discrete steps, and uses recall-oriented instructions ("Be comprehensive", "missing edges is worse") to encourage high recall of potential causal relationships.
    
    \item \textbf{Response Parsing}: We parse the LLM response (which returns edge pairs in the format "EDGES: (var1, var2), (var3, var4), ...") to extract the predicted skeleton $\widehat{\mathcal{S}}$ for that trial.
    
    \item \textbf{Trial-specific pairing}: For each trial, we sample $n=100$ observations from the Sachs dataset and provide both the data and the corresponding LLM-predicted skeleton to gPC-Guess. This approach generates 30 independent LLM predictions (one per trial), each paired with an independent data subsample.
\end{itemize}

The LLM expert leverages pre-trained biological knowledge to make predictions about protein signaling networks, providing a realistic test of how modern AI systems can augment causal discovery.

\subsection{Conditional Independence Testing}\label{appendix: implementations}

For synthetic data experiments with linear Gaussian models, all algorithms use Fisher's Z-test \citep{fisher1921probable} for conditional independence testing with significance level $\alpha = 0.05$. The test statistic is:
$$Z = \frac{1}{2}\sqrt{n-|S|-3} \log\left(\frac{1+\hat{\rho}_{XY|S}}{1-\hat{\rho}_{XY|S}}\right)$$
where $\hat{\rho}_{XY|S}$ is the sample partial correlation between variables $X$ and $Y$ given conditioning set $S$, and $n$ is the sample size. Under the null hypothesis of conditional independence, $Z$ follows a standard normal distribution. Fisher's Z-test is appropriate for continuous data generated from linear Gaussian structural equation models.

For experiments with the Sachs dataset, we use the chi-square test \citep{pearson1900criterion} of independence, which is appropriate for discrete data. The test statistic is:
$$\chi^2 = \sum_{i,j} \frac{(O_{ij} - E_{ij})^2}{E_{ij}}$$
where $O_{ij}$ are observed frequencies and $E_{ij}$ are expected frequencies under the null hypothesis of independence in the contingency table. This test evaluates conditional independence by comparing observed and expected frequencies across all combinations of variable values and conditioning set states.

\subsection{Packages and Dependencies}

The experimental code uses the following Python packages:

\begin{itemize}
    \item \textbf{numpy} (1.21.0+): Array operations and linear algebra
    \item \textbf{scipy} (1.7.0+): Statistical functions and hypothesis testing
    \item \textbf{causal-learn} (0.1.3.8+): PC and PC-Stable algorithm implementations
    \item \textbf{networkx} (2.6.0+): Graph operations and d-separation queries
    \item \textbf{boto3} (1.26.0+): Amazon Bedrock API access for LLM experiments
    \item \textbf{matplotlib} (3.5.0+): Plotting and visualization
    \item \textbf{tqdm} (4.62.0+): Progress bars for long-running experiments
    \item \textbf{json}: Configuration and results serialization (Python standard library)
    \item \textbf{itertools}: Combinatorial operations for conditioning sets (Python standard library)
    \item \textbf{concurrent.futures}: Parallel experiment execution (Python standard library)
    \item \textbf{datetime}: Experiment timestamping and logging (Python standard library)
\end{itemize}

\subsection{Compute Details}

All experiments were conducted using Python 3.8+, and run on a EC2 instance with AMD EPYC 7R13 processors, 192 vCPUs (96 cores with 2 threads per core), and 740 GiB of memory running Amazon Linux 2. Parallel experiments used up to 8 concurrent workers via Python's ProcessPoolExecutor to balance computational efficiency with resource constraints.
\newpage

\section{Additional Experimental Results}\label{appendix: additional experimental results}

\subsection{Runtime Results for Figure 2a and 2b in Main Text}\label{appendix: runtime results}
\begin{figure*}[h!]
    \centering
    \begin{subfigure}[b]{0.49\textwidth}
        \centering
        \includegraphics[width=\textwidth]{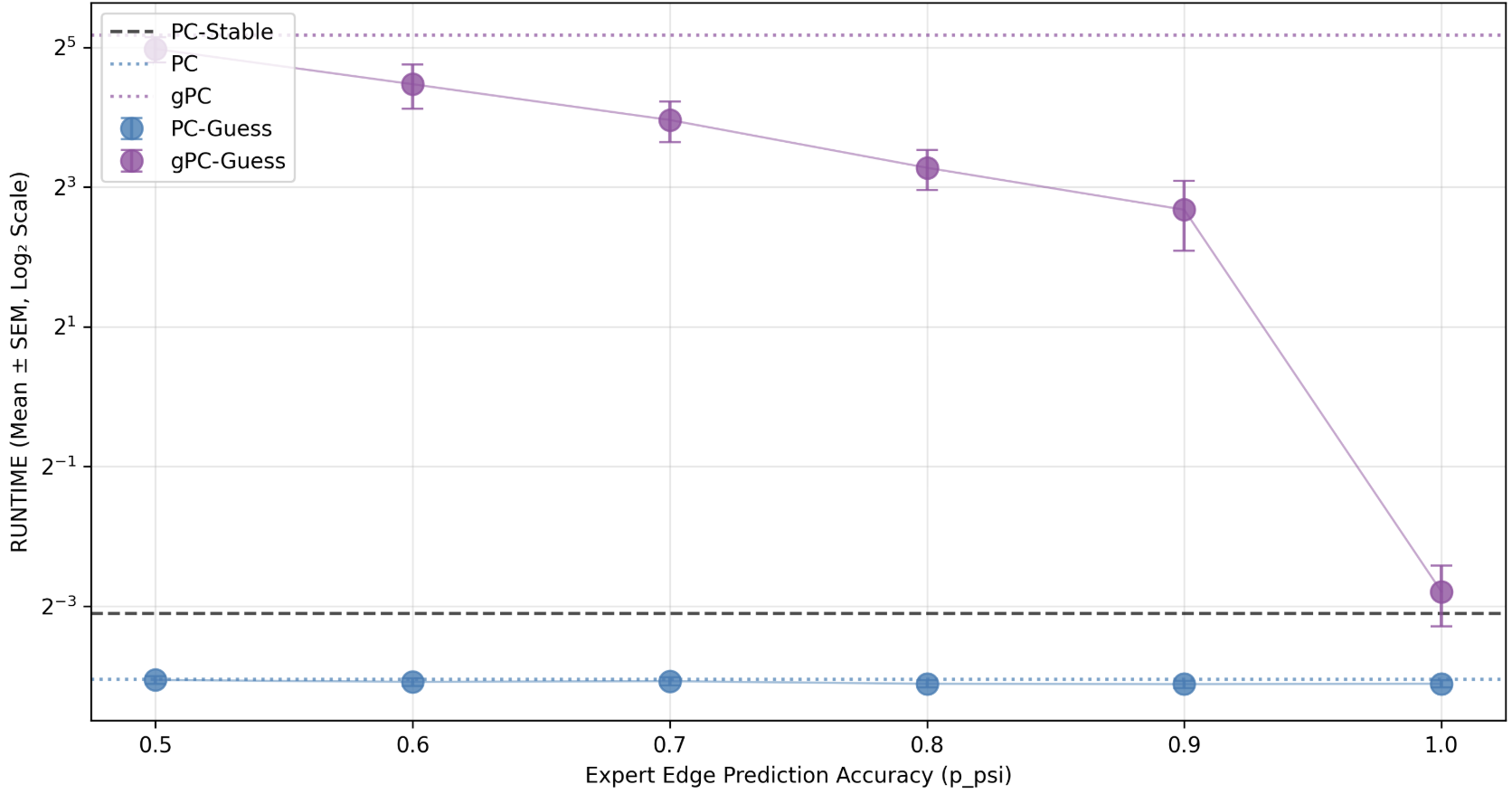}
        \caption{Runtime (s) results for Figure~\ref{fig:tenth} (ER1, $d=20,n=100$).}
        \label{fig:eleventh}
    \end{subfigure}
    \hfill
    \begin{subfigure}[b]{0.49\textwidth}
        \centering
        \includegraphics[width=\textwidth]{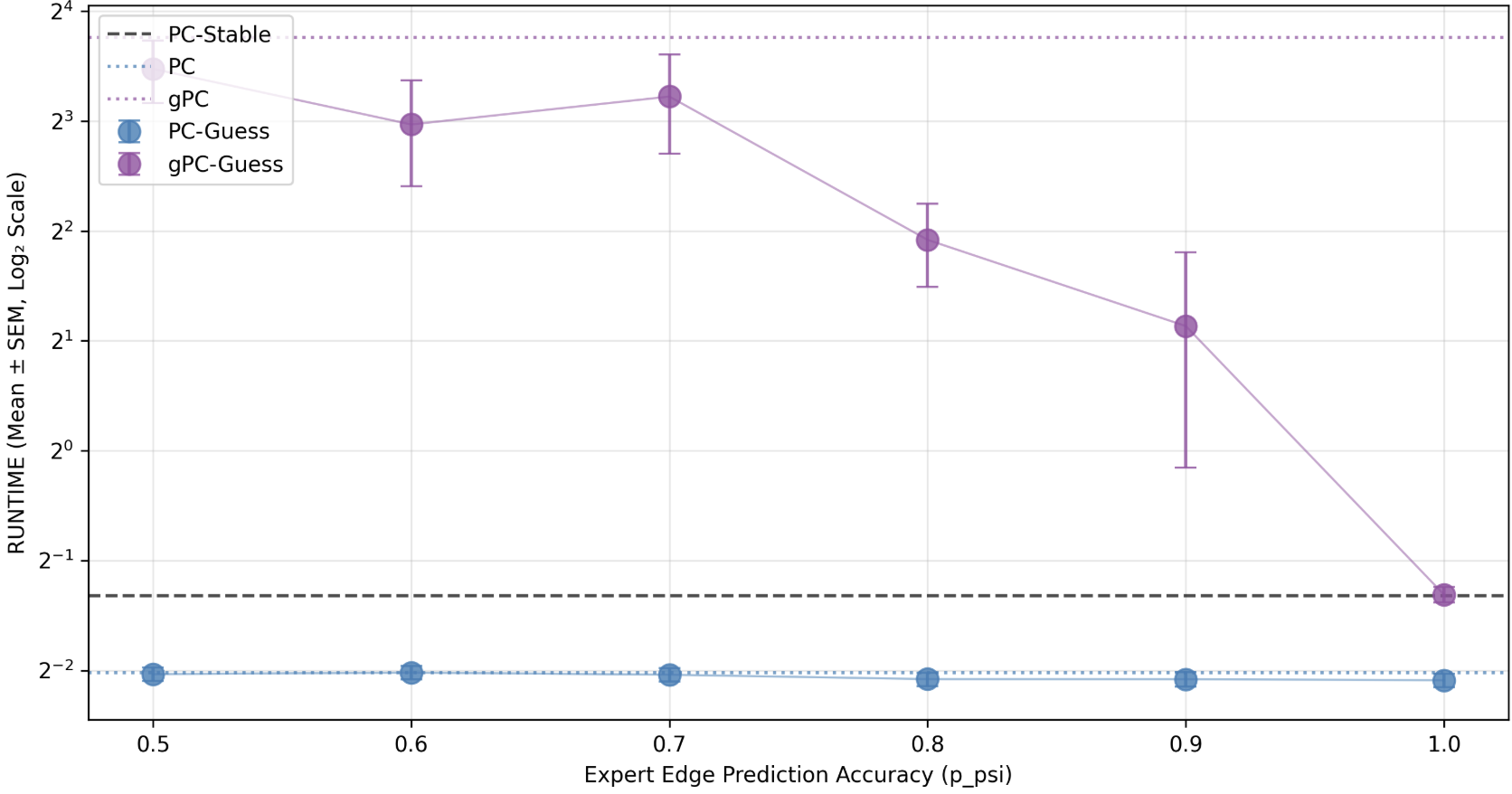}
        \caption{Runtime (s) results for Figure~\ref{fig:first} (ER3, $d=20,n=100$).}
        \label{fig:twelveth}
    \end{subfigure}
    \caption{Results for runtime when varying $p_\psi$.}
    \label{fig:combined1}
\end{figure*}

Runtime of gPC-Guess and PC-Guess both decline as expert prediction $p_\psi$ increases, although the reduction is much larger in the dense rather than sparse setting, and in both settings the runtime reduction for gPC-Guess is far larger than for PC-Guess.

\subsection{Sparse Graphs}\label{appendix: sparse graph results}
\begin{figure}[h!]
        \centering
        \includegraphics[width=0.45\textwidth]{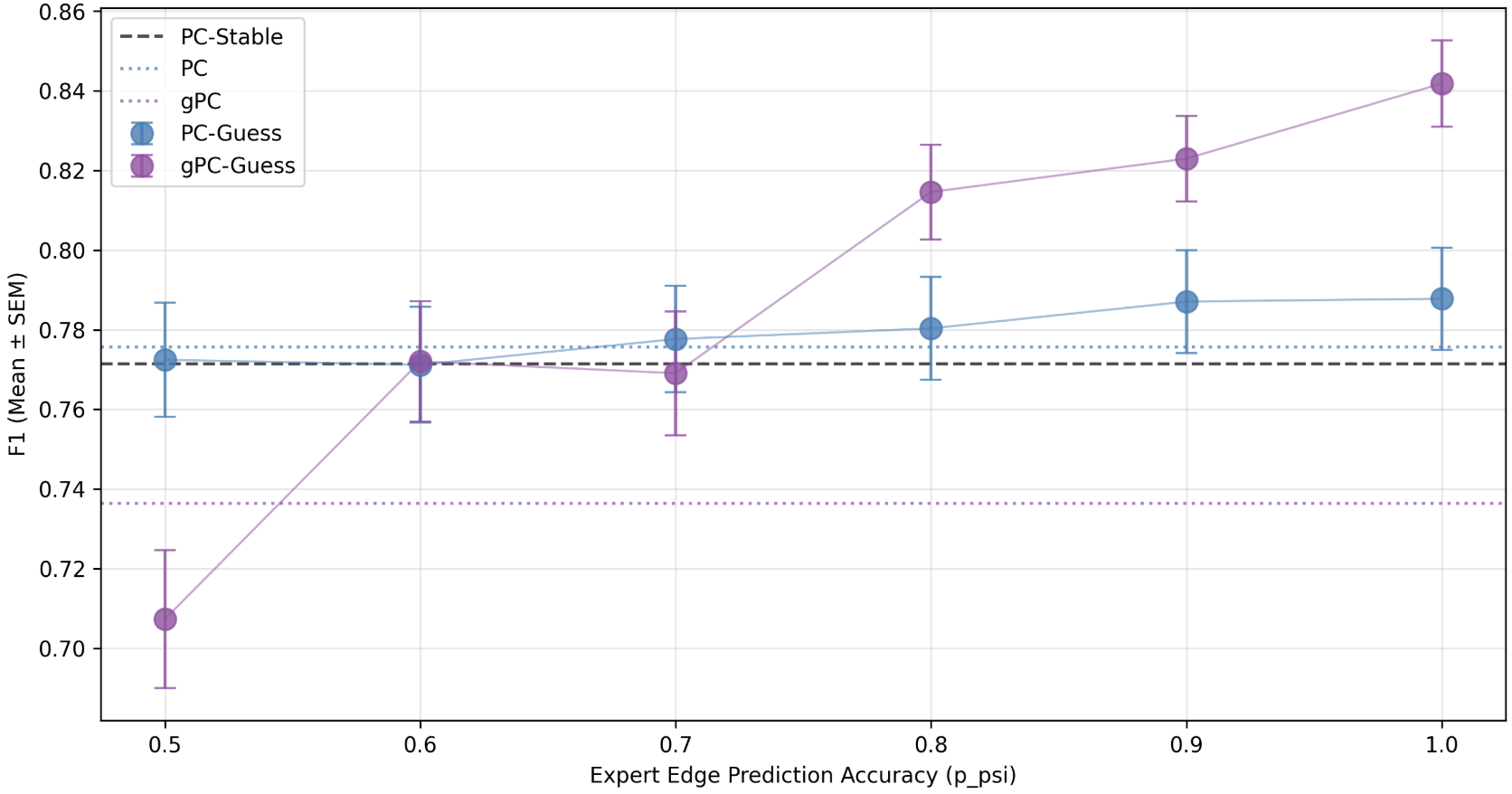}
    \caption{Method performance as $p_\psi$ increases in sparse graphs (ER1, $d=10$, $n=100$).}\label{fig:tenth}
\end{figure}

PC-Guess continues to outperform baselines as $p_\psi$ grows, but with a smaller increase in F1 than observed in the dense setting. Similar to the dense setting, gPC-Guess performance increases the most with $p_\psi$, again surpassing all other methods when $p_\psi=0.7$.

\newpage
\subsection{Varying Sample Size}\label{appendix: sample size results}
\begin{figure}[h!]
        \centering
        \includegraphics[width=0.45\textwidth]{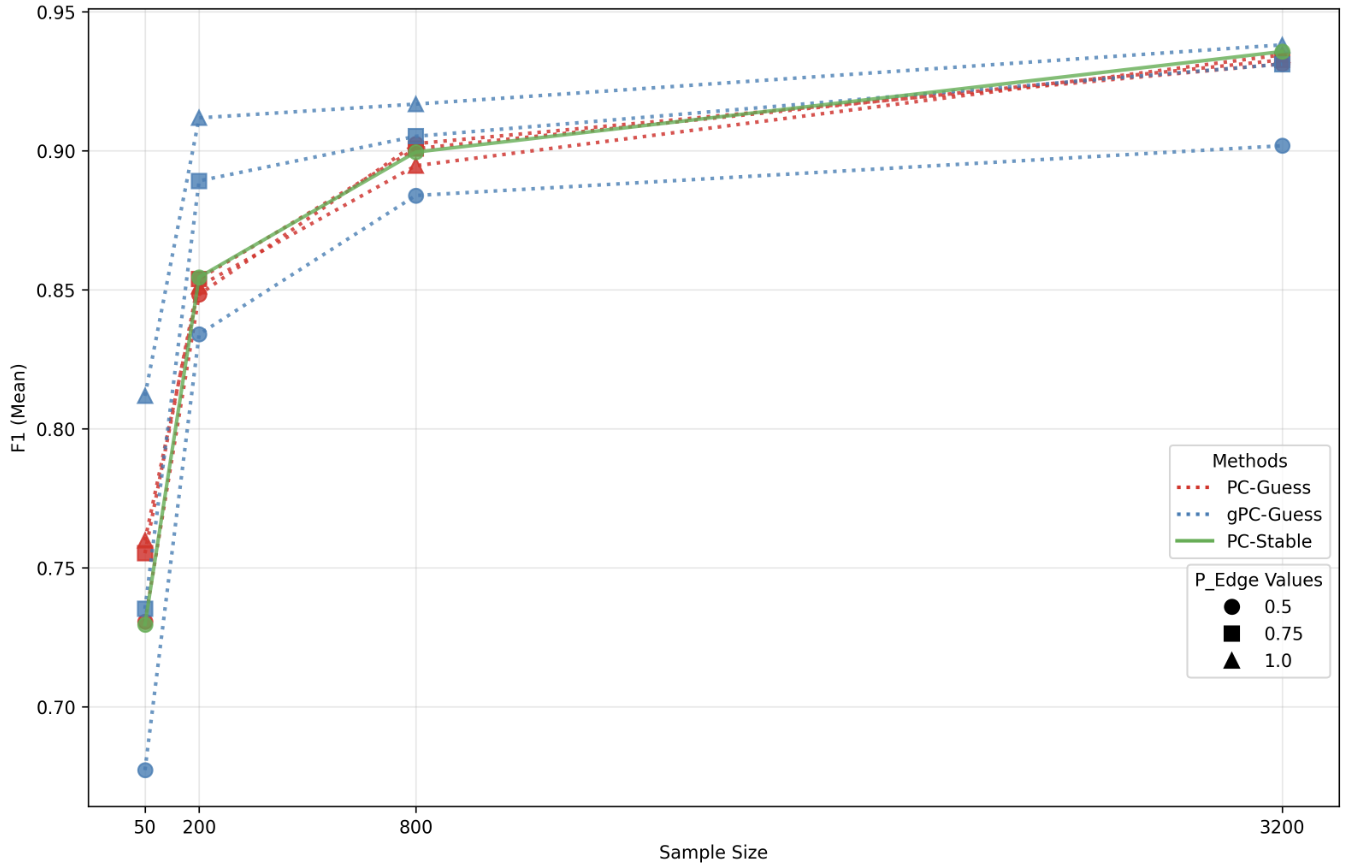}
    \caption{Method performance across different values of $p_\psi$, in sparse graphs (ER1, $d=10$), as sample size is rapidly increased.}\label{fig:fourtenth}
\end{figure}

As expected from the correctness result provided for PC-Guess and gPC-Guess (Theorem~\ref{theorem: pc performance}) that holds independent of expert quality, all methods (no matter what the expert edge prediction accuracy $p_\psi$ is) are converging to perfect accuracy with increasing sample size.

\subsection{Varying Dimensionality}\label{appendix: dimensionality results}
\begin{figure}[h!]
        \centering
        \includegraphics[width=0.45\textwidth]{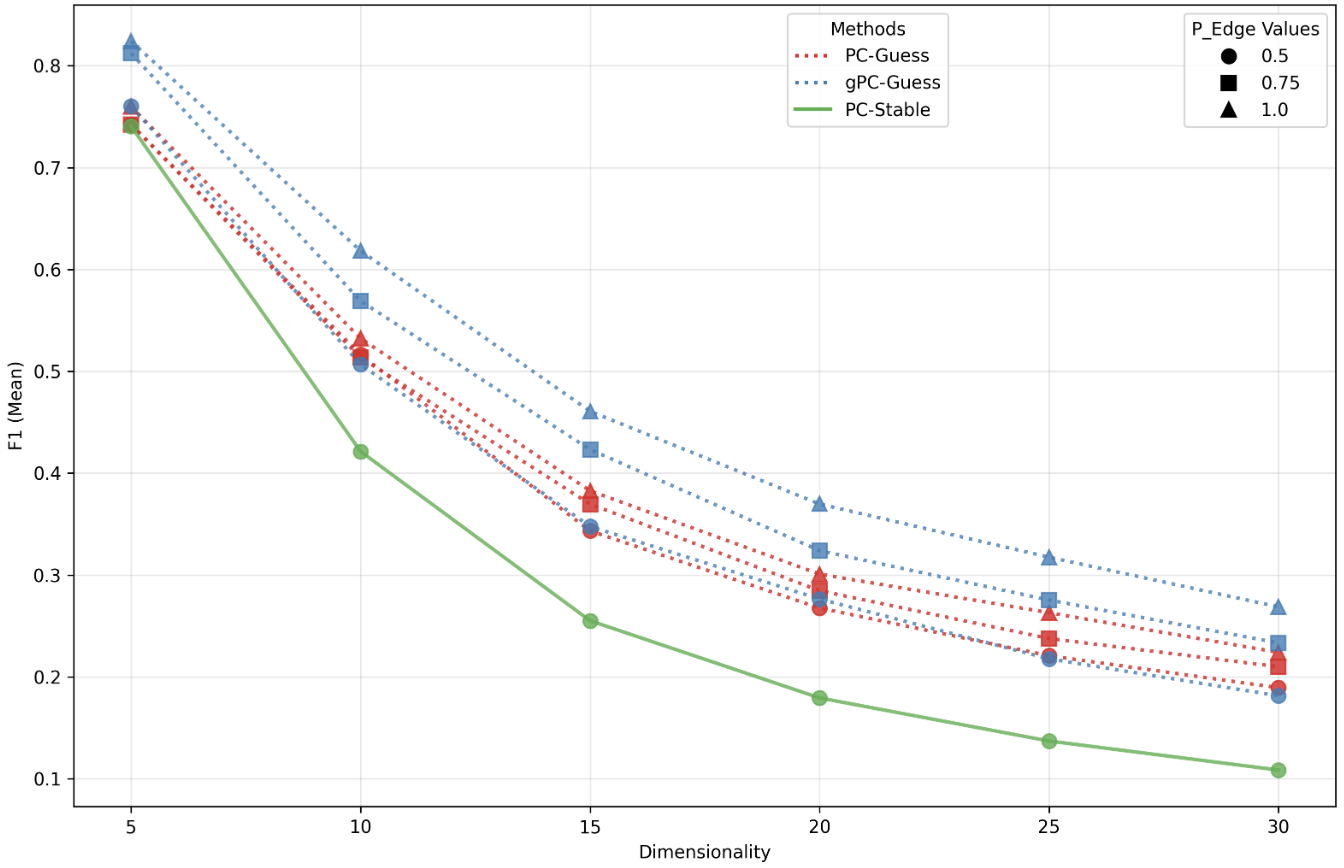}
    \caption{Method performance across different values of $p_\psi$, in dense graphs (ER3, $n=100$), as graph dimensionality $d$ is increased.}\label{fig:thirtenth}
\end{figure}

We note that gPC-Guess and PC-Guess continue to outperform the baseline PC-Stable as dimensionality is increased, even with sample size fixed. The gap between gPC-Guess and the baseline PC-Stable widens as dimensionality increases---for $p_\psi=1.0$, the gap between gPC-Guess and PC-Stable at $d=5$ is only $\sim$7 percentage points, whereas at $d=30$ the gap between them is $\sim$18 percentage points. This suggests that the value of expert guidance to performance increases in high-dimensional settings that are challenging for purely data-driven methods.

\newpage
\subsection{Results for Worst-Case Expert Performance}\label{appendix: worst case results}
\begin{figure*}[h!]
    \centering
    \begin{subfigure}[b]{0.49\textwidth}
        \centering
        \includegraphics[width=\textwidth]{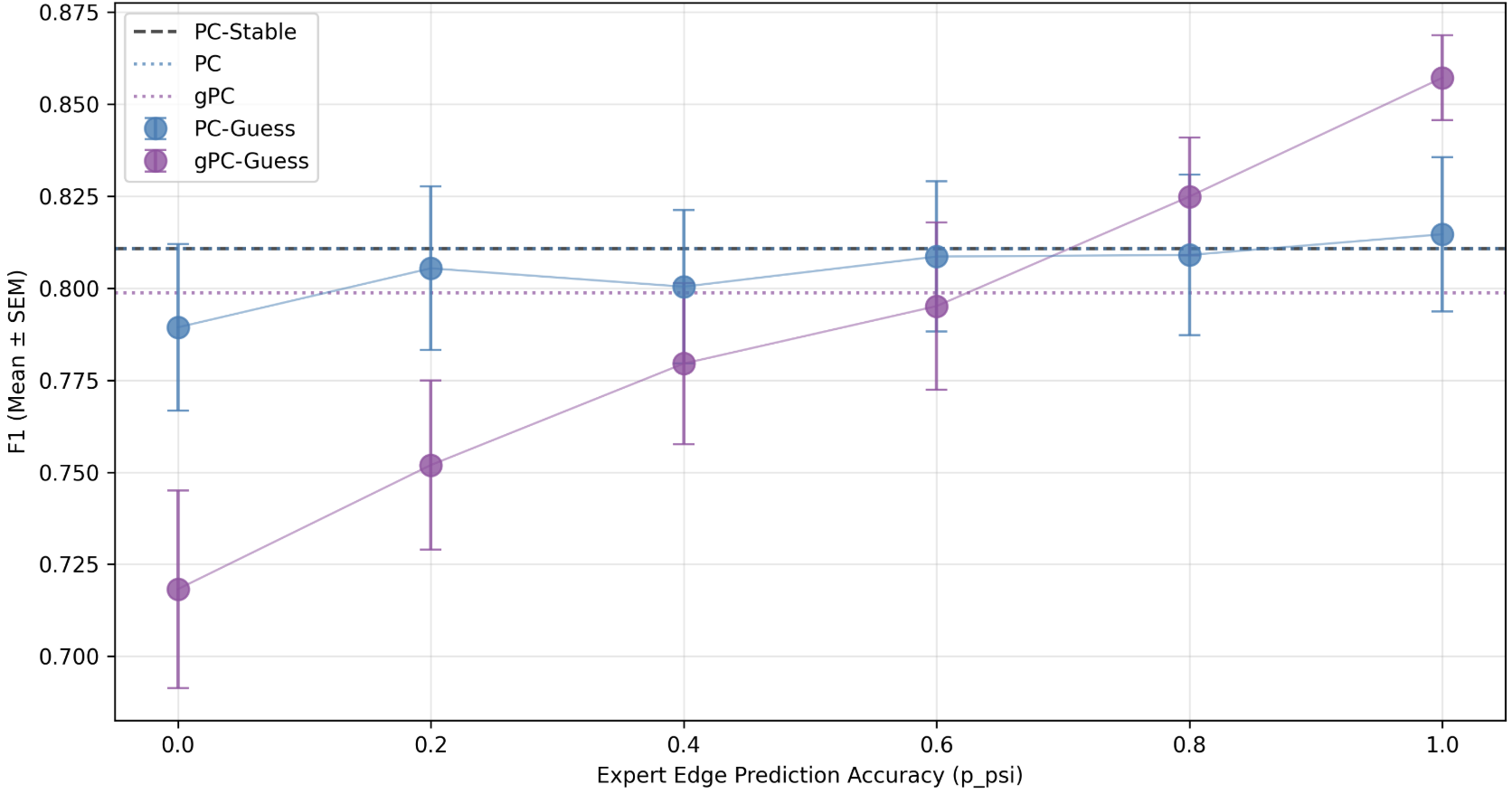}
        \caption{Method performance in sparse graphs (ER1, $d=10,n=100$) with varying $p_\psi$.}
        \label{fig:fifth}
    \end{subfigure}
    \hfill
    \begin{subfigure}[b]{0.49\textwidth}
        \centering
        \includegraphics[width=\textwidth]{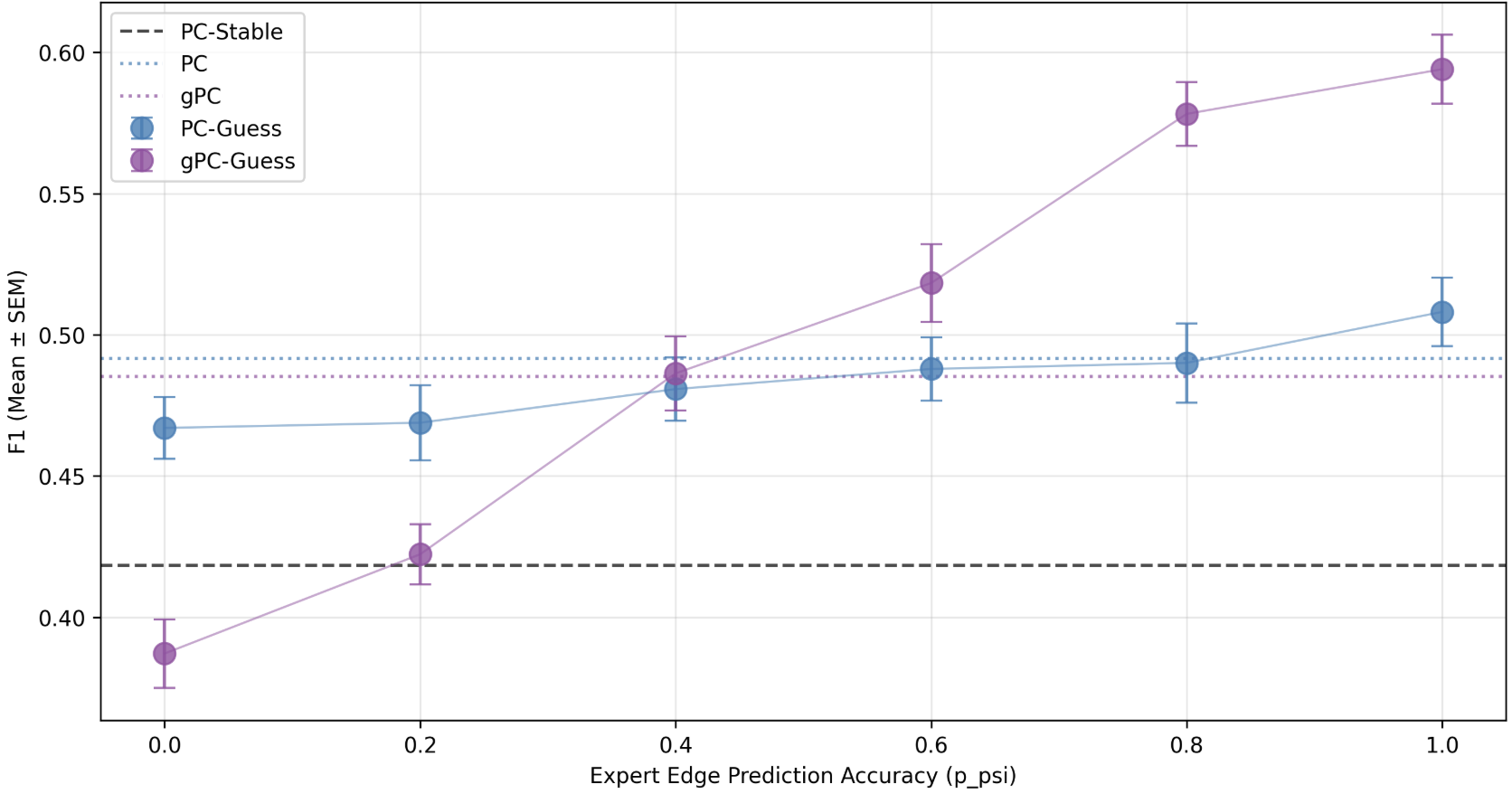}
        \caption{Method performance in dense graphs (ER3, $d=10,n=100$) with varying $p_\psi$.}
        \label{fig:sixth}
    \end{subfigure}
    \caption{Results for varying $p_\psi$ in the worst case, i.e., the expert is worse than random ($p_\psi\leq 0.5$).}
    \label{fig:combined2}
\end{figure*}

As expected from our theoretical monotonicity results (Lemma~\ref{lemma: edge-accuracy-monotonic}, Theorem~\ref{theorem: pc performance}, Theorem~\ref{theorem: sgs performance}), we see that both PC-Guess's and gPC-Guess's performances are worse than their counterparts PC and gPC when the expert prediction is worse than random, i.e., $p_\psi\leq 0.5$. We note that PC-Guess's performance is impacted less than gPC-Guess's performance, with a smaller reduction when the expert is poor, but gPC-Guess has a larger gain in performance when the expert is good. However, due to our robust correctness guarantees (Theorem~\ref{theorem: pc performance}), in both the dense and sparse setting the worst case performance (i.e., when the expert is entirely inaccurate, every single edge prediction is wrong, $p_\psi=0$) the performance drop from baseline is only up to roughly 8 percentage points. Unlike expert-aided soft/hard constraint methods, even when expert guidance is poor the drop in performance is bounded because the expert never replaces tests, only guides sequences.

\subsection{Varying D-Separation Prediction}\label{appendix: varying d-seperation results}
\begin{figure*}[h!]
    \centering
    \begin{subfigure}[b]{0.49\textwidth}
        \centering
        \includegraphics[width=\textwidth]{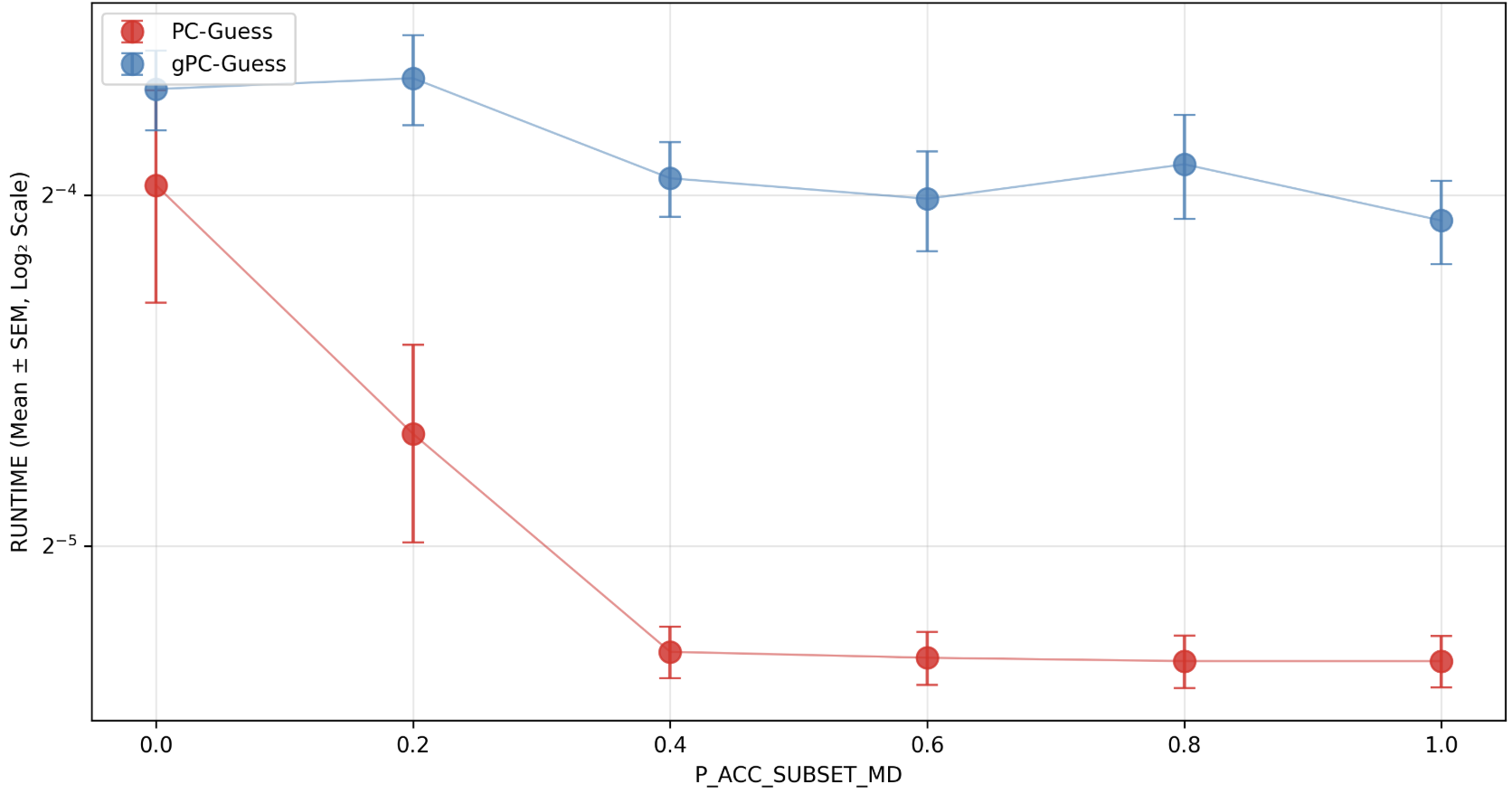}
        \caption{Method runtime in sparse graphs (ER1, $d=10$, $n=100$).}
        \label{fig:seventh}
    \end{subfigure}
    \hfill
    \begin{subfigure}[b]{0.49\textwidth}
        \centering
        \includegraphics[width=\textwidth]{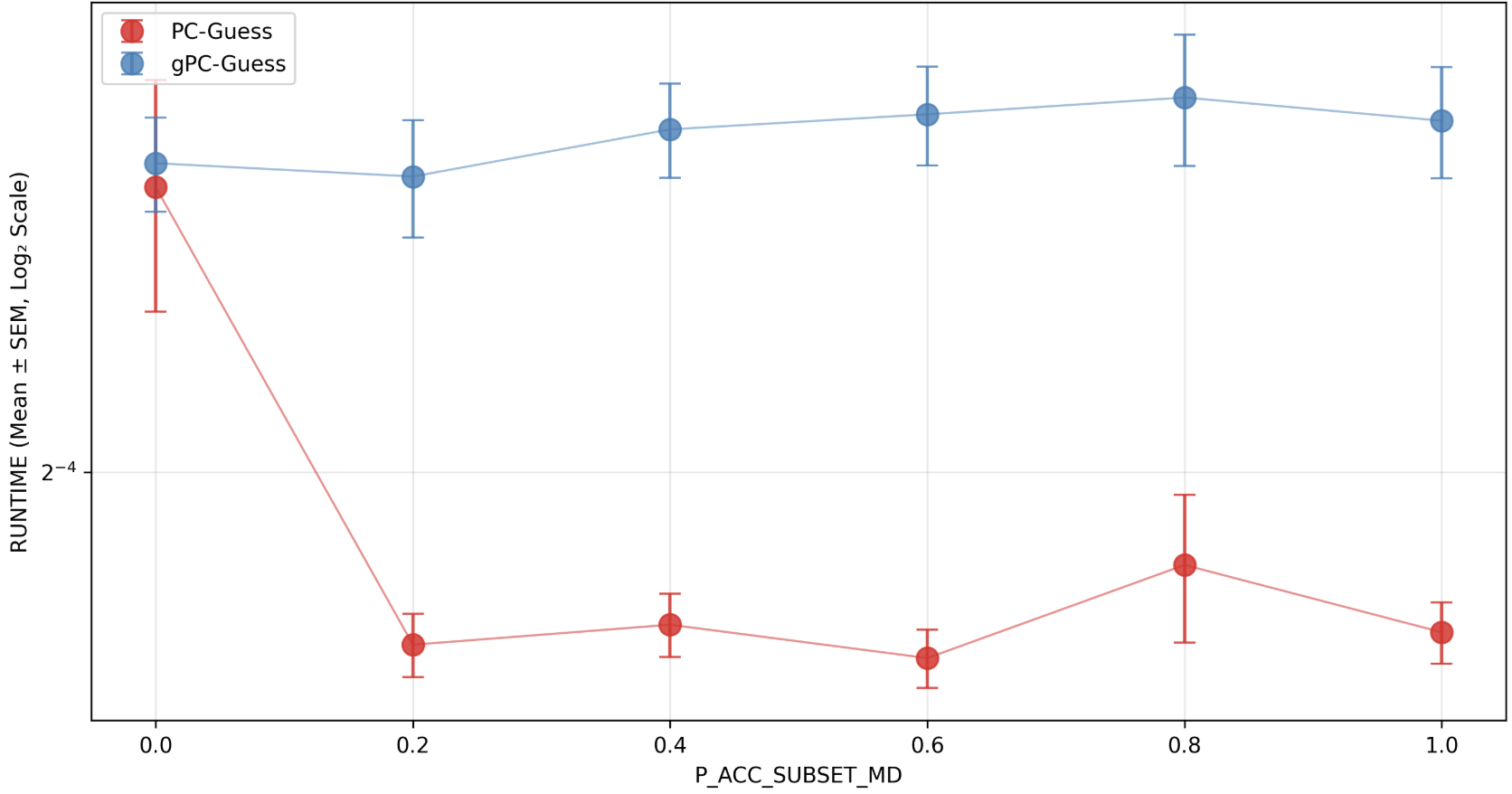}
        \caption{Method runtime in dense graphs (ER3, $d=10$, $n=100$).}
        \label{fig:eighth}
    \end{subfigure}
    \caption{Method runtime as d-separating prediction accuracy $p_{\text{d-sep}}$ is varied.}
    \label{fig:combined}
\end{figure*}

We note that increasing $p_{\text{d-sep}}$ appears to decrease runtime of both PC-Guess and gPC-Guess in sparse settings as expected by Lemma~\ref{lemma: expert-accuracy}, but the reduction in dense graphs is not observable for gPC-Guess, and only observable for low $p_{\text{d-sep}}$ values for PC-Guess.

%  old appendix stuffs
% \input{appendix/causal-discovery-background}
% \newpage
% \input{appendix/problem-setup}
% \newpage
% \input{appendix/failure-mode}
% \newpage
% \input{appendix/edge-ordering}
% \newpage
% \input{appendix/subset-ordering}
% \newpage
% \input{appendix/leverage-expert}
% \newpage
% \input{appendix/proofs}

% \section{Notation}
% \newpage
% \input{appendix/algorithms}

% \newpage
% \input{appendix/se-ges-guess}

% \newpage
% \input{appendix/resit-guess}

% \newpage
% \input{appendix/additional-experiments}

% \newpage
% \input{appendix/experimental-details}

\end{document}